\documentclass{elsarticle}
\pdfoutput=1
\usepackage{hyperref}

\journal{Journal}

\usepackage{times}
\usepackage{soul}
\usepackage{url}
\usepackage[small]{caption}
\usepackage{graphicx}
\usepackage{amsmath}
\usepackage{amsthm}
\usepackage{booktabs}
\usepackage{algorithm}
\usepackage{algorithmic}
\usepackage{bcprules}
\usepackage{mathpartir}

\usepackage{amssymb,amsfonts}
\usepackage{bm}
\usepackage{xcolor}
\usepackage{proof}
\usepackage{ifthen}
\usepackage{nicefrac}
\usepackage{thmtools}
\usepackage{thm-restate}

\declaretheorem[name=Lemma]{lem}

\declaretheorem[name=Example,style=definition]{example}

\declaretheorem[name=Definition,style=definition]{definition}
\makeatletter
\let\MYcaption\@makecaption
\makeatother
\usepackage{subcaption}
\usepackage[normalem]{ulem}

\allowdisplaybreaks[4]

\newcommand\Eq[1] {Equation~\eqref{#1}}
\newcommand\Sec[1] {Section~\ref{#1}}
\newcommand\Sects[1] {Sections~\ref{#1}}
\newcommand\App[1] {\ref{#1}}
\newcommand\Thm[1] {Theorem~\ref{#1}}

\newcommand\Propo[1]{Proposition~\ref{#1}}
\newcommand\Lem[1] {Lemma~\ref{#1}}
\newcommand\Def[1] {Definition~\ref{#1}}
\newcommand\Eg[1] {Example~\ref{#1}}
\newcommand\Tbl[1] {Table~\ref{#1}}
\newcommand\Fig[1] {Figure~\ref{#1}}

\newcommand\etal[0] {\emph{et al.}}

\newcommand\rulesp{\vspace{1ex}}
\renewcommand{\phi} {\varphi}

\newcommand{\eqdef}{\ensuremath{\stackrel{\mathrm{def}}{=}}}
\newcommand{\detassign}{\ensuremath{:=}}

\newcommand{\range}{\mathsf{range}}

\newcommand{\Dists}{\mathbb{D}} 

\newcommand{\size}[1]{\mathsf{size}(#1)}
\newcommand{\len}[1]{\mathsf{len}(#1)}

\newcommand{\popl}{\mathit{P}} 
\newcommand{\poplA}{\mathit{P_{\!\alg}}} 
\newcommand{\dompop}{\mathcal{X}} 
\newcommand{\xiPPL}[0]{\xi} 
\newcommand{\poplxt}{\popl(\xiPPL, \theta)} 
\newcommand{\mean}{\mathsf{mean}}

\newcommand{\ppl}[0]{}

\newcommand{\sideT}[0]{\mathsf{T}}
\newcommand{\sideL}[0]{\mathsf{L}}
\newcommand{\sideU}[0]{\mathsf{U}}

\newcommand{\nuh}[0]{\mathop{\backsim\!}}

\newcommand{\phiT}[0]{\phi_{\sideT}}
\newcommand{\phiL}[0]{\phi_{\sideL}}
\newcommand{\phiU}[0]{\phi_{\sideU}}

\newcommand{\psiPre}[0]{\psi_{{\sf pre}}}

\newcommand{\psiPostAB}[0]{\psi_{12}^{{\sf post}}}
\newcommand{\psiPostAC}[0]{\psi_{13}^{{\sf post}}}
\newcommand{\phiPost}[0]{\phi_{{\sf post}}}

\newcommand{\Test}[0]{\mathit{t}}
\newcommand{\tstat}[0]{\mathit{r}}

\newcommand{\Dtx}[1]{\mathit{D}_{\Test,#1}}
\newcommand{\Dtix}[1]{\mathit{D}_{\Test_i,#1}}
\newcommand{\Dtjx}[1]{\mathit{D}_{\Test_b,#1}}
\newcommand{\Dtax}[1]{\mathit{D}_{\Test_1,#1}}
\newcommand{\Dtbx}[1]{\mathit{D}_{\Test_2,#1}}
\newcommand{\Dtp}[0]{\mathit{D}_{\Test,\phi}}

\newcommand{\Dtpo}[0]{\mathit{D}_{\Test,\phi_0}}

\newcommand{\Dthp}[0]{\mathit{D}_{\Test_\theta,\phi_0}}

\newcommand{\unlikely}[0]{\preccurlyeq}
\newcommand{\unlikelys}[1]{\unlikely^{(#1)}}
\newcommand{\unlikelyst}[2]{\unlikely^{(#1)}_{#2}}

\newcommand{\nats}{\mathbb{N}}
\newcommand{\ints}{\mathbb{Z}}
\newcommand{\reals}{\mathbb{R}}
\newcommand{\realsnng}{\mathbb{R}_{\ge0}}
\newcommand{\cala}{\mathcal{A}}

\newcommand{\cali}{\mathcal{I}}

\newcommand{\calo}{\mathcal{O}}
\newcommand{\calp}{\mathcal{P}}
\newcommand{\calr}{\mathcal{R}}
\newcommand{\cals}{\mathcal{S}}

\newcommand{\calw}{\mathcal{W}}

\newcommand{\ov}[1]{\overline{#1}}

\newcommand{\Normal}[0]{\mathit{N}}

\newcommand{\es}[0]{\mathit{es}}

\newcommand{\Prog}[0]{{\sf Prog}}
\newcommand{\mytrue}{\mathtt{true}}
\newcommand{\myfalse}{\mathtt{false}}

\newcommand{\myskip}{\mathtt{skip}}
\newcommand{\myIf}[3]{\mathtt{if} \ #1 \ \mathtt{then}\ #2\ \mathtt{else}\ #3}

\newcommand{\myIfsubstack}[3]{\begin{scriptsize}\begin{array}[b]{l}\!\mathtt{if} \ #1 \!\\[0.1ex] \!\mathtt{then}\ #2 \!\\[0.1ex] \!\mathtt{else}\ #3\!\end{array}\end{scriptsize}}
\newcommand{\myLoop}[2]{\mathtt{loop}\ #1\ \mathtt{do}\ #2}

\newcommand{\ac}[0]{\mathit{a}}
\newcommand{\testHist}[0]{\mathit{H}}
\newcommand{\Act}[0]{\mathsf{Act}}
\newcommand{\Cmd}{\mathsf{Cmd}}
\newcommand{\Smpl}[0]{\mathsf{Smpl}}
\newcommand{\transX}[1]{\xrightarrow{#1}}
\newcommand{\transA}{\transX{\ac}}

\newcommand{\alg}[0]{\mathit{A}}
\newcommand{\algp}[1]{\alg_{#1}}
\newcommand{\algt}[1]{\alg^{(#1)}}

\newcommand{\ListTests}[1]{\mathit{\ell}_{#1}}

\newcommand{\err}{\mathit{\epsilon}}
\newcommand{\side}{\mathit{s}}

\newcommand{\ELHT}[0]{{\sf ELHT}}

\newcommand{\ATerm}[0]{\mathsf{ATerm}}

\makeatletter
\providecommand{\leftsquigarrow}{%
  \mathrel{\mathpalette\reflect@squig\relax}%
}
\newcommand{\reflect@squig}[2]{%
  \reflectbox{$\m@th#1\rightsquigarrow$}%
}
\makeatother
\makeatletter
\newcommand{\subalign}[1]{%
  \vcenter{%
    \Let@ \restore@math@cr \default@tag
    \baselineskip\fontdimen10 \scriptfont\tw@
    \advance\baselineskip\fontdimen12 \scriptfont\tw@
    \lineskip\thr@@\fontdimen8 \scriptfont\thr@@
    \lineskiplimit\lineskip
    \ialign{\hfil$\m@th\scriptstyle##$&$\m@th\scriptstyle{}##$\hfil\crcr
      #1\crcr
    }%
  }%
}
\makeatother
\newcommand{\sampled}[3]{#1 \mathbin{\leftsquigarrow_{\substack{\\[0.1ex]\hspace{-1.5ex}#3\hspace{0.2ex}}\!}} #2}

\newcommand{\followed}[2]{#1 \leftarrowtail #2}
\newcommand{\TestXx}[2]{\tau_{#1}^{#2}}
\newcommand{\TestA}[1]{\tau_{#1}}
\newcommand{\TestE}[1]{\TestXx{#1}{<}}
\newcommand{\TestLE}[1]{\TestXx{#1}{\le}}

\newcommand{\TestLeE}[1]{\TestXx{#1}{\le}}
\newcommand{\TestBowE}[1]{\TestXx{#1}{\bowtie}}

\newcommand{\TestBowCnegphi}{\TestBowE{y,\algp{\nuh\phi}}}
\newcommand{\NegXx}[2]{\nu_{#1}^{#2}}
\newcommand{\NegA}[1]{\NegXx{#1}{}}
\newcommand{\NegE}[1]{\NegXx{#1}{<}}
\newcommand{\NegLeE}[1]{\NegXx{#1}{\le}}
\newcommand{\NegBowE}[1]{\NegXx{#1}{\bowtie}}

\newcommand{\NegBowEya}{\NegBowE{y,\alg}}

\newcommand{\kappaX}[1]{\kappa_{#1}}

\newcommand{\kappaYA}[1]{\kappa_{#1,\alg}}
\newcommand{\kappaS}{\kappa_{S}}

\newcommand{\cmpds}[1]{\varpi_{#1}}

\renewcommand{\implies}{\rightarrow}

\newcommand{\Know}{\mathbf{K}}
\newcommand{\Possible}{\mathbf{P}}

\newcommand{\KnowXx}[2]{\mathop{\mathbf{K}^{#1}_{#2}}}
\newcommand{\KnowXt}[1]{\mathop{\mathbf{K}^{#1}_{y,\alg}}}

\newcommand{\PossibleXx}[2]{\mathop{\mathbf{P}^{#1}_{#2}}}

\newcommand{\KnowAt}{\KnowXx{\alpha}{y,\alg}}

\newcommand{\KnowE}{\KnowXx{\err}{y,\alg}}
\newcommand{\KnowLeEt}{\KnowXx{{<} \err}{y,\alg}}
\newcommand{\KnowEqEt}{\KnowXx{{=}\err}{y,\alg}}
\newcommand{\KnowBowEt}{\KnowXx{\bowtie \err}{y,\alg}}
\newcommand{\KnowBowEphi}{\KnowXx{\bowtie \err}{y,\algp{\nuh\phi}}}

\newcommand{\PossibleE}{\PossibleXx{\err}{y,\alg}}
\newcommand{\PossibleLeEt}{\PossibleXx{{<} \err}{y,\alg}}

\newcommand{\PossibleBowEt}{\PossibleXx{\bowtie \err}{y,\alg}}

\newcommand{\Pred}[0]{\mathsf{Pred}}

\newcommand{\Fsym}[0]{\mathsf{Fsym}}

\newcommand{\Fml}[0]{\mathsf{Fml}}

\newcommand{\propNu}{\textbf{SB$\nu$}}
\newcommand{\propSBfour}{\textbf{SB4}}
\newcommand{\propSBfive}{\textbf{SB5}}
\newcommand{\propSBk}{\textbf{SBk}}
\newcommand{\propSBdeg}{\textbf{SB-$<$}}
\newcommand{\propSBf}{\textbf{SBf}}
\newcommand{\propBHT}{\textbf{BHT}}
\newcommand{\propBHk}{\textbf{BH$\kappa$}}
\newcommand{\propBHTand}{\textbf{BHT-$\land$}}
\newcommand{\propBHTor}{\textbf{BHT-$\lor$}}

\newcommand{\var}{\mathsf{Var}}
\newcommand{\varO}{\mathsf{Var}_{\mathsf{obs}}}
\newcommand{\varI}{\mathsf{Var}_{\mathsf{inv}}}
\newcommand{\IntVar}{\mathsf{IntVar}}
\newcommand{\Env}{\mathsf{Env}}
\newcommand{\env}{\Gamma}
\newcommand{\envO}{\env^{\mathsf{obs}}}
\newcommand{\envI}{\env^{\mathsf{inv}}}

\newcommand{\typelist}[1]{\mathop{\mathsf{list}} #1}

\newcommand{\triple}[3]{\{#1\}\ #2\ \{#3\}}

\newcommand{\axSkip}{\textsc{Skip}}
\newcommand{\axUpdVar}{\textsc{UpdVar}}
\newcommand{\ruleSeq}{\textsc{Seq}}
\newcommand{\ruleIf}{\textsc{If}}
\newcommand{\ruleLoop}{\textsc{Loop}}
\newcommand{\ruleConseq}{\textsc{Conseq}}
\newcommand{\axHist}{\textsc{Hist}}
\newcommand{\axPAR}{\textsc{Par}}

\newcommand{\axTT}{\textsc{Two-HT}}
\newcommand{\axLT}{\textsc{Low-HT}}
\newcommand{\axUT}{\textsc{Up-HT}}
\newcommand{\ruleOR}{\textsc{Mult-}{\mbox{$\lor$}}}
\newcommand{\ruleAND}{\textsc{Mult-}{\mbox{$\land$}}}

\newcommand{\M}{\mathfrak{M}}
\newcommand{\obs}{\mathit{obs}}  

\newcommand{\relO}{\calr}

\newcommand{\mem}{m}
\newcommand{\sigmaw}{\mem_{w}}
\newcommand{\sigmawp}{\mem_{w'}}

\newcommand{\sem}[1]{\mbox{$[\![ #1 ]\!]$}}

\newcommand{\fv}[1]{\mathsf{fv}(#1)}
\newcommand{\upd}[1]{\mathsf{upd}(#1)}
\newcommand{\qfun}[1]{q_{#1}}

\newcommand{\wpcond}[0]{\mathsf{wp}}
\newcommand{\wpc}[1]{\wpcond^{#1}}
\newcommand{\wpf}[2]{\mathit{F}^{#1}_{#2}}
\newcommand{\modelsi}[0]{\mathbin{\models^{\cali}}}

\newcommand{\phack}{\mathrm{hack}}

\newif\ifcommentson\commentsonfalse

\newif\ifconferenceon\conferenceontrue
\ifconferenceon
\newcommand{\arxiv}[1]{}
\newcommand{\conference}[1]{#1}
\newcommand{\conferenceShort}[1]{}
\else
\newcommand{\arxiv}[1]{#1}
\newcommand{\conference}[1]{}
\newcommand{\conferenceShort}[1]{}
\fi

\begin{document}

\begin{frontmatter}

\title{Sound and Relatively Complete Belief Hoare Logic for Statistical Hypothesis Testing Programs}

\author[mymainaddress,mysecondaryaddress]{Yusuke Kawamoto\corref{mycorrespondingauthor}}
\cortext[mycorrespondingauthor]{Corresponding author}

\author[mythirdaddress]{Tetsuya Sato}

\author[myfourthaddress]{Kohei Suenaga}

\address[mymainaddress]{AIST, Tokyo, JAPAN.}
\address[mysecondaryaddress]{PRESTO, JST, Japan}
\address[mythirdaddress]{Tokyo Institute of Technology, Japan}
\address[myfourthaddress]{Kyoto University, Japan}

\begin{abstract}
We propose a new approach to formally describing the requirement for statistical inference and checking whether a program uses the statistical method appropriately. Specifically, we define belief Hoare logic (BHL) for formalizing and reasoning about the statistical beliefs acquired via hypothesis testing. This program logic is sound and relatively complete with respect to a Kripke model for hypothesis tests. We demonstrate by examples that BHL is useful for reasoning about practical issues in hypothesis testing. In our framework, we clarify the importance of prior beliefs in acquiring statistical beliefs through hypothesis testing, and discuss the whole picture of the justification of statistical inference inside and outside the program logic.
\end{abstract}

\begin{keyword}
knowledge representation \sep
epistemic logic \sep
program logic \sep
Kripke model \sep
statistical hypothesis testing
\MSC[2010] 00-01\sep  99-00
\end{keyword}

\end{frontmatter}

\pagestyle{plain}
\thispagestyle{plain}

\section{Introduction}
\label{sec:intro}
Statistical inferences have been increasingly used to derive and justify scientific knowledge 
in a variety of academic disciplines, from natural sciences to social sciences.
This has significantly raised the importance of statistics, but also brought concerns about the inappropriate procedure and the incorrect interpretation of statistics in scientific research.
Notably, previous studies have pointed out that 
many research articles in biomedical science 
contain severe errors in applying statistical methods and interpreting their outcomes~\cite{Lang:14:inbook}.
Furthermore, large proportions of these errors have been reported for basic statistical methods, possibly performed by researchers who can use only elementary techniques.
In particular, the concept of \emph{statistical significance}, evaluated using \emph{$p$-values}, has been commonly misused and misinterpreted~\cite{Wasserstein:16:AS}. 

Key factors underlying these human errors are that (i) the requirements for statistical inference are typically implicit or unrecognized, and that (ii) the logical aspects of statistical inference are described informally in natural language and  
handled manually by analysts who may not fully understand the statistical methods.
As a result, analysts may overlook some of the assumptions necessary for statistical methods, hence choosing inappropriate methods.
Nevertheless, to our knowledge, no prior work on formal methods has
specified the preconditions for statistical inference programs
or verified the choice of statistical techniques in programs.

In this paper, we propose a method for formalizing and reasoning about statistical inference using symbolic logic.
Specifically, we introduce sound and relatively complete \emph{belief Hoare logic} (\emph{BHL}) to formalize the statistical beliefs acquired via hypothesis tests, and to prevent errors in the choice of hypothesis tests by describing their preconditions explicitly.
We demonstrate by examples that this logic can be used to reason about practical issues concerning statistical inference.

\subsection{Contributions}
\label{sub:contributions}
Our main contributions are as follows:
\begin{itemize}
\item We propose a new approach to formalizing and reasoning about statistical inference in a program.
In particular, this approach formalizes and checks the requirement for statistical methods to be used appropriately in a program.
\item We define an epistemic language to express statistical beliefs obtained by hypothesis tests on datasets.
Specifically, we formalize a statistical belief in a hypothesis $\phi$ as the knowledge that either (i) $\phi$ holds, (ii) the sampled dataset is unluckily far from the population, or (iii) the population does not satisfy the requirements for the hypothesis test.
Then we introduce a Kripke model for hypothesis tests to define the interpretation of this language.
\item Using this epistemic language, we construct belief Hoare logic (BHL) for reasoning about statistical hypothesis testing programs.
Then we prove that BHL is sound and relatively complete w.r.t. the Kripke model for hypothesis tests.
\item 
We clarify the importance of prior beliefs in acquiring statistical beliefs, 
and prove essential properties of statistical beliefs by using our framework.
\item We show that BHL is useful for reasoning about practical issues concerning statistical inference, such as $p$-value hacking and multiple comparison problems.
\item We provide the whole picture of the justification of statistical beliefs acquired via hypothesis tests inside and outside BHL.
In particular, we discuss 
the empirical conditions for hypothesis tests and the epistemic aspects of statistical inference.
\end{itemize}

To the best of our knowledge, this appears to be the first attempt to introduce a program logic that can specify the requirements for hypothesis tests to be applied appropriately.
We consider this as the first step to building a framework for formalizing and verifying the validity of empirical science and data-driven artificial intelligence.

\subsection{Relation with the Preliminary Version}
\label{sub:relation}

A preliminary version of this work, considering only sound but \emph{not} complete program logic, appeared in~\cite{Kawamoto:21:KR}.
The main novelties of this paper to that version are:
\begin{itemize}
\item We introduce a sound and \emph{relatively complete} belief Hoare logic (BHL) that has a simpler set of axioms and inference rules than our preliminary version~\cite{Kawamoto:21:KR}.
\item We extend the notion of a possible world with a hypothesis test history, and redefine the assertion language. This enables us to provide a more rigorous model for statistical beliefs and to prove the relative completeness of BHL.
\item We add propositions and discussion for hypothesis formulas and show the importance of prior beliefs in hypothesis testing by using our framework in \Sec{sec:properties}.
\item We present all proofs for our technical results
in \App{sec:proofs}.
\end{itemize}

\subsection{Related Work}
\label{sub:related}

\emph{Hoare logic}~\cite{Hoare:69:CACM,winskel} is a form of program logic for an imperative programming language.
This program logic is then extended and adapted so that it can handle various types of programs and assertions~\cite{Apt:19:FAOC},
including
heap-manipulating programs~\cite{separationLogic},
hybrid systems~\cite{infinitesimalKSUenaga}, and
probabilistic programs~\cite{hoareForProbability}.
Atkinson and Carbin propose an extension of Hoare logic
with epistemic assertions~\cite{epistemicHoareLogic}.
In their work, an epistemic assertion is used to reason about
the belief of a program about a partially observable environment,
whereas their logic does not deal with a belief
arising from statistical tests conducted in a program.
To the best of our knowledge, ours appears to be the first program logic that formalizes
the concept of statistical beliefs in hypothesis testing.

\emph{Epistemic logic}~\cite{vonWright:51:book,Hintikka:62:book} is a branch of modal logic for reasoning about knowledge and belief~\cite{Fagin:95:book}.
It has been used to specify and verify various knowledge properties in systems, e.g., authentication~\cite{Burrows:90:TOCS} and anonymity~\cite{Syverson:99:FM,Garcia:05:FMSE}.
Many previous works on epistemic logic incorporate
probabilistic notions of beliefs~\cite{Halpern:03:book} and
certain notions of degrees of beliefs and confidence~\cite{Huber:08:book}.
Notably, Bacchus \etal{}~\cite{Bacchus:99:AI} define the degree of belief in a possible world semantics where each world is associated with a weight and the degree of belief in a formula $\phi$ is defined as the normalized sum of the weights of all accessible possible worlds satisfying $\phi$.
However, this line of studies has not modeled the degree of belief in the sense of statistical significance in a hypothesis test.
In contrast, our framework models the degree of belief in terms of a $p$-value without assigning a weight to a possible world.

The first attempt to express statistical properties of hypothesis tests using modal logic is the work on statistical epistemic logic (StatEL)~\cite{Kawamoto:19:FC,Kawamoto:19:SEFM,Kawamoto:20:SoSyM}.
They introduce a belief modality weaker than S5, and a Kripke model with an accessibility relation defined using a statistical distance between possible worlds.
Unlike our work, however, StatEL cannot describe the procedures of statistical methods or reason about their correctness.

\emph{Dynamic epistemic logic (DEL)}~\cite{Van:07:book-dynamic} is a branch of modal logic for reasoning about the changes of knowledge and belief when events take place.
For a precondition $\psi$, a postcondition $\phi$, and a terminating program $C$, a DEL formula $\psi \rightarrow [C] \phi$ with a single agent expresses a Hoare triple $\triple{\psi}{C}{\phi}$. Therefore, DEL may be extended to deal with hypothesis testing programs by incorporating the modal operators and predicate symbols for statistical notions introduced in our paper.

\emph{Fuzzy logic}~\cite{Zadeh:65:IC} is a branch of many-valued logic where the truth values range over $[0, 1]$.
It has been used to model and reason about the degrees of uncertainty in beliefs and confidence~\cite{Nguyen:18:book,Eberhart:21:TAP}.
To the best of our knowledge, however, no prior work on fuzzy logic can reason about the correct application of statistical hypothesis testing.

\emph{Default logic}~\cite{Reiter:80:AIJ} is a branch of logic for \emph{reasoning-by-default}, in which the absence of evidence for an exception leads to a conclusion by default.
By extending default logic, a few studies~\cite{Kyburg:02:NMR,Kyburg:06:CI} have formalized the reasoning in hypothesis tests.
In particular, they manage to syntactically deal with the \emph{non-monotonicity}~\cite{Kyburg:99:IJPRAI} of statistical reasoning to formalize that a conclusion of a hypothesis test may be retracted on the basis of further data.
However, since these extensions of default logic do not allow for describing programs, they do not derive the correctness of hypothesis testing programs or methods.
Furthermore, these previous studies do not provide a semantics for their default logic based on statistical models and do not prove the soundness of the logic in terms of statistics.
In contrast, we prove the soundness and relative completeness of our logic w.r.t. a possible world semantics extended with statistics.

We remark that reasoning-by-default is not necessary to formalize the correctness of hypothesis testing methods in our program logic.
This is because, given a (mathematical) statistical model, the requirements for hypothesis testing methods can be formally expressed as assumptions 
equipped with belief modality in our assertion language, without requiring the notion of \emph{justifications} in default logic.
Furthermore, we handle the non-monotonicity of statistical reasoning using our programming language and its operational semantics.
Instead of dealing directly with non-monotonic statements about $p$-values,
we introduce a \emph{test history} that grows monotonically with the executions of hypothesis test commands in a Kripke model.

From a broader perspective, many studies formalize and reason about programs based on knowledge~\cite{Fagin:95:PODC} and beliefs~\cite{Laverny:05:synthese}.
For example, 
Sardina and Lesp\'{e}rance~\cite{Sardina:09:PROMAS} extend the situation calculus-based agent programming language GOLOG~\cite{Levesque:97:JLP} with BDI (belief-desire-intention)~\cite{Bratman:87:book} agents.
Belle and Levesque~\cite{BelleL:15:IJCAI} propose a belief-based programming language called ALLEGRO to deal with the probabilistic degrees of beliefs in programs with noisy acting and sensing.
However, no prior work appears to have studied belief-based programs involving statistical  hypothesis testing.

\subsection{Plan of the Paper}

In \Sec{sec:preliminaries}, we review fundamental concepts from statistical hypothesis testing.
In \Sec{sec:overview}, we present an illustrating example to explain the basic ideas of our framework.
In \Sec{sec:model}, we introduce a Kripke model for describing statistical properties and define hypothesis testing.
In \Sec{sec:simple-language}, we introduce the syntax and the semantics of an imperative programming language $\Prog$.
In \Sec{sec:assertion-logic}, we define an assertion language, called \emph{epistemic language for hypothesis testing} (\ELHT{}), that can express statistical beliefs.
In \Sec{sec:properties}, we clarify the importance of prior beliefs in acquiring statistical beliefs, and show the essential properties of statistical beliefs in our framework.
In \Sec{sec:belief-hoare-logic}, we introduce \emph{belief Hoare logic} (BHL) for formalizing and reasoning about statistical inference using hypothesis tests.
Then we show the soundness and relative completeness of BHL.
In \Sec{sec:reasoning-with-BHL}, we apply our framework to the reasoning about \emph{$p$-value hacking} and \emph{multiple comparison problems} using BHL.
In \Sec{sec:discuss}, we provide the whole picture of the justification of statistical beliefs inside and outside BHL.
In \Sec{sec:conclude}, we present our final remarks.

In \App{sub:instantiation}, we present examples of the instantiations of derived rules with concrete hypothesis testing methods.
In \App{sec:proofs}, we show the proofs for the propositions about assertions, 
basic results on structural operational semantics,
remarks on parallel compositions, and
the proofs for BHL's soundness and relative completeness.

\section{Preliminaries}
\label{sec:preliminaries}
In this section, we introduce notations used in this paper and recall background on statistical hypothesis testing~\cite{Hogg:04:book:ims,Kanji:06:book:100stat}.

Let $\nats$, 
$\reals$, $\realsnng$ be the sets of non-negative integers, 
real numbers, and non-negative real numbers, respectively.
Let $[0, 1]$ be the set of non-negative real numbers less than or equal to $1$.
We denote 
the set of all finite vectors of elements in $\cals$ by $\cals^*$,
the \emph{set of all multisets of elements in $\cals$ by $\calp(\cals)$}, and
the \emph{set of all probability distributions} over a set~$\cals$ by $\Dists\cals$.
Given two distributions $D_1\in\Dists\cals_1$ and $D_2\in\Dists\cals_2$, 
a \emph{coupling} of $D_1$ and $D_2$ is a joint distribution $D \in \Dists(\cals_1\times\cals_2)$ 
whose marginal distributions $\sum_{s_2\in\cals_2} D[s_1,s_2]$ and $\sum_{s_1\in\cals_1} D[s_1,s_2]$
are identical to $D_1$ and $D_2$, respectively.

\subsection*{Statistical Hypothesis Testing}
\label{sub:hypothesis-testing}
\emph{Statistical hypothesis testing} is a method of statistical inference about an unknown \emph{population} $x$ (the collection of items of interest) on the basis of a \emph{dataset} $y$ sampled from $x$.
In a hypothesis test, 
an \emph{alternative hypothesis}~$\phi_1$ is a proposition that we wish to prove about the population $x$, and
a \emph{null hypothesis} $\phi_0$ is a proposition that contradicts $\phi_1$.
The goal of the hypothesis test is to determine whether we \emph{accept} the alternative hypothesis~$\phi_1$ by \emph{rejecting} the null hypothesis $\phi_0$.

In a hypothesis test, we calculate a \emph{test statistic} $\Test(y)$ from a dataset $y$, and 
see whether the $\Test(y)$ value contradicts the assumption that the null hypothesis $\phi_0$ is true.
Specifically, we calculate the \emph{$p$-value}, showing the degree of likeliness of obtaining $\Test(y)$ when the null hypothesis $\phi_0$ is true.
If the $p$-value is smaller than a threshold (e.g., $0.05$), 
we regard the dataset $y$ is unlikely to be sampled from the population satisfying the null hypothesis $\phi_0$,
hence we reject $\phi_0$ and accept the alternative hypothesis~$\phi_1$.

A hypothesis test is based on a \emph{statistical model} $\popl(\xiPPL, \theta)$ with unknown parameters $\xiPPL$, known parameters $\theta$, and
(assumed) probability distributions of the parameters $\xiPPL$.

\begin{example}[$Z$-test for two population means] \label{eg:hypothesis-test}
As an illustrating example, we present the \emph{two-tailed $Z$-test} for 
means of two populations.
We introduce its statistical model as two normal distributions
$\Normal(\mu_{\ppl1},\sigma^2)$ and $\Normal(\mu_{\ppl2},\sigma^2)$ with a known variance $\sigma^2$ and unknown true means $\mu_{\ppl1},\mu_{\ppl2}$.
Let $y_1$ and $y_2$ be two given datasets where each data value was sampled from $\Normal(\mu_{\ppl1},\sigma^2)$ and $\Normal(\mu_{\ppl2},\sigma^2)$, respectively.

In the $Z$-test, we wish to prove the alternative hypothesis $\phi_1 \eqdef (\mu_{\ppl1} \neq \mu_{\ppl2})$ by rejecting the null hypothesis $\phi_0 \eqdef (\mu_{\ppl1} = \mu_{\ppl2})$.
First, we calculate the $Z$-test statistic:
\[
\Test(y_1,y_2) = \frac{\mean(y_1) - \mean(y_2)}{\sigma\sqrt{\nicefrac{1}{\size{y_1}} + \nicefrac{1}{\size{y_2}}}}
\]
where for $b=1,2$, $\size{y_b}$ is the sample size of the dataset $y_b$ and $\mean(y_b)$ is the mean of all data in $y_b$.
Then we calculate the \emph{$p$-value}
\[
\Pr_{(d_1,d_2) \sim \Normal(\mu_{\ppl1},\sigma^2)\times \Normal(\mu_{\ppl1},\sigma^2)}[\, |\Test(d_1,d_2)| > |\Test(y_1,y_2)| \,]
\]
under the  null hypothesis $\phi_0$.
When the $p$-value is small enough, the datasets $y_1$ and $y_2$ are unlikely to be sampled from the same distribution, i.e., the null hypothesis $\mu_{\ppl1} = \mu_{\ppl2}$ is unlikely to hold.
Hence, in the $Z$-test, if the $p$-value is smaller than a certain threshold (e.g., $0.05$), we reject the null hypothesis $\phi_0$ and accept the alternative hypothesis $\phi_1$.

When we have prior knowledge of $\mu_{\ppl1} \ge \mu_{\ppl2}$ (resp. $\mu_{\ppl1} \le \mu_{\ppl2}$), then we apply the \emph{upper-tailed} (resp. \emph{lower-tailed}) $Z$-test with the alternative hypothesis $\mu_{\ppl1} > \mu_{\ppl2}$ (resp. $\mu_{\ppl1} < \mu_{\ppl2}$) and the null hypothesis $\mu_{\ppl1} = \mu_{\ppl2}$, and obtain the $p$-value
$\Pr[\Test(d_1,d_2) > \Test(y_1,y_2)]$ (resp. $\Pr[\Test(d_1,d_2) < \Test(y_1,y_2)]$).

For more details, the readers are referred to standard textbooks, e.g.,~\cite{Hogg:04:book:ims,Kanji:06:book:100stat}.
\begin{table}[t]
   \centering
   \caption{Hypotheses in the $Z$-tests.}
   \label{tab:hypo:Z}
   \begin{tabular}{@{} lccc @{}} 
      \toprule
      Tails & prior knowledge 
      & alternative hypothesis $\phi_1$ & null hypothesis $\phi_0$ \\[-0.2ex]
      \midrule
      Two    & nothing & $\mu_{\ppl1} \neq \mu_{\ppl2}$ & $\mu_{\ppl1} = \mu_{\ppl2}$ \\
      Upper & $\mu_{\ppl1} \ge \mu_{\ppl2}$ & $\mu_{\ppl1} > \mu_{\ppl2}$ & $\mu_{\ppl1} = \mu_{\ppl2}$ \\
      Lower & $\mu_{\ppl1} \le \mu_{\ppl2}$ & $\mu_{\ppl1} < \mu_{\ppl2}$ & $\mu_{\ppl1} = \mu_{\ppl2}$ \\
      \bottomrule
   \end{tabular}
\end{table}
\end{example}

\section{Illustrating Example}
\label{sec:overview}
Throughout the paper, we use the following simple illustrating example to explain the basic ideas of our framework.

\begin{example}[Comparison tests on drugs] \label{eg:illustrate}
Let us consider three drugs 1, 2, 3 that may decrease blood pressure.
To compare the efficacy of these drugs, we perform experiments and obtain a set $y_i$ of the reduced values of blood pressure after taking drug $i$.
Then we apply hypothesis tests on the dataset $y = (y_1, y_2, y_3)$.
We assume that the data values in $y_i$ have been sampled from the population that follows a normal distribution $\Normal(\mu_i, \sigma^2)$ with a mean $\mu_i$ and a variance $\sigma^2$.
For simplicity, we consider the situation where we know the variance $\sigma^2$ but do not know the means $\mu_i$.

Suppose that drug 1 is composed of drugs 2 and 3, and we investigate whether drug 1 has better efficacy than \emph{both} drugs 2 and 3.
Then we take the following procedure:
\begin{itemize}
\item We first compare drugs 1 and 2 concerning the average decreases in blood pressure.
We apply a two-tailed $Z$-test $\alg_{12}$ (\Eg{eg:hypothesis-test}) to see whether the means of the populations are different, i.e., $\mu_1 \neq \mu_2$.
In this test, the alternative hypothesis $\phi_{12}$ is the inequality $\mu_1 \neq \mu_2$, and the null hypothesis $\neg \phi_{12}$ is $\mu_1 = \mu_2$.
\item 
Let $\alpha_{ij}$ be the $p$-value when only comparing drugs $i$ and~$j$.
\item If $\alpha_{12} \ge 0.05$, the $Z$-test $\alg_{12}$ does not reject the null hypothesis $\neg \phi_{12}$ and concludes that the efficacy of drugs 1 and 2 may be the same.
Then we are not interested in drug 1 any more, and skip the comparison with drug 3.
\item  If $\alpha_{12} < 0.05$, the $Z$-test $\alg_{12}$ rejects the null hypothesis $\neg \phi_{12}$ and concludes that the alternative hypothesis $\phi_{12}$ is true.
Then we apply another $Z$-test $\alg_{13}$ to check whether the alternative hypothesis $\phi_{13} \eqdef \allowbreak (\mu_1 \neq \mu_3)$ is true.
\item Finally, we calculate the $p$-value of the combined test $\alg$ consisting of $\alg_{12}$ and $\alg_{13}$, with the \emph{conjunctive} alternative hypothesis $\phi_{12} \land \phi_{13}$.
\end{itemize}
\end{example}

\smallskip
\noindent{\textbf{Overview of the Framework.}}
In our framework, we describe a procedure of statistical tests as a program using a programming language (\Sec{sec:simple-language});
in \Eg{eg:illustrate}, we denote the $Z$-test program comparing drugs $i$ with $j$ by $C_{ij}$, and the whole procedure~by:
\begin{align} \label{eq:illustrate:program}
C_{\rm drug} \eqdef 
C_{12}; \myIf{\alpha_{12} < 0.05}{C_{13}}{\myskip{}}
\end{align}

Then we use an assertion logic (\Sec{sec:assertion-logic}) to describe the requirement for the hypothesis tests as a \emph{precondition} formula.
In \Eg{eg:illustrate}, the precondition is given by:
\begin{align*} 
&\psiPre \eqdef 
\hspace{-1ex}
\bigwedge_{i=1,2,3}\hspace{-1ex} 
\sampled{y_i}{\Normal(\mu_i, \sigma^2)}{n_i}\, \land
\Possible (\phi_{12} \land \phi_{13})
\land \kappaX{\emptyset}.
\end{align*}
In this formula, 
$\sampled{y_i}{\Normal(\mu_i, \sigma^2)}{n_i}$
represents that a set $y_i$ of $n_i$ data is sampled from the population that follows the normal distribution $\Normal(\mu_i, \sigma^2)$.
The modal formula $\Possible (\phi_{12} \land \phi_{13})$ represents that before conducting the hypothesis tests, we have the \emph{prior belief} that the alternative hypothesis $\phi_{12} \land \phi_{13}$ \emph{may be} true (see \Sec{sec:properties} for discussion).
The formula $\kappaX{\emptyset}$ represents that no hypothesis test has been conducted previously.

The statistical belief we want to acquire is specified as a \emph{postcondition} formula.
In \Eg{eg:illustrate}, 
the postcondition is:
\begin{align} \label{eq:illustrate:Hoare:post}
\phiPost \eqdef
\KnowXx{\le0.05}{y, A_{12}}\! \phi_{12} \rightarrow
\KnowXx{\le\min(\alpha_{12},\alpha_{13})}{y, A} (\phi_{12} \land \phi_{13}).
\end{align}
Intuitively, by testing on the dataset $y$,
when we believe $\phi_{12}$ with a $p$-value $\alpha \le 0.05$, we believe the combined hypothesis $\phi_{12} \land \phi_{13}$
with a $p$-value at most $\min(\alpha_{12},\alpha_{13})$.

Finally, we combine all the above and describe the whole statistical inference as a \emph{judgment}.
In \Eg{eg:illustrate}, we write:
\begin{align} \label{eq:illustrate:Hoare}
\env\vdash\triple{ \psiPre }{ C_{\rm drug} }{ \phiPost }.
\end{align}
By proving this judgment using rules in BHL (\Sec{sec:belief-hoare-logic}), we conclude that the statistical inference is appropriate.

We remark that the $p$-value can be larger for a different purpose of testing.
Suppose that in \Eg{eg:illustrate}, drug 1 was a new drug and we wanted to find out that it had better efficacy than \emph{at least one} of drugs 2 and~3.
Then the procedure is:
\begin{align} \label{eq:illustrate:program:disjunctive}
C_{\rm multi} \eqdef C_{12} \,\|\, C_{13},
\end{align}
and the alternative hypothesis is $\phi_{12} \lor \phi_{13}$
with a $p$-value greater than $\alpha_{12}$ and $\alpha_{13}$ (at most $\alpha_{12} + \alpha_{13}$).
This is the \emph{multiple comparisons problem}~\cite{Bretz:10:book}, arising when the combined alternative hypothesis is \emph{disjunctive}.
We explain more details in \Sec{sec:belief-hoare-logic}.

\section{Model}
\label{sec:model}
In this section, we introduce a Kripke model for describing statistical properties and formally define hypothesis tests.

\subsection{Variables, Data, and Actions}
\label{sub:variables-data}

We introduce a finite set $\var$ of variables comprised of two disjoint sets of \emph{invisible variables} and of \emph{observable variables}: $\var = \varI \cup \varO$.
We can directly observe the values of 
the latter, but not those of the former.
Throughout the paper, we use 
$y$ as an observable variable denoting a dataset sampled from the population.

We write $\calo$ for the set of all \emph{data values} that consists of the Boolean values, integers, real numbers, distributions of data values, and lists of data values.
A \emph{dataset} is a list of lists of data values.
In particular, we deal with a list of real vectors as a dataset.
Then the vectors range over $\dompop \eqdef \reals^l$ for an $l \in\nats$.
A distribution over a population has type $\Dists\dompop$,
and a dataset has type $\typelist{\dompop}$.
We remark that  distributions and datasets are elements of $\calo$; i.e., $\Dists\dompop \subseteq \calo$ and $\typelist{\dompop} \subseteq \calo$.
$\bot$ denotes the undefined value.

We write $d \sim D^n$ for the sampling of a set $d$ of $n$ data from a population $D$ where all these data are independent and identically distributed (i.i.d.).
Let $\Smpl$ be a set of i.i.d.~samplings of datasets from populations (e.g., $d \sim D^n$), and $\Cmd$ be a set of program commands
(e.g., $v \detassign e$ and $\myskip{}$).
Then we define an \emph{action} as a sampling of a dataset or a program command; i.e., $\Act = \Smpl\cup\Cmd$.
In \Sec{sec:simple-language}, we instantiate $\Cmd$ with concrete commands used in a programming language.

\subsection{States and Possible Worlds}
\label{sub:possible-worlds}

We introduce the notions of states and possible worlds equipped with test histories.
We write $\cala$ for a finite set of hypothesis tests we consider.

\begin{definition}[States] \label{def:states} 
A \emph{state} is 
a tuple $(\mem, \ac, \testHist)$ 
consisting of (i) the \emph{current assignment} $\mem: \var \rightarrow \calo \cup \{ \bot \}$ of data values to variables, (ii) the action $\ac\in\Act$ that has been executed in the last transition, and (iii) the \emph{test history} $\testHist: (\typelist{\dompop}) \rightarrow \calp(\cala)$ that maps a dataset $d$ to the \emph{multiset} of all hypothesis tests that have used the dataset $d$.
\end{definition}
We remark that $\testHist(d)$ is a multiset rather than a set, because the same test on the same dataset $d$ can be performed multiple times.

\begin{definition}[Possible worlds] \label{def:worlds} 
A \emph{possible world} $w$ is a sequence of states 
$(w[0], \allowbreak w[1], \ldots, w[k-1])$ where $w[i]$ is the $i$-th state in $w$.
$w[0]$ and $w[k-1]$ are called the \emph{initial state} and the \emph{current state}, respectively.
The length $k$ is denoted by $\len{w}$.
We write $(\mem_{w}, \ac_{w}, \testHist_{w})$ for the current state $w[k-1]$ of a possible world $w$.
We assume that the test history is empty at the initial states.
\end{definition}

Since a possible world records all updates of data values, it can be used to model the updates of knowledge and beliefs.
As with previous works on epistemic logic~\cite{Fagin:95:book},
agents' knowledge and belief are defined from their \emph{observation} of possible worlds.
\begin{definition}[Observation] \label{def:observe} 
The \emph{observation} of a state $w[i] = (\mem, \ac, \testHist)$ is defined by 
$\obs(w[i]) = (\mem^{\sf obs}, \ac, \testHist)$ with
an assignment $\mem^{\sf obs}: \varO \rightarrow \calo \cup \{ \bot \}$ 
such that $\mem^{\sf obs}(v) = \mem(v)$ for all $v\in\varO$, 
and that  $\mem^{\sf obs}(v) = \bot$ for all $v\in\varI$.
The \emph{observation} of a world $w$ is given by
$\obs(w) = (\obs(w[0]), \ldots, \obs(w[k-1]))$.
\end{definition}

\subsection{Kripke Model}
\label{sub:Kripke}

We introduce a \emph{Kripke model with labeled transitions} where two kinds of relations $\transA$ and $\relO$ may relate possible worlds.
A \emph{transition relation} $w\transA w'$ represents a transition from a world $w$ to another $w'$ by performing an action~$\ac$.
An \emph{observability relation} $w \relO w'$ represents that two possible worlds $w$ and $w'$ have the same observation, i.e., $\obs(w) = \obs(w')$.
Then $\relO$ is an equivalence relation.
Furthermore, for any worlds $w$ and $w'$, $w \relO w'$ implies $\testHist_{w} = \testHist_{w'}$.
In \Sec{sec:assertion-logic}, this relation is used to model \emph{knowledge} in the conventional Hintikka style.

\begin{definition}[Kripke model] \label{def:epistemic-Kripke} 
A \emph{Kripke model} 
is a tuple 
$\M =(\calw, (\transA)_{a\in\Act}, \relO, \allowbreak (V_w)_{w\in\calw})$ 
consisting of:
\begin{itemize}
\item a non-empty set $\calw$ of possible worlds;
\item for each $\ac\in\Act$, a transition relation $\transA \subseteq \calw \times \calw$;
\item an observability relation $\relO = \{ (w, w') \in \calw \times \calw \mid \obs(w) = \obs(w') \}$;
\item for each $w\in\calw$, a valuation 
$V_w$ that 
maps a $k$-ary predicate symbol to a set of $k$-tuples of data values.
\end{itemize}
We assume that each world in a model has the same sets $\varI$ and $\varO$ of variables.
\end{definition}

In \Sec{sub:program:semantics}, we instantiate the actions in a Kripke model with concrete program commands described in a programming language, and define the transition relation $\transA$ as the semantic relation $\sem{\ac}$.
For example, in a transition $w \transX{v \detassign 1} w'$, an assignment action $v \detassign 1$ is executed and the resulting state $(\mem_{w'}, \ac_{w'}, \testHist_{w'})$ is $(\sigmaw[ v \mapsto 1],  v \detassign 1, \testHist_{w})$.
This is formally defined as $(w, w') \in \sem{v \detassign 1}$ in \Sec{sub:program:semantics}.

Throughout this paper, we deal with a class $\Pred$ of predicate symbols whose interpretations are identical in any possible world having the same memory;
i.e., for any $\eta \in \Pred$ and any $w, w' \in \calw$, if $\mem_{w} = \mem_{w'}$ then $V_w(\eta) = V_{w'}(\eta)$.

\subsection{Formulation of Hypothesis Testing}
\label{sub:notation-hypothesis-test}

Next, we formalize the notion of hypothesis tests as follows.

\begin{definition}[Hypothesis tests] \label{def:hypothesis-tests} 
We consider a \emph{basic test type} $\side\in\{ \sideL, \sideU, \sideT \}$ each representing a \emph{lower-tailed}, \emph{upper-tailed}, and \emph{two-tailed} test.
A \emph{hypothesis test} is a tuple 
$\algp{\phi_0} = (\phi_0, \Test, \Dtpo, \unlikelys{\side}, \popl)$ 
consisting of:
\begin{itemize}
\item $\phi_0$ is an assertion, called a null hypothesis;
\item $\Test$ is a function that maps a dataset $d\in\typelist{\dompop}$ to its test statistic $\Test(d)$, usually with $\range(\Test) = \reals^k$ for a $k \ge 1$;
\item $\Dtpo\in\Dists(\range(\Test))$ is a probability distribution of the test statistic when the null hypothesis $\phi_0$ is true;
\item $\unlikelyst{\side}{\Test} \in \range(\Test)\times\range(\Test)$ is a \emph{likeliness relation} where for a test type $\side$ and for values $\tstat$ and $\tstat'$ of the test statistic,\, $\tstat \unlikelyst{\side}{\Test} \tstat'$ represents that $\tstat$ is at most as likely as $\tstat'$.
For brevity, we often omit $\Test$ to write ${\unlikelys{\side}}$;
\item 
$\poplxt$ denotes the population following a statistical model $\popl$ with unknown parameters $\xiPPL$ and known parameters $\theta$.
\end{itemize}
For brevity, we 
abbreviate $\algp{\phi}$ as $\alg$.
We denote by $\poplA$ the statistical model $\popl$ of a hypothesis test $\alg$,
and by $\cala$ a finite set of hypothesis tests we consider.
\end{definition}

\begin{example}[The likeliness relation for $Z$-test]\label{eg:Z-test:likeliness}
The two-tailed $Z$-test for means of two populations in Example \ref{eg:hypothesis-test} can be denoted by the following hypothesis test:
\[
\algp{\phi_0} = (\phi_0, \allowbreak \Test, \Normal(0, 1), \allowbreak \unlikelys{\sideT}\!, \Normal(\mu_{\ppl1},\sigma^2)\times \Normal(\mu_{\ppl2},\sigma^2)).
\]
The likeliness relation $\tstat \unlikelys{\sideT} \tstat'$ expresses $|\tstat| \geq |\tstat'|$.
When the null hypothesis $\phi_0$ is true, the test statistic $\Test(y_1,y_2)$ follows the standard normal distribution $\Normal(0, 1)$, hence 
\[
\Pr[ \Test(y_1,y_2) \unlikelys{\sideT} 1.96 ] = \Pr[ |\Test(y_1,y_2)| \geq 1.96] \approx 0.05.
\]
For the upper-tailed (lower-tailed) test, 
with alternative hypothesis $\phiU \eqdef (\mu_{\ppl1} > \mu_{\ppl2})$
(resp. $\phiL \eqdef (\mu_{\ppl1} < \mu_{\ppl2})$),
the likeliness relation $\tstat \unlikelys{\sideU} \tstat'$ (resp. $\tstat \unlikelys{\sideL} \tstat'$) is defined by $\tstat \geq \tstat'$ (resp. $\tstat \leq \tstat'$).
\end{example}

Next, we define \emph{disjunctive/conjunctive combinations} of hypothesis tests.
Intuitively, a disjunctive combination $\algp{\phi_1\lor\phi_2}$ (resp. conjunctive combination $\algp{\phi_1\land \phi_2}$) is a hypothesis test with a null hypothesis $\phi_1\lor\phi_2$ (resp. $\phi_1\land\phi_2$) that performs two hypothesis tests $\algp{\phi_1}$ and $\algp{\phi_2}$ in parallel.
\begin{definition}[Combination of tests] \label{def:combine} 
For $b=1,2$, let 
$\algp{\phi_b} = (\phi_b, \Test_b, \Dtjx{\phi_b}, {\unlikelyst{\side_b}{\Test_b}}\!, \popl_b)$ be two hypothesis tests.
The \emph{disjunctive combination} of $\algp{\phi_1}$ and $\algp{\phi_2}$ is given by 
\[
\algp{\phi_1\lor\phi_2} = (\phi_1 \lor \phi_2, \Test, \Dtx{(\phi_1,\phi_2)}, \unlikelyst{\side_1,\side_2}{\Test}, \allowbreak \popl)
\]
where 
$\Test(y_1,y_2) \allowbreak = (\Test_1(y_1), \Test_2(y_2))$,
$\Dtx{(\phi_1,\phi_2)}$ is a coupling of $\Dtax{\phi_1}$ and $\Dtbx{\phi_2}$ (i.e., it is a joint distribution such that $\Dtax{\phi_1}$ and $\Dtbx{\phi_2}$ are the marginal distributions of $\Dtx{(\phi_1,\phi_2)}$),
$(\tstat_1,\tstat_2) {\unlikelyst{\side_1,\side_2}{\Test}} (\tstat'_1,\tstat'_2)$ iff
either $\tstat_1 \unlikelyst{\side_1}{\Test_1} \tstat'_1$ or $\tstat_2 \unlikelyst{\side_2}{\Test_2} \tstat'_2$, and
$\popl$ is a coupling of $\popl_1$ and $\popl_2$.
Similarly, the \emph{conjunctive combination} of $\algp{\phi_1}$ and $\algp{\phi_2}$ is
\[
\algp{\phi_1\land\phi_2} = (\phi_1 \land \phi_2, \Test, \Dtx{(\phi_1,\phi_2)}, {\unlikelyst{\side_1,\side_2}{\Test}}\!, \popl)
\]
where
$(\tstat_1,\tstat_2) \unlikelyst{\side_1,\side_2}{\Test} (\tstat'_1,\tstat'_2)$ iff
$\tstat_1 \unlikelyst{\side_1}{\Test_1} \tstat'_1$ and $\tstat_2 \unlikelyst{\side_2}{\Test_2} \tstat'_2$.
\end{definition}

Then we define a function $\ListTests{y, \alg}$ to decompose a combined test into individual tests.

\begin{definition} \label{def:decompose}
For a combination $\alg$ of $n$ hypothesis tests $\alg_1, \ldots, \alg_n$ and a tuple of $n$ datasets $y=(y_1, \ldots, y_n)$, 
the multiset of all the pairs of datasets and tests is:
\begin{align} \label{eq:list-tests}
\ListTests{y, \alg} \eqdef \{ (y_i, \alg_i) \mid i = 1, \ldots, n \}.
\end{align}
\end{definition}
For instance, 
$\ListTests{(y_1, y_2), \allowbreak \algp{\phi_1\land \phi_2}} \allowbreak = \{ (y_1,\algp{\phi_1}), \allowbreak (y_2,\algp{\phi_2}) \}$.

\section{A Simple Programming Language}
\label{sec:simple-language}
In this section, we introduce an imperative programming language $\Prog$ to describe programs for hypothesis testing
on a dataset that has already sampled from a population.
This language has the following two features.
First, it has a command for performing a hypothesis test and assigning a test statistic.
This command also updates the history of all previously executed hypothesis tests,
which is used to calculate the $p$-values for single and multiple tests.
Second, $\Prog$ supports parallel compositions of independently running programs
to ensure that multiple hypothesis tests do not interfere with each other.

\subsection{Syntax of \Prog}
\label{sub:program:syntax}

Let $\Fsym$ be the set of all function symbols, 
where constants are dealt as functions with arity $0$.
We define the syntax of \Prog{} by the following BNF:
\begin{align*}
T &\mathbin{::=} 
\mathtt{bool} \mid \mathtt{int} \mid 
\mathtt{real} \mid 
T \times T \mid \mathtt{list}(T) &\text{(Types)}\\
e & \mathbin{::=} v \mid f(e,\ldots,e) & \hspace{-6ex} \text{(Program terms)}\\
c & \mathbin{::=}
  \myskip{}
  \mid v \detassign e
  & \text{(Commands)}\\
C & \mathbin{::=} c
  \mid  C ; C
  \mid  C \mathbin{\|} C
  \mid \myIf{e}{C}{C}
  \mid \myLoop{e}{C}
  & \text{(Programs)}
\end{align*}
where 
$v \in \varO$ and  $f \in \Fsym$.
Then a program can handle only observable variables.

$T$ represents \emph{types}.
A type is either $\mathtt{bool}$ for Boolean values, $\mathtt{int}$ for integers, 
$\mathtt{real}$ for real numbers, 
$T_1 \times T_2$ for pairs consisting of a value of type $T_1$ and a value of type $T_2$, or $\mathtt{list}(T)$ for lists of values of type $T$.
$e$ represents \emph{expressions} that evaluate to values.  An expression is either a variable $v$ or a function call $f(e_1,\dots,e_k)$; the latter is typically a call to a function that computes a test statistic.
$c$ and $C$ represent \emph{commands} and \emph{programs}, respectively.
We give their intuitive explanation as follows.
\begin{itemize}
\item $\myskip{}$ does nothing.
\item $v \detassign e$ updates $v$ with the result of an evaluation of $e$.
\item $C_1 ; C_2$ executes $C_1$ and then $C_2$.
\item $C_1 \| C_2$ executes $C_1$ and $C_2$ in parallel that may share some data.
\item $\myIf{e}{C_1}{C_2}$ executes $C_1$ if $e$ evaluates to $\mytrue$; executes $C_2$ if $e$ evaluates to $\myfalse$.
\item $\myLoop{e}{C}$ iteratively executes $C$ as long as $e$ evaluates to $\mytrue$.
\end{itemize}
For instance, 
the programs in \Sec{sec:reasoning-with-BHL} conform to the programming language $\Prog$.

Hereafter we assume that all programs are well-typed although we do not explicitly mention the types. 
Checking this condition for our language can be done by adapting a standard type-checking algorithm to our setting.

We write $\upd{C}$ for the set of all variables that may be updated by executing $C$:
$\upd{\myskip} = \emptyset$, $\upd{v {\detassign} e} \allowbreak = \{v\}$, 
$\upd{C_1; C_2} = \upd{C_1 \| C_2} = \upd{\myIf{e}{C_1\allowbreak}{C_2}} = \upd{C_1} \cup \upd{C_2}$, and
$\upd{\myLoop{e}{C}} = \upd{C}$.

Then we impose the following restriction to every occurrence of $C_1 \| C_2$:
$\upd{C_1} \cap \var(C_2) = \upd{C_2} \cap \var(C_1) = \emptyset$.
This restriction is to ensure that an execution of $C_1$ does not interfere with that of $C_2$, and vice versa.

\subsection{Semantics of \Prog}
\label{sub:program:semantics}

We define the semantics of $\Prog$ over a Kripke model $\M$ 
with labeled transitions 
(\Sec{sub:Kripke})
based on the standard structural operational semantics (e.g. \cite{10.5555/1296072}).
Intuitively, executing a program command $c$ in a possible world $w = (\mem,\ac,\testHist)$ updates the memory $\mem$, records the command $c$ as the previous action $\ac$, and stores all previously executed hypothesis tests in the history $\testHist$.

Formally, for a possible world $w \in \calw$ and $n = \len{w}$, we write 
\[
w = w[0],w[1],\ldots,w[n-2],(\mem,\ac,\testHist)
\]
where $(\mem,\ac,\testHist)$ is the current state $w[n-1]$ with 
an assignment $\mem \colon \var\to\mathcal{O} \cup \{\bot\}$,
an action $\ac$ in the last transition in $w$, and
a test history $\testHist$.

For the assignment $\mem$ of the current state, we define the evaluation $\sem{e}_{\mem}$ of a program term $e$ inductively by $\sem{v}_{\mem} = \mem(v)$ and $\sem{f(e_1,\ldots,e_k)}_{\mem} = \sem{f}(\sem{e_1}_{\mem},\ldots,\sem{e_k}_{\mem})$.

As in \Fig{fig:rules:semantics}, 
we define a binary relation 
\[
{\longrightarrow} \subseteq (\Prog \times \calw) \times ( (\Prog \times \calw) \cup \calw)
\]
that relates a pair $\langle C, w \rangle$ consisting of a program $C$ and a possible world $w$ to its next step of execution. 
If $C$ is terminated, the next step will be a possible world $w'$, otherwise the execution continues to the $\langle C', w' \rangle$. 
\begin{figure}[htb]
\begin{align*}
\langle \myskip, w \rangle 
&\longrightarrow w; (m,\myskip,\testHist)
\\[0.8ex]
\langle v \detassign e, w \rangle &
\longrightarrow w; (\mem[v \mapsto \sem{e}_{\mem}],v \detassign e,\testHist)
\\[0.8ex]
\langle v \detassign f_{\!\alg}(y), w \rangle
&\longrightarrow
  w; (\mem',v \detassign f_{\!\alg}(y), \testHist') \\[-0.5ex]
&\qquad\text{ where }
 \mem' \eqdef \mem[v \,{\mapsto} \sem{f_{\!\alg}(y)}_{\mem}, h_{y,\alg} \mapsto \sem{h_{y,\alg}{+}1}_{\mem} ] \\[-0.5ex]
&\quad\qquad\qquad\testHist' \eqdef \testHist \uplus \{ \mem(y) \mapsto \{\!\alg\}
\end{align*}
\begin{mathpar}
\dfrac{
  \langle C_1, w \rangle \longrightarrow w'
}{
  \langle C_1;C_2, w \rangle \longrightarrow \langle C_2, w' \rangle
}

\dfrac{
  \langle C_1, w \rangle \longrightarrow \langle C'_1, w' \rangle
}{
  \langle C_1;C_2, w \rangle \longrightarrow \langle C'_1;C_2, w' \rangle
}

{
\langle \myIf{e}{C_1}{C_2}, w \rangle \longrightarrow 
\begin{cases}
 \langle C_1, w \rangle & \sem{e}_m = \mytrue\\
 \langle C_2, w \rangle &\sem{e}_m = \myfalse
\end{cases}
}

{
\langle\myLoop{e}{C},w\rangle \longrightarrow
\begin{cases}
\langle C;\myLoop{e}{C},w\rangle \!&\!\sem{e}_m = \mytrue\\
 w  \!&\!\sem{e}_m = \myfalse
\end{cases}
}

\dfrac{
    \langle C_1,w\rangle \longrightarrow \langle C'_1,w'\rangle
}{
    \langle C_1\|C_2,w\rangle \longrightarrow \langle C'_1\|C_2,w'\rangle
}

\dfrac{
    \langle C_1,w\rangle \longrightarrow w'
}{
    \langle C_1\|C_2,w\rangle \longrightarrow \langle C_2,w'\rangle
}

\dfrac{
    \langle C_2,w\rangle \longrightarrow \langle C'_2,w'\rangle
}{
    \langle C_1\|C_2,w\rangle \longrightarrow \langle C_1\|C'_2,w'\rangle
}

\dfrac{
    \langle C_2,w\rangle \longrightarrow w'
}{
    \langle C_1\|C_2,w\rangle \longrightarrow \langle C_1,w'\rangle
}
\end{mathpar}
\caption{Rules of execution of programs.
The operation $\uplus$ on a test history $\testHist$ is defined in \eqref{eq:test-history}.}
\label{fig:rules:semantics}
\end{figure}

We remark that the semantics of a program contains the trace of commands executed in it. 
Hence, even if different programs finally result in the same memory, their semantics may be different.
For instance, when the value of a variable $v$ is $1$, the execution of the two programs $v \detassign v+1$ and $v \detassign 2*v$ result in different worlds with the same memory:
\begin{align*}
    \langle v \detassign v+1 , ([v{\mapsto}1],\ac,\testHist) \rangle &\longrightarrow ([v{\mapsto}1],\ac,\testHist), ([v{\mapsto}2],v \detassign v+1,\testHist)\\
    \langle v \detassign 2*v , ([v{\mapsto}1],\ac,\testHist) \rangle &\longrightarrow ([v{\mapsto}1],\ac,\testHist),([v{\mapsto}2], v \detassign 2*v,\testHist)
\end{align*}

We define the semantic relation $\sem{C} \subseteq \calw \times \calw$ by
\[
\sem{C}(w) = \{
w' ~|~ \langle C, w\rangle \longrightarrow^\ast w'
\}
\]
where $\longrightarrow^\ast$ is the transitive closure of $\longrightarrow$.
When the program $C$ does not terminate, $\sem{C}(w) = \emptyset$.
With the semantic relation $\sem{c}$ for single commands $c$ of the programming language $\Prog{}$, we instantiate the transition relation in the Kripke model $\M$ (\Def{def:epistemic-Kripke});
i.e., we define the transition relation $\transX{c}$ as the semantic relation $\sem{c}$.

\subsubsection{Remark on Parallel Compositions}
\label{Remark:nondeterminism:parallel}
Since parallel compositions are nondeterministic, 
$w' \in \sem{C_1 \| C_2}(w)$ may not be unique.
However, the resulting world $w'$ is essentially the same, because $w' \in \sem{C_1 \| C_2}(w)$ is convertible to a pair of $w_1 \in \sem{C_1 }(w)$ and $w_2 \in \sem{C_1}(w)$ and vice versa.

If we have
$\langle C_b,w \rangle \longrightarrow^\ast w;u_b$ for $b = 1,2$,
then by $\upd{C_b} \cap \var(C_{3-b}) = \emptyset$,
we obtain a sequence $u'$ such that $\langle C_1\|C_2,w\rangle \longrightarrow^\ast w;u'$ by combining $u_1$ and $u_2$.

Conversely, if  $\langle C_1\|C_2,w\rangle \longrightarrow^\ast w;u'$, we can decompose $u'$ into $u_1$ and $u_2$ such that $\langle C_b,w \rangle \longrightarrow^\ast w;u_b$ for $b = 1,2$ 
(for details, see \App{sub:appendix:parallel}).

\subsubsection{Procedures of Hypothesis Testing}

We define the interpretation of a program 
$f_{\alg}$ 
for a hypothesis test 
$\alg = (\phi_0, \Test, \Dtpo, \allowbreak \unlikelys{\side}, \popl)$ 
with a null hypothesis $\phi_0$, a test statistic $\Test$, a test type $\side$, and a statistical model~$\popl$.
For a dataset $y$ and an assignment $m$, 
$\sem{f_{\alg}(y)}\!\,_m$ 
represents the $p$-value:
\begin{align}\label{eq:sem:HT}
\sem{f_{\alg}(y)}\!\,_m =
\Pr_{\tstat \sim \Dtpo}[\, \tstat \unlikelys{s} \Test(m(y)) \,]
{,}
\end{align}
which is the probability that a value $\tstat$ is at most as likely as the test statistic $\Test(m(y))$ when it is sampled from 
$\Dtpo$ in the world where the null hypothesis $\phi_0$ is true.

As in \Fig{fig:rules:semantics}, the execution of a program $v \detassign f_{\alg}(y)$ for a hypothesis test $\alg$ updates the test history $\testHist$ so that $\alg$ is added to the \emph{multiset} $\testHist(\mem(y))$ of all tests using the dataset $\mem(y)$.
Formally, the 
operation $\uplus$
in \Fig{fig:rules:semantics} is the union of multisets:
\begin{align} \label{eq:test-history}
\testHist'(o) = 
\begin{cases}
\testHist(o) \uplus \{ \alg \}
&\mbox{ if $o = \mem(y)$ }
\\
\testHist(o)
&\mbox{ otherwise}
\end{cases}
\end{align}

To refer to the test history $\testHist_{w}$ in a possible world $w$, 
we introduce a \emph{history variable} $h_{y,\alg} \in \varI$ for each variable $y$ and each hypothesis test $\alg$.
A history variable $h_{y,\alg}$ takes an integer value representing the number of executions of a hypothesis test $\alg$ on a dataset $y$.
Since $h_{y,\alg}$ is an invisible variable, it never appears in a program.

The interpretation of $h_{y,\alg}$ is consistent with the test history $\testHist_{w}$;
namely, $\mem_{w}(h_{y,\alg})$ represents the number of occurrences of $\alg$ in the multiset $\testHist_{w}(\mem_{w}(y))$.
As shown in \Fig{fig:rules:semantics}, if a program command updates $\testHist_{w}$, the values $\mem_{w}(h_{y,\alg})$ of history variables $h_{y,\alg}$ are also updated consistently.
Although $h_{y,\alg}$ is an invisible variable, the test history $\testHist_{w}$ itself is observable (\Def{def:observe}) and is used to define knowledge.

\section{Assertion Language}
\label{sec:assertion-logic}
In this section, we define an assertion language called \emph{epistemic language for hypothesis testing} (\ELHT{}) that can express knowledge and statistical beliefs.
 
\ELHT{} is based on the modal logic S5 in which $\Know$ is the box modal operator for expressing knowledge.
This language has predicates for describing assertions about sampling and hypothesis testing.
The statistical belief modality $\KnowLeEt$ is defined as syntax sugar in a form of disjunctive knowledge to express a belief about an alternative hypothesis obtained by a hypothesis test $\alg$ on a dataset $y$ with a $p$-value $\epsilon$.
Specifically, we formalize a statistical belief $\KnowLeEt \phi$ in an alternative hypothesis $\phi$ as the knowledge that either (i) $\phi$ holds, (ii) the sampled dataset is unluckily far from the population, or (iii) the population does not satisfy the requirements for the hypothesis test.

\subsection{Syntax of the Assertion Language}
\label{sub:assertion:syntax}

We introduce the syntax of the assertion language \ELHT{}.
We define assertion terms, formulas, and predicate symbols for statistical notions.
Then we introduce the modality of statistical beliefs as disjunctive knowledge.

\subsubsection{Assertion Terms}
\label{sub:assertion-terms}

We introduce \emph{assertion terms} to denote data values as follows.
Recall that $\var$ is the finite set of all variables and $\Fsym$ is the set of all function symbols.
We introduce a set $\IntVar$ of \emph{integer variables} denoting finite tuples of integers such that $\IntVar \cap \var = \emptyset$.
Then the set $\ATerm$ of assertion terms is defined by:
\begin{align*}
u &\, \mathbin{::=} x \mid i \mid f(u,\ldots,u)
\end{align*}
where $x\in\var$, $i \in \IntVar$, and $f\in\Fsym$.
Notice that the assertion terms can deal with invisible variables unlike the program terms (\Sec{sub:program:syntax}).
$\Fsym$ includes function symbols denoting families of probability distributions (e.g., $N$ for normal distributions) and those denoting data operations (e.g., $\mean$ for calculating the mean of data values).

\subsubsection{Assertion Formulas}
\label{sub:epistemic-language}
We define the syntax of assertion formulas with a modal operator $\Know$ for \emph{knowledge}.
As in previous studies,
a formula $\Know\phi$ expresses that we know $\phi$.
Formally, 
the set $\Fml$ of \emph{formulas} is defined by:
\begin{align*}
\phi \,\mathbin{::=}\,\,&
\eta(u_1, \ldots, u_n) \mid
\neg \phi \mid \phi \land \phi \mid
\Know\phi 
\\
\phi' \,\mathbin{::=}\,\,& 
\phi \mid \neg \phi' \mid \phi' \land \phi' \mid \forall i. \phi'
\end{align*}
where $\eta\in\Pred$, $u_1, \ldots, u_n \in\ATerm$, and $i \in \IntVar$.
In the formulas, the quantifiers $\forall$ and $\exists$ never appear inside the epistemic modality $\Know$.
They are used only to prove the relative completeness of our program logic in later sections.
We remark that there is no universal/existential quantification over the observable and invisible variables.
We denote the set of all variables occurring in a formula $\phi$ by $\fv{\phi}$.

As syntax sugar, we use 
\emph{disjunction} $\vee$, 
\emph{implication} $\rightarrow$, 
and
\emph{existential quantifier} $\exists$.
We also define \emph{epistemic possibility} $\Possible$ as usual by 
$\Possible\phi \eqdef \neg \Know \neg \phi$.

\subsubsection{Hypothesis Formulas}
\label{sub:hypotheses}

We introduce notations for alternative/null hypotheses in hypothesis tests  (\Sec{sec:preliminaries}).
Recall that an alternative hypothesis $\phi_1$ is a proposition that we wish to prove, and that
a null hypothesis $\phi_0$ is a proposition that contradicts the alternative hypothesis $\phi_1$.
We write $\nuh \phi_1$ for the null hypothesis corresponding to an alternative hypothesis $\phi_1$.
In \Sec{sub:properties:hypo}, we define it as syntax sugar and discuss the details.

\subsubsection{Predicate Symbols}
\label{sub:predicate}

We introduce the following predicate symbols for statistical notions:
\begin{itemize}
\item 
$u = u'$ represents the equality of two data values $u$ and $u'$.
\item 
$\sampled{y}{x}{n}$ expresses that a dataset $y$ consists of $n$ data sampled from a population~$x$.
\item 
$\followed{y}{x}$ represents that a dataset $y$ has been sampled from a population $x$.
\item 
For $\bowtie\, \in \{ =, \le, \ge, <, > \}$, $\err \in [0, 1]$, $y\in\var$, and $\alg \in \cala$,\,
$\NegBowEya(\err)$ 
represents that the observation of a dataset $y$ is unlikely to occur (with exception $\bowtie\err$) according to the hypothesis test~$\alg$.
\end{itemize}

For brevity, we define the formula $\cmpds{S}$ for a multiset $S$ of pairs of variables and hypothesis tests.
Intuitively, $\cmpds{(y, \alg)}$ represents that a dataset $y$ has been sampled from a population that satisfies the hypothesis test $\alg$'s requirement.
Formally:
\begin{align*}
\cmpds{S} \eqdef
\hspace{-0.3ex}\!\bigwedge_{(y, \alg)\in S} \hspace{-0.8ex} \followed{y}{\poplA(\xi_{\alg}, \theta_{\alg})}
{,}
\end{align*}
where $\poplA(\xi_{\alg}, \theta_{\!\alg})$ denotes the population following the statistical model $\poplA$ for a test $\alg$ (\Sec{sub:notation-hypothesis-test}).
When $S$ is a singleton $\{ (y, \alg) \}$, we abbreviate 
$\cmpds{\{(y, \alg)\}}$ as $\cmpds{y,\alg}$.

\subsubsection{Modality of Statistical Beliefs}
\label{sub:statistical-formulas}
We use the following formulas on executions of hypothesis tests:
\begin{itemize}
\item
$\kappa$ describes the record of all hypothesis tests conducted so far.
Formally, $\kappaS$ is the formula representing that $S$ is the finite multiset of every pair $(y, \alg)$ of a dataset $y$ and a hypothesis test $\alg$ that has been applied to $y$.
This formula is defined as equations between history variables $h_{y,\alg}$ (\Sec{sub:program:semantics}) and their values by: 
\begin{align}
\label{eq:kappaS}
\kappaS \eqdef \hspace{-0.3ex}
\!\bigwedge_{(y, \alg)\in S} \hspace{-0.8ex} h_{y,\alg} = n_{(y, \alg, S)}
\land \hspace{-1ex}
\bigwedge_{(y, \alg)\in (\varO\times\cala)\setminus S} \hspace{-0.8ex} h_{y,\alg} = 0
\end{align}
where 
$n_{(y, \alg, S)}$ is the integer representing the number of occurrences of $(y, \alg)$ in the multiset $S$,
$\varO$ is the finite set of all observable variables, and $\cala $ is the finite set of all hypothesis tests we consider.
When $S$ is a singleton $\{ (y, \alg) \}$, we abbreviate $\kappaX{\{(y, \alg)\}}$ as $\kappaYA{y}$.
\item 
For $\bowtie\, \in \{ =, \le, \ge, <, > \}$, $\err \in [0, 1]$, $y\in\var$, and a hypothesis test $\alg = (\phi_0, \Test, \Dtpo, {\unlikelys{\side}\!,} \allowbreak \poplxt)$,
we define the formula
$\TestBowE{y,\alg}(\err)$ by:
\begin{align*}
\TestBowE{y,\alg}(\err) & \eqdef
\NegBowEya(\err) \land \kappaX{\ListTests{(y,\alg)}},
\end{align*}
where $\ListTests{(y,\alg)}$ is the multiset of each individual test and its dataset in the (combined) hypothesis test $\alg$ (\Eq{eq:list-tests} in \Sec{sub:notation-hypothesis-test}).
\end{itemize}

As syntax sugar, we introduce the \emph{statistical belief modality} $\KnowBowEt$.
Intuitively, a \emph{statistical belief} $\KnowLeEt\phi$ expresses that we believe a hypothesis $\phi$ based on a statistical test $\alg$ on an observed dataset $y$ with a certain error level ($p$-value) at most~$\err$.
We formalize this as the knowledge that either (i) the hypothesis $\phi$ holds, (ii) the observed dataset $y$ is unluckily far from the population (from which $y$ is sampled), or (iii) the dataset $y$ did not come from a population that satisfies the test $\alg$'s requirement (e.g., a population following a normal distribution).

Formally, for a hypothesis test $\algp{\nuh\phi}$ with an alternative hypothesis $\phi$ and its null hypothesis $\nuh\phi$, we define:
\begin{align} \label{eq:KnowBowEphi}
\KnowBowEphi \phi \eqdef \Know( \phi  \lor  \TestBowCnegphi(\err) 
\lor \neg \cmpds{y,\algp{\nuh\phi}}).
\end{align}
As the dual modality, we define the \emph{statistical possibility} $\PossibleBowEt$ by $\PossibleBowEt\phi \eqdef \neg \KnowBowEt \neg \phi$.
For brevity, we often omit the subscript $\nuh\phi$ from $\algp{\nuh\phi}$ to abbreviate $\KnowBowEphi \phi$ as $\KnowBowEt \phi$.
We also write $\KnowE$ instead of $\KnowEqEt$ and
$\Know^{\bowtie\varepsilon} \phi$ instead of $\KnowBowEt \phi$.

\subsection{Semantics of the Assertion Language}
\label{sub:assertion:semantics}

We define semantics for the assertion language \ELHT{}
using a Kripke model.

\subsubsection{Interpretation of Assertion Terms and Formulas}
We introduce an \emph{interpretation function} $\cali: \IntVar \rightarrow \ints^*$ that assigns a finite tuple of integers to an integer variable.
Then we define the interpretation $\sem{u}_{\mem}^{\cali}$ of an assertion term $u$ w.r.t. $\cali$ and an assignment $\mem: \var \rightarrow \calo\cup \{\bot\}$ inductively by $\sem{x}_{\mem}^{\cali} = \mem(x)$, $\sem{i}_{\mem}^{\cali} = \cali(i)$, and $\sem{f(u_1,\ldots,u_k)}_{\mem}^{\cali} = \sem{f}(\sem{u_1}_{\mem}^{\cali},\ldots,\sem{u_k}_{\mem}^{\cali})$.

We define the interpretation of formulas in a world $w$ in a Kripke model $\M =(\calw, (\transA)_{a\in\Act}, \relO, \allowbreak (V_w)_{w\in\calw})$ in \Def{def:epistemic-Kripke} as follows:
\begin{align*}
\M, w \modelsi \eta(u_1, \ldots, u_k) 
& ~\mbox{ iff }~
(\sem{u_1}_{\sigmaw}^{\cali}, \ldots, \sem{u_k}_{\sigmaw}^{\cali}) \in V_w(\eta)
\\
\M, w \modelsi \neg \phi
& ~\mbox{ iff }~
\M, w \not\modelsi \phi
\\
\M, w \modelsi \phi \land \phi'
& ~\mbox{ iff }~
\M, w \modelsi \phi
~\mbox{ and }~
\M, w \modelsi \phi'
\\
\M, w \modelsi \Know \phi
& ~\mbox{ iff }~
\mbox{for all $w' \in\calw$,~ $(w, w') \in \relO$}
\hspace{0.5ex}~
\mbox{implies~ } \M, w' \modelsi \phi
\\
\M, w \modelsi \forall i. \phi
& ~\mbox{ iff }~
\M, w \models^{\cali[n/i]} \phi
~\mbox{ for all $n\in\nats$}
{.}
\end{align*}
$\M$ is sometimes omitted when it is clear from the context.

\subsubsection{Interpretation of Predicate Symbols}

We define the interpretation of predicate symbols.
Let $\alg = (\phi, \Test, \allowbreak \Dtp, \unlikelys{\side}, \popl)$ be a hypothesis test.
Recall that the population's distribution has type $\Dists\dompop$, and that
$\unlikelys{\side}$ is the likeliness relation (Section~\ref{sub:notation-hypothesis-test}).
In a world $w$, we interpret predicate symbols by:
\begin{align*}
V_w(=) & =
\bigl\{ (o, o) \in \calo \times \calo \bigr\}
\\[-0.3ex]
V_w(\leftsquigarrow) & =
\bigl\{ (d, D, n) \in (\typelist{\dompop}) {\times} (\Dists\dompop) {\times} \nats \,\big|
\substack{\scriptsize\,\mbox{ There is an $i\in\nats$ s.t.} \\
\scriptsize w[i] \xrightarrow{\!d {\sim} D^n\hspace{-1ex}\!} w[i+1]} 
\,\bigr\}
\\[-0.3ex]
V_w(\leftarrowtail) & =
\bigl\{ (d, D)  ~\big|\, 
(d, D, n) \in V_w(\leftsquigarrow)
\bigr\}
\\[-0.3ex]
V_w(\NegBowEya) & =
\bigl\{ \err \in [0, 1] ~\big|\,
\mbox{\small $\displaystyle\Pr_{\tstat \sim \Dtp}[\, \tstat \unlikelys{\side} \Test(\sigmaw(y)) \,] \bowtie \err$~}
\bigr\}
{.}
\end{align*}
Intuitively, the set $V_w(\NegA{y,\alg})$ consists of only the $p$-value $\err$ with which the hypothesis test $\alg$ on the dataset $\sigmaw(y)$ rejects the null hypothesis $\phi$.
Then $\M, w \models \NegA{y,\alg}(\err)$ represents that in a possible world $w$, the observation of a dataset $y$ is unlikely to occur (except with probability $\err$) according to the hypothesis test $\alg$
where the test statistic follows the distribution $\Dtp$ in the world $w$.

Formally, we have:
\begin{align} \label{eq:NegA}
\M, w \models \NegA{y,\alg}(\err)
& ~\mbox{ iff }~
\Pr_{\tstat \sim \Dtp}\![\, \tstat \unlikelys{\side} \Test(\sigmaw(y)) \,] = \err
{,}
\end{align}
where 
the \emph{$p$-value}
$\Pr_{\tstat \sim \Dtp}[\, \tstat \unlikelys{\side} \Test(\sigmaw(y)) \,]$ 
is the probability that a value $\tstat$ is at most as likely as the test statistic $\Test(\sigmaw(y))$ when it is sampled from 
the distribution $\Dtp$ in the possible world $w$;
e.g., when $\Dtp$ is the standard normal distribution $\Normal(0,1)$, then
$\Pr_{\tstat \sim \Normal(0,1)}[\, \tstat \unlikelys{\sideT} 1.96 \,] =
\Pr_{\tstat \sim \Normal(0,1)}[\, |\tstat| \ge 1.96 \,] \approx 0.05$.
We remark that the $p$-value is \emph{not} a probability in the real world, but a probability in the possible world $w$ where the null hypothesis $\phi$ is true.

Analogously, the interpretation of $\NegBowEya$ is defined in terms of a range of $p$-values.
For instance, $\NegE{y,\alg}(\err)$ represents that the $p$-value of a test $\alg$ on  a dataset $y$ is less than $\err$.

The interpretation of the formula $\kappaS$ is given by:
\begin{align}
\label{eq:kappaS:sem}
\M, w \models \kappaS 
& ~\mbox{ iff }~
\testHist_{w} = \hspace{-1ex}
\biguplus_{(y, \alg) \in S}\hspace{-1ex} \{ \mem_{w}(y) \mapsto \{ \alg \} \}
{,}
\end{align}
where $\testHist_{w}$ is the test history that maps a dataset $o$ to the multiset of all hypothesis tests applied to the dataset $o$ in the world $w$ (\Sec{sub:possible-worlds}).

\subsection{Interpretation of Statistical Belief Modality}
\label{sub:interpret-statistical-modality}

The interpretation of the statistical belief modality $\KnowLeEt$ is given as follows.
\begin{align*}
\M, w \models \KnowLeEt \phi
& \,\mbox{ iff }\,
\M, w \models 
\Know\bigl( \phi  \lor  \TestE{y,\alg}(\err) \lor \neg \cmpds{y,\alg} \bigr)
\\[-0.3ex]
& \,\mbox{ iff }\,
\mbox{for all $w'$\!,\, $(w, w') \,{\in}\, \relO$ implies 
$\M, w' \models (\neg\phi \land \cmpds{y,\alg}) \,{\implies}\, \TestE{y,\alg}(\err)$}
{.}
\end{align*}

Intuitively, $\KnowLeEt \phi$ expresses a belief that an alternative hypothesis $\phi$ on the population is true.
For a two-tailed test $\alg$,
$w' \models (\neg\phi \land \cmpds{y,\alg}) \implies \TestE{y,\alg}(\err)$ 
means that if we consider a possible world $w'$ where the null hypothesis $\neg\phi$ is true and 
the dataset $y$ is drawn from a population satisfying the test $\alg$'s requirement $\cmpds{y,\alg}$,
then the execution of
$\alg$ 
would conclude that the observation of the dataset $y$ is unlikely to occur (with exceptions at most $\err$), 
i.e., 
$w' \models \TestE{y,\alg}(\err)$.
See \Sects{sub:assertion:remark}, \ref{sub:test:meaningful}, and \ref{sec:properties} for discussion.

Although the modality $\Know{}$ expresses the knowledge in terms of S5, 
$\KnowLeEt\phi$ represents a \emph{belief} instead of a knowledge. 
This is because $\phi$ can be \emph{false} 
when $\TestE{y,\alg}(\err) \lor \neg \cmpds{y,\alg}$ holds; 
i.e., we may have a false belief in $\phi$ (i) when the sampled dataset $y$ is unluckily far from the population or (ii) when the dataset $y$ did not come from the population that satisfies the test $\alg$'s requirement.

\begin{example}[Statistical belief in $Z$-tests]\label{eg:Z-test:belief}
Recall again the two-tailed $Z$-test for two population means 
in \Eg{eg:hypothesis-test}.
The alternative hypothesis is 
$\phi \eqdef (\mu_{\ppl1} \neq \mu_{\ppl2})$, and 
the null hypothesis $\nuh \phi$ is given by $\mu_{\ppl1} = \mu_{\ppl2}$.
As in \Eg{eg:Z-test:likeliness}, we denote this $Z$-test by 
$\alg = (\nuh \phi, \allowbreak \Test, \Normal(0, 1), \allowbreak \unlikelys{\sideT}\!, \Normal(\mu_{\ppl1},\sigma^2)\times \Normal(\mu_{\ppl2},\sigma^2))$.

Suppose that in a world $w$, we sample two datasets $\sigmaw(y_1)$ and $\sigmaw(y_2)$ respectively from two populations $\Normal(\mu_{\ppl1},\sigma^2)$ and $\Normal(\mu_{\ppl2},\sigma^2)$.
If the null hypothesis $\nuh \phi$ is true,  the $Z$-test statistic $\Test(\sigmaw(y_1), \sigmaw(y_2))$ follows the distribution $\Normal(0,1)$.

If $\Test(\sigmaw(y_1), \sigmaw(y_2)) = 3$, we have:
\[
\Pr_{\tstat \sim \Normal(0,1)}[ \tstat \unlikelys{\sideT} \Test(\sigmaw(y_1), \sigmaw(y_2)) ] \allowbreak 
< 0.05.
\]
Then the null hypothesis $\nuh \phi$ is rejected, and
we obtain the statistical belief that the alternative hypothesis $\phi$ is true with the significance level $0.05$, i.e., $w \models \KnowXx{<0.05}{y,\alg} \phi$.

In contrast, if $\Test(\sigmaw(y_1), \sigmaw(y_2)) = 1.8$, then
$w \models \neg \KnowXx{<0.05}{y,\alg} \phi$, because we have:
\[
\Pr_{\tstat \sim \Normal(0,1)}[\, \tstat \unlikelys{\sideT} \Test(\sigmaw(y_1), \sigmaw(y_2)) \,] > 0.05.
\]
\end{example}

\subsection{Remark on the Universe of the Kripke Model}
\label{sub:assertion:remark}

We remark that
the universe $\calw$ of the model $\M$ is assumed to include all possible worlds we can imagine.
If there is no possible world satisfying a null hypothesis $\nuh\phi$ in $\M$, then the alternative hypothesis $\phi$ is satisfied in all worlds in $\M$, hence so are $\Know \phi$ and $\KnowLeEt \phi$.
This implies that if we cannot imagine a possible world where $\nuh\phi$ is true, then we already know that $\phi$ is true without conducting the hypothesis test $\alg$.

\subsection{When Hypothesis Tests are Meaningful}
\label{sub:test:meaningful}

The formula $\KnowLeEt \phi$ expresses a belief after conducting a hypothesis test $\alg$ on a dataset $y$, and covers the following two cases where the execution of $\alg$ is not useful: 
\begin{enumerate}
\item[(i)] we \emph{knew} that the alternative hypothesis $\phi$ is true without conducting the test $\alg$;
\item[(ii)] we know that the requirement $\cmpds{(y,\alg)}$ for the test $\alg$ on $y$ is \emph{not} satisfied.
\end{enumerate}
Hence, deriving only the formula $\KnowLeEt \phi$ is not sufficient to conclude the correctness of the alternative hypothesis $\phi$ from the execution of the hypothesis test $\alg$.

Formally, $\KnowLeEt \phi$ is satisfied also when we have the prior knowledge $\Know(\phi \lor \neg\cmpds{(y,\alg)})$ that
(i) $\phi$ is satisfied or (ii) the test $\alg$'s requirement is not satisfied.
Thus, the execution of $\alg$ is meaningful only when we do \emph{not} have the prior knowledge $\Know(\phi \lor \neg\cmpds{(y,\alg)})$,
i.e., only when we have the \emph{prior belief} $\Possible(\neg\phi \land \cmpds{(y,\alg)})$.

For the outcome of the test $\alg$ to be meaningful, the requirement $\cmpds{(y,\alg)}$ must hold in the real world.
In practice, however, we usually have a limited knowledge of the population 
(\Sec{sec:discuss}), and may not know whether the population satisfies the requirement $\cmpds{(y,\alg)}$.
For this reason, in \Sec{sec:belief-hoare-logic} and \ref{sec:reasoning-with-BHL}, we aim to derive a statistical belief $\KnowLeEt \phi$ 
under the assumption that $\cmpds{(y,\alg)}$ holds as a precondition instead of $\Know\cmpds{(y,\alg)}$.

\section{Prior Beliefs and Posterior Statistical Beliefs in \ELHT}
\label{sec:properties}
In this section, we clarify the importance of prior beliefs in the acquisition of statistical beliefs 
by describing them using the assertion language \ELHT{}.
We then present the essential properties of statistical beliefs; e.g., $\KnowBowEt$ expresses a belief weaker than the S5 knowledge modality $\Know$.
We also show how a statistical belief is derived from a test history.

Prior belief/knowledge of hypotheses is essential in choosing which hypothesis testing method is appropriate for a given situation.
For example, to apply a two-tailed $Z$-test, analysts must have the prior belief that both tails ($\mu_1 > \mu_2$ and $\mu_1 < \mu_2$) are possible.
In contrast, to apply an upper-tailed test, they must have the prior knowledge that the lower tail ($\mu_1 < \mu_2$) is impossible.
Using the assertion language \ELHT{}, we explain that such prior beliefs are important for the application of a hypothesis test to be meaningful.

\subsection{Hypothesis Formulas}
\label{sub:properties:hypo}
To formalize the prior beliefs for hypothesis testing,
we introduce notations for the alternative and null hypotheses in the assertion language \ELHT{}.

We use two formulas $\phiU$ and $\phiL$ to represent the alternative hypotheses in an upper-tailed test and a lower-tailed test,  respectively.
Then $\phiU$ and $\phiL$ cannot be true simultaneously;
i.e., $\models \neg \phiU \lor \neg \phiL$.
The alternative hypothesis of the two-tailed test is:
\begin{align}
\label{eq:phiT}
\phiT & \eqdef\, \phiU \lor \phiL
{.}
\end{align}
The syntax sugar $\nuh \phiT$, $\nuh \phiU$, and $\nuh \phiL$ for the null hypotheses can be defined by:
\begin{align}
\label{eq:nuh}
\nuh \phi_{\side} & \eqdef\, \neg \phiU \land \neg \phiL
~~~\mbox{ for }\, \side \in \{ \sideT, \sideU, \sideL \}
{.}
\end{align}

\begin{example}[Hypothesis formulas in $Z$-tests] \label{eg:alt:Z-tests}
For $\mu_{\ppl1}, \mu_{\ppl2} \in \reals$,
the two-tailed, upper-tailed, and lower-tailed $Z$-test (\Eg{eg:hypothesis-test}) have the alternative hypotheses:
\[
\phiT \eqdef (\mu_{\ppl1} \neq \mu_{\ppl2}),
~\mbox{  }~
\phiU \eqdef (\mu_{\ppl1} > \mu_{\ppl2}),
~\mbox{ and }~
\phiL \eqdef (\mu_{\ppl1} < \mu_{\ppl2}).
\]
This is because the upper-tailed (resp. lower-tailed) test is based on the assumption $\mu_{\ppl1} \ge \mu_{\ppl2}$ (resp. $\mu_{\ppl1} \le \mu_{\ppl2}$).
We can see that 
$\models \phiT \leftrightarrow (\phiU \lor \phiL)$ and 
$\models \neg \phiU \lor \neg \phiL$.
The null hypotheses of these tests are
$\nuh \phi_{\side} \eqdef (\mu_1 = \mu_2)$
for
$\side \in \{ \sideT, \sideU, \sideL \}$.
See \Tbl{tab:hypo:form} for the summary of the hypothesis formulas in $Z$-tests.
\end{example}

\begin{table}[t]
   \centering
   \caption{Hypothesis formulas in the $Z$-tests (\Eg{eg:hypothesis-test}).}
   \label{tab:hypo:form}
   \begin{small}
   \begin{tabular}{@{} lll @{}} 
      \toprule
      Tails & alternative hypotheses & null hypotheses \\
      \midrule
      Two    & $\phiT\eqdef(\mu_{\ppl1} \neq \mu_{\ppl2})$ & $\nuh \phiT\eqdef(\mu_{\ppl1} = \mu_{\ppl2})$ \\
      Upper & $\phiU\eqdef(\mu_{\ppl1} > \mu_{\ppl2})$ & $\nuh\phiU\eqdef(\mu_{\ppl1} = \mu_{\ppl2})$ \\
      Lower & $\phiL\eqdef(\mu_{\ppl1} < \mu_{\ppl2})$ & $\nuh\phiL\eqdef(\mu_{\ppl1} = \mu_{\ppl2})$ \\
      \bottomrule
   \end{tabular}
   \end{small}
\end{table}

In the two-tailed test, the null hypothesis $\nuh \phiT$ is logically equivalent to $\neg \phiT$, i.e., $\mu_1 = \mu_2$.
In contrast, in the upper-tailed test, $\nuh \phiU$ (i.e., $\mu_1 = \mu_2$) implies $\neg \phiU$ (i.e., $\mu_1 \le \mu_2$) \emph{but not vice versa}.
This is also the case with the lower-tailed test.

\subsection{Prior Beliefs in Hypothesis Tests}
\label{sub:prior:beliefs}
We formally describe the prior knowledge of hypothesis tests using epistemic formulas.
We show an example in \Tbl{tab:prior:know}.

\begin{table}[ht]
   \centering
   \caption{Prior belief/knowledge in the $Z$-tests (\Eg{eg:hypothesis-test}) where $\phiU$ and $\phiL$ are respectively the alternative hypotheses of the upper-tailed and lower-tailed $Z$-tests in \Tbl{tab:hypo:form}.}
   \label{tab:prior:know}
   \begin{small}
   \begin{tabular}{@{} lll @{}} 
      \toprule
      Tails & \multicolumn{2}{l}{prior belief/knowledge} \\
      & general forms & $Z$-tests \\
      \midrule
      Two    & \!$\Possible \phiU \land \Possible \phiL$
      & $\Possible (\mu_{\ppl1} {>} \mu_{\ppl2}) \land \Possible(\mu_{\ppl1} {<} \mu_{\ppl2})$ \\[0.8ex]
      Upper & \!$\Possible \phiU \land \neg\Possible \phiL$
      & $\Possible (\mu_{\ppl1} {>} \mu_{\ppl2}) \land \Know (\mu_{\ppl1} {\ge} \mu_{\ppl2})$ \\[0.8ex]
      Lower & \!$\Possible \phiL \land \neg\Possible \phiU$
      & $\Possible (\mu_{\ppl1} {<} \mu_{\ppl2}) \land \Know (\mu_{\ppl1} {\le} \mu_{\ppl2})$ \\[0.3ex]
      \bottomrule
   \end{tabular}
   \end{small}
\end{table}

\subsubsection{Prior Beliefs in Two-Tailed $Z$-Tests}
For an application of the two-tailed $Z$-test to be meaningful, we are supposed to have the prior belief that 
$\mu_1 > \mu_2$ is possible (denoted by $\Possible \phiU$), and that $\mu_1 < \mu_2$ is possible (denoted by $\Possible \phiL$).
\ELHT{} naturally explains that these prior beliefs are essential to interpret the results of hypothesis tests as follows.
Assume that, in a world $w$, we had neither of these prior beliefs, but obtained a statistical belief $\KnowAt \phiT$ by conducting a two-tailed hypothesis test $\alg$;
i.e., $w \models \neg \Possible \phiU \land \neg \Possible \phiL \land \KnowAt \phiT$.
Since $\Possible$ is the dual operator of $\Know$, we have $w \models \Know \neg \phiU \land \Know \neg \phiL \land \KnowAt \phiT$.
By $\phiT \eqdef\, \phiU \lor \phiL$,
we have:
\[
w \models \Know \neg \phiT \land \KnowAt \phiT,
\]
that is, we already know that the alternative hypothesis $\phiT$ is false regardless of the result of the test $\alg$ (that aims to show that $\phiT$ is true).
Clearly, the execution of the test $\alg$ is meaningless when we know that $\phiT$ is false.
For this reason, the prior beliefs $\Possible \phiU$ and $\Possible \phiL$ are essential for the application of the two-tailed test to be meaningful.

We remark that even if we do not have these prior beliefs, the definition of the formula $\KnowAt \phiT$ is still consistent with the principle of the hypothesis testing (although the test is useless, as mentioned above).
Recall that the statistical belief is defined by
$\KnowAt \phiT \eqdef \Know( \phiT \lor \TestA{y,\alg}(\alpha) \lor \neg \cmpds{y,\alg} )$.
Then $w \models \Know \neg \phiT \land \KnowAt \phiT$ implies 
$w \models \Know( \TestA{y,\alg}(\alpha) \lor \neg \cmpds{y,\alg} )$;
i.e., we learn that either (i) the sampled dataset $y$ is unluckily far from the population, or (ii) $y$ was sampled from a population that does not satisfy the requirement $\cmpds{y,\alg}$ for the hypothesis test~$\alg$ on $y$.

\subsubsection{Prior Beliefs in One-Tailed $Z$-Tests}
When we apply the \emph{upper-tailed} $Z$-test, we are supposed to have the prior belief that 
$\mu_1 > \mu_2$ is possible (denoted by $\Possible \phiU$), and the prior knowledge that $\mu_1 < \mu_2$ is impossible (denoted by $\neg \Possible \phiL$ or by $\Know\neg\phiL$).

This prior knowledge $\Know\neg\phiL$ is used to select an upper-tailed test rather than a two-tailed.
In \Propo{prop:PriorBelief}, we show that $\Know\neg\phiL$ is logically equivalent to $\Know (\phiU \lor \nuh \phiU)$;
i.e., under the knowledge $\Know\neg\phiL$, either the alternative hypothesis $\phiU$ (i.e., $\mu_1 > \mu_2$) or the null hypothesis $\nuh \phiU$ (i.e., $\mu_1 = \mu_2$) holds.
Hence, the prior knowledge $\Know\neg\phiL$ allows for applying the upper-tailed test.
Without this prior knowledge, we cannot apply the upper-tailed test, because we do not see whether one of the alternative hypothesis $\phiU$ and the null hypothesis $\nuh \phiU$ holds.

In conclusion, we can use our assertion logic to explain that the prior knowledge $\Know\neg\phiL$ is crucial to apply the upper-tailed test.
Symmetrically, the lower-tailed test requires the prior knowledge $\Know\neg\phiU$, as indicated in \Tbl{tab:prior:know}.

\subsubsection{Posterior Beliefs in $Z$-Tests}
We remark that the prior beliefs in alternative hypotheses will not change even after conducting hypothesis tests.
For example, in the case of the two-tailed $Z$-test, the prior belief $\Possible \phiL \land \Possible \phiU$ remains to hold after conducting the test and obtaining a $p$-value $\alpha > 0$.
This is because the statistical belief is defined as disjunctive knowledge $\Know( \phiT \lor \TestA{y,\alg}(\alpha) \lor \neg \cmpds{y,\alg} )$, and thus cannot conclude any knowledge of the alternative hypothesis (e.g., $\Know \phiT$ or $\Know \neg\phiT$).

\subsubsection{Properties of Prior Beliefs in Hypotheses}
Now we show basic properties of prior beliefs in hypotheses as follows.

\begin{restatable}[Basic properties of prior beliefs]{prop}{PropPriorBelief}
\label{prop:PriorBelief}
Recall that $\phiT \eqdef \phiU \lor \phiL$ and
$\nuh \phi_{\side} \eqdef \neg \phiU \land \neg \phiL$ 
for each $\side \in \{ \sideT, \sideU, \sideL \}$.
\begin{enumerate}
\item 
In a two-tailed hypothesis test, either the null hypothesis $\phiT$ or the alternative hypothesis $\nuh \phiT$ is always satisfied; i.e.,
$\models \phiT \lor \nuh \phiT$.
\item 
We know that the lower-tail $\phiL$ is impossible iff 
we know that either the null hypothesis $\phiU$ or the alternative hypothesis $\nuh \phiU$ for the upper-tail test is satisfied:
\begin{align*} 
\models \Know \neg \phiL \leftrightarrow \Know (\phiU \lor \nuh \phiU)
{.}
\end{align*}
\item 
We know that the upper-tail $\phiU$ is impossible iff 
we know that either the null hypothesis $\phiL$ or the alternative hypothesis $\nuh \phiL$ for the lower-tail test is satisfied:
\begin{align*} 
\models \Know \neg \phiU \leftrightarrow \Know (\phiL \lor \nuh \phiL)
{.}
\end{align*}
\end{enumerate}
\end{restatable}
The proof is shown in~\App{sub:proof:assertions}.

\subsection{Type II Error}
Symmetrically to the $p$-value (type I error rate) $\alpha$,
the \emph{type II error rate} $\beta$ is 
the probability that a hypothesis test $\alg$ does \emph{not} reject the null hypothesis $\nuh\phi$ when the alternative hypothesis $\phi$ is true.
For instance, in the two-tailed $Z$-test (\Eg{eg:hypothesis-test}), 
$\beta$ is the probability that the $Z$-test fails to reject the null hypothesis $\mu_1 = \mu_2$ when the alternative hypothesis $\mu_1 \neq \mu_2$ is true.
We remark that $\beta$ is determined by the \emph{effect size} $\nicefrac{|\mu_1 - \mu_2|}{\sigma}$;
For a smaller distance 
$|\mu_1 - \mu_2|$, 
it is more difficult for the $Z$-test to distinguish the null and alternative hypotheses, hence the type II error rate $\beta$ is larger.

Formally, let $y'$ be a dataset such that the $p$-value $\alpha$ of a test $\alg$ is $0.05$ in a world $w$; i.e., 
$w \models \KnowXx{0.05}{y'\!,\alg} \phi$.
To calculate the type II error rate $\beta$, we consider an effect size $\es > 0$.
Suppose that a hypothesis $\xiPPL \eqdef (\es \,{=}\, \nicefrac{|\mu_1 - \mu_2|}{\sigma})$ is satisfied, i.e., $w \models \xiPPL$.
The belief about the type II error 
is expressed by
$w \models \KnowXx{\beta}{y'\!,\alg} \neg\xiPPL$;
i.e., in the world $w$, we believe that $\xiPPL$ is false with a degree $\beta$ of belief, although $\xiPPL$ is actually true in $w$.

\subsection{Properties of Statistical Beliefs}
\label{sub:assertion:properties}

Next, we present properties of the statistical belief modality $\KnowBowEt$.
\Propo{prop:StatBelief} explains basic properties of statistical belief; e.g., $\KnowBowEt$ expresses a belief weaker than the S5 knowledge $\Know$.
\Propo{prop:StatBeliefHT} shows how a $p$-value is derived from a test history.

To see these, we remark that the dual operator $\PossibleLeEt$ represents the statistical possibility;
$\PossibleLeEt \phi$ means that we think a null hypothesis $\phi$ may be true after a hypothesis test $\alg$ did not reject $\phi$ with a significance level $\err$.
Formally:
\begin{align*}
w \models \PossibleLeEt \phi
& \,\mbox{ iff }\,
\mbox{there is a $w'$ s.t. $(w, w') \in \relO$ and }
w' \not\models \neg\phi \lor \TestE{y,\alg}(\err) \lor \neg \cmpds{y,\alg}
\\[-0.5ex] & \,\mbox{ iff }\,
w \models \Possible (\phi \land \neg\TestE{y,\alg}(\err) \land \cmpds{y,\alg})
{.}
\end{align*}

We obtain the following basic properties of statistical beliefs.

\begin{restatable}[Basic properties of statistical beliefs]{prop}{PropStatBelief}
\label{prop:StatBelief}
Let $y\in\varO$, $\varepsilon, \varepsilon' \in\realsnng$, and $\bowtie\, \in \{ =, \le, \ge, <, > \}$.
Let $f_{\alg}$ be a program for a hypothesis test $\alg$ with an alternative hypothesis~$\phi$.
\begin{enumerate}
\item {\rm (\propNu)} \label{item:HT:neg}
The output of $f_{\alg}$ is the $p$-value of the hypothesis test $\alg$ on the dataset $y$; i.e., $\models \NegA{y,\alg}(f_{\alg}(y))$.

\item {\rm (\propSBfour)} If we believe $\phi$ based on a test $\alg$, then we know this statistical belief; i.e.,
$\models \KnowBowEt\phi \rightarrow \Know\KnowBowEt\phi$.
\item {\rm (\propSBfive)} If we failed to reject $\phi$ and think it possible, then we know this possibility; i.e.,
$\models \PossibleBowEt\phi \rightarrow \Know\PossibleBowEt\phi$.
\item {\rm (\propSBk)} Knowledge is also regarded as belief: $\models \Know\phi \rightarrow \KnowBowEt\phi$.
\item {\rm (\propSBdeg)}
If $\err \le \err'$,
$\models \KnowE\phi \rightarrow \KnowXx{\err'}{y,\alg}\phi$
and $\models \PossibleXx{\err'}{y,\alg}\phi \rightarrow \PossibleE\phi$.
\item {\rm (\propSBf)}
$\KnowE\phi$ may be a false belief.
The alternative hypothesis $\phi$ we believe may be false, i.e., the rejected null hypothesis may be true:
$\err > 0$ iff $\not\models \KnowE\phi \rightarrow \phi$.
\end{enumerate}
\end{restatable}
The proof is shown in~\App{sub:proof:assertions}.

We remark that for $\epsilon = 0$, (\propSBfour{}) and (\propSBfive{}) correspond to the axioms (4) and (5) of the modal logic S5, respectively, and thus the statistical belief modality $\KnowXx{0}{y,\alg}$ coincides with the knowledge modality $\Know$
when the test $\alg$'s requirement $\cmpds{y,\alg}$ is satisfied.

Next, we present the relationships between hypothesis tests and statistical beliefs.
Recall that a formula of the form $\kappaX{S}$ represents a test history.
The following proposition allows for deriving a $p$-value from a test history.

\begin{restatable}[Statistical beliefs by hypothesis tests]{prop}{PropStatBeliefHT}
\label{prop:StatBeliefHT}
Let $y_1,y_2\in\varO$.
Let $f_{\alg_1}$ and $f_{\alg_2}$ be programs for hypothesis tests $\alg_1$ and $\alg_2$ with alternative hypotheses $\phi_1$ and $\phi_2$, respectively.
Let $S = \{ (y_1,\alg_1), (y_2,\alg_2)\}$.
\begin{enumerate}
\item {\rm (\propBHk)}
For any $S' \subseteq \var\times\cala$, we have 
$\models \kappaX{S'} \leftrightarrow \Possible \kappaX{S'} \leftrightarrow \Know \kappaX{S'}$.
\item {\rm (\propBHT)}
Let $y\in\varO$ and $f_{\alg}$ be a program for a hypothesis tests $\alg$ with an alternative hypothesis $\phi$.
If we execute the test $\alg$ on the dataset $y$, then we obtain the statistical belief in $\phi$ with the $p$-value $f_{\alg}(y)$; i.e.,
$\models \kappaYA{y} \rightarrow \KnowXt{f_{\alg}(y)} \phi$.
\item {\rm (\propBHTor)}
Let $y \eqdef (y_1, y_2)$, $\alg$ be the disjunctive combination of $\alg_1$ and $\alg_2$, and
$\epsilon \eqdef f_{\alg_1}(y_1) + f_{\alg_2}(y_2)$.
If we execute $\alg_1$ on the dataset $y_1$ and $\alg_2$ on $y_2$ separately, then we obtain the statistical belief in $\phi_1 \lor \phi_2$ with the $p$-value at most 
$\epsilon$; 
i.e.,
$\models \kappaS \rightarrow \KnowXx{\le \epsilon}{y, \alg} (\phi_1 \lor \phi_2)$.
\item {\rm (\propBHTand)}
Let $y \eqdef (y_1, y_2)$, $\alg$ be the conjunctive combination of $\alg_1$ and $\alg_2$, and
$\epsilon' \eqdef \min(f_{\alg_1}(y_1),\allowbreak f_{\alg_2}(y_2))$.
If we execute $\alg_1$ on the dataset $y_1$ and $\alg_2$ on $y_2$ separately, then we obtain the statistical belief in $\phi_1 \land \phi_2$ with the $p$-value at most $\epsilon'$; i.e.,
$\models \kappaS \rightarrow \KnowXx{\le \epsilon'}{y, \alg} (\phi_1 \land \phi_2)$.
\end{enumerate}
\end{restatable}
The proof is shown in \App{sub:proof:assertions}.

Intuitively, (\propBHk) implies that the analysts know the history $S$ of all previously executed hypothesis tests.
Otherwise, they could not check whether a $p$-value is correctly calculated from the history $S$.
(\propBHT) derives a statistical belief from a history $\{ (y, \alg) \}$ consisting of a single hypothesis test, 
while (\propBHTor) and (\propBHTand) derive statistical beliefs from histories $S$ of multiple hypothesis tests.
In \Sec{sub:derived:rules}, we use these properties to obtain helpful derived rules for belief Hoare logic.

\section{Belief Hoare Logic for Hypothesis Testing}
\label{sec:belief-hoare-logic}
We introduce \emph{belief Hoare logic} (BHL) for formalizing and reasoning about statistical inference using hypothesis tests.
We define the notions of judgments and partial correctness (\Sec{sub:BHL:triples}) and the inference rules of BHL (\Sec{sub:rules}).
We then show the soundness and relative completeness of BHL (\Sec{sub:completeness}) and present useful derived rules for typical forms of hypothesis testing (\Sec{sub:derived:rules}).

\begin{figure}[t]
\begin{small}
\infax[Skip]{
  \env\vdash\triple{\psi}{\myskip{}}{\psi}
}

\rulesp

\infrule[UpdVar]
{\env(v) = \env(e)}
{
  \env\vdash\triple{\phi[v \mapsto e]}{v := e}{\phi}
}

\rulesp

\infrule[Seq]{
  \env\vdash\triple{\psi}{C_1}{\psi'}
  \andalso
  \env\vdash\triple{\psi'}{C_2}{\phi}
}{
  \env\vdash\triple{\psi}{C_1; C_2}{\phi}
}

\rulesp

\infrule[If]{
  \env\vdash\triple{\psi\land e}{C_1}{\phi}
  ~~~~~
  \env\vdash\triple{\psi\land \neg e}{C_2}{\phi}
}{
  \env\vdash\triple{\psi}{\myIf{e}{C_1}{C_2}}{\phi}
}

\rulesp

\infrule[Loop]{
  \env\vdash\triple{\psi \land e}{C}{\psi} 
}{
  \env\vdash\triple{\psi}{\myLoop{e}{C}}{\psi \land \neg e}
}

\rulesp

\infrule[Conseq]{
\hspace{-0.3ex}
  \env\models\!\psi \!\implies\! \psi'
~~
  \env\vdash\!\triple{\psi'}{\!C\!}{\phi'}
~~
  \env\models\!\phi' \!\implies\! \phi
\hspace{-0.3ex}
}{
  \env\vdash\!\triple{\psi}{\!C\!}{\phi}
}
\end{small}
\caption{Axioms and rules for basic constructs for commands.}
\label{fig:rules:basic}
\end{figure}

\begin{figure}[t]
\hspace{-4ex}
\begin{minipage}{50ex}
\centering
\begin{small}
\infrule[\axHist]{\begin{array}{c}
  \hspace{-5ex}
  \{ h_{y,\alg}:\nats \} \subseteq \envI,~~
  \{ y\!:\typelist{\dompop},\, v\!:[0,1] \} \subseteq \envO,
  \hspace{-4ex}
\\
  \psiPre \eqdef \psi[v \,{\mapsto} f_{\alg}(y),\, h_{y,\alg} \mapsto (h_{y,\alg} + 1)]
\end{array}
}{
  (\envI, \envO) \vdash
  \triple{ \psiPre }{v := f_{\alg}(y)}{ \psi }
}
\end{small}
\end{minipage}
\hspace{-8ex}
\hfill
\begin{minipage}{33ex}
\begin{small}
\infrule[\axPAR]{
\begin{array}{c}
  \env \vdash\triple{\psi}{C_1; C_2}{\psi'}
\end{array}
}{
  \env \vdash\triple{\psi}{C_1 \,\|\, C_2}{\psi'}
}
\end{small}
\end{minipage}
\caption{An axiom and a rule for hypothesis tests.
(\axHist{}) is the axiom for hypothesis tests. (\axPAR{}) is the rule for exchanging the sequential composition with the parallel composition.}
\label{fig:rules:HT}
\end{figure}

\subsection{Hoare Triples}
\label{sub:BHL:triples}

We define an \emph{environment} as a pair $\env = (\envI, \envO)$ 
consisting of 
an \emph{invisible environment} $\envI$ and an \emph{observable environment} $\envO$ that assign types to invisible variables and to observable variables, respectively.
We write $\env \models \varphi$ if $\M, w \models \varphi$ for any model $\M$ and any world $w$ that respects the type information in $\env$ (i.e., the type of $w(v)$ being $\env(v)$ for any $v\in\var$).
Let $\Env$ be the set of all possible environments.

A \emph{judgment} is of the form $\env\vdash\triple{\psi}{C}{\phi}$ where $\env\in\Env$, $\psi,\phi\in\Fml$, and $C\in\Prog$.
Intuitively, this represents that whenever the \emph{precondition} $\psi$ is satisfied, executing the program $C$ results in satisfying the \emph{postcondition} $\phi$ if $C$ terminates.

We say that 
a judgment $\env\vdash\triple{\psi}{C}{\phi}$ is \emph{valid} iff 
for any model $\M$ and any possible world $w$, if 
$\M, w\models\psi$, then
$\M, w' \models\phi$ for all $w' \in \sem{C}(w)$.
A valid judgment $\env\vdash\triple{\psi}{C}{\phi}$ expresses the \emph{partial correctness} of the program $C$: It respects the precondition $\psi$ and the postcondition $\phi$ up to the termination of $C$.

\begin{figure*}[t]
\centering
\begin{small}
\begin{gather*}
\dfrac{\begin{array}{c}
  \{ h_{y,\algt{\sideT}}:\nats \} \subseteq \envI,~~
  \{ y\!:\typelist{\dompop},\, \alpha\!:[0,1] \} \subseteq \envO,~~ \\
  \alpha, h_{y,\algt{\sideT}} \not\in \fv{ \{ \phiL, \phiU, \psi \} },~~
  \Gamma \models \psi \implies (\cmpds{y,\algt{\sideT}} \land \Possible \phiL \land \Possible \phiU)
\end{array}
}{
  (\envI, \envO) \vdash
  \triple{
  \psi \land \kappaX{\emptyset}
  }{\alpha := f_{\!\algt{\sideT}}(y)}{
  \psi \land \kappaX{y,\algt{\sideT}} \land \KnowXx{\alpha}{y,\algt{\sideT}} \phiT
  }
}
\tag{$\axTT{}$}
\\[0.8ex]
\dfrac{\begin{array}{c}
  \{ h_{y,\algt{\sideL}}:\nats \} \subseteq \envI,~~
  \{ y\!:\typelist{\dompop},\, \alpha\!:[0,1] \} \subseteq \envO,~~ \\
  \alpha, h_{y,\algt{\sideL}} \not\in \fv{ \{ \phiL, \phiU, \psi \} },~~
  \Gamma \models \psi \implies (\cmpds{y,\algt{\sideL}} \land \Possible \phiL \land \neg\Possible \phiU)
\end{array}
}{
  (\envI, \envO) \vdash
  \triple{
  \psi \land \kappaX{\emptyset}
  }{\alpha := f_{\!\algt{\sideL}}(y)}{
  \psi \land \kappaX{y,\algt{\sideL}} \land \KnowXx{\alpha}{y,\algt{\sideL}} \phiL
  }
}
\tag{$\axLT{}$}
\\[0.8ex]
\dfrac{\begin{array}{c}
  \{ h_{y,\algt{\sideU}}:\nats \} \subseteq \envI,~~
  \{ y\!:\typelist{\dompop},\, \alpha\!:[0,1] \} \subseteq \envO,~~ \\
  \alpha, h_{y,\algt{\sideU}} \not\in \fv{ \{ \phiL, \phiU, \psi \} },~~
  \Gamma \models \psi \implies (\cmpds{y,\algt{\sideU}} \land \neg \Possible \phiL \land \Possible \phiU)
\end{array}
}{
  (\envI, \envO) \vdash
  \triple{
  \psi \land \kappaX{\emptyset}
  }{\alpha := f_{\!\algt{\sideU}}(y)}{
  \psi \land \kappaX{y,\algt{\sideU}} \land \KnowXx{\alpha}{y,\algt{\sideU}} \phiU
  }
}
\tag{$\axUT{}$}
\\[0.8ex]
\dfrac{\begin{array}{c}
  \{ h_{y_2,\alg_2}:\nats \} \subseteq \envI,~~
  \{ y_1:\typelist{\dompop_1},\, y_2:\typelist{\dompop_2},\, \alpha_1:[0,1],\, \alpha_2:[0,1] \} \subseteq \envO,~~
y = (y_1, y_2) 
\\[0.2ex]
  \!\alpha_1,\alpha_2, h_{y_2,\alg_2} {\not\in}\, \fv{ \{ \phi_1, \phi_2, \psi \} },~
  S \,{=}\, \{ (y_1, \alg_1), (y_2, \alg_2) \},~
  \Gamma \,{\models}\, \psi \,{\implies}\, (\cmpds{y_2,\alg_2} {\land}\, \Possible (\phi_1 {\lor}\, \phi_2))
\end{array}}{
  (\envI, \envO) \vdash\triple{\psi \land \kappaX{y_1\!,\alg_1} \land  \KnowXx{\alpha_1}{y_1\!, \alg_1}\!\phi_1}{ \alpha_2 \mathbin{:=} f_{\alg_2}(y_2) }{\psi \land \kappaX{S} \land \KnowXx{\le\alpha_1+\alpha_2}{y,\alg} (\phi_1 {\lor}\, \phi_2)}
}
\tag{\ruleOR{}}
\\[0.0ex]
\dfrac{\begin{array}{c}
  \{ h_{y_2,\alg_2}:\nats \} \subseteq \envI,~~
  \{ y_1:\typelist{\dompop_1},\, y_2:\typelist{\dompop_2},\, \alpha_1:[0,1],\, \alpha_2:[0,1] \} \subseteq \envO,~~
y = (y_1, y_2) 
\\[0.2ex]
  \!\alpha_1,\alpha_2, h_{y_2,\alg_2} {\not\in}\, \fv{ \{ \phi_1, \phi_2, \psi \} },~
  S \,{=}\, \{ (y_1, \alg_1), (y_2, \alg_2) \},~
  \Gamma \,{\models}\, \psi \,{\implies}\, (\cmpds{y_2,\alg_2} {\land}\, \Possible (\phi_1 {\land}\, \phi_2))
\end{array}}{
  \!(\envI, \envO) \vdash\triple{\psi \land \kappaX{y_1\!,\alg_1} \land  \KnowXx{\alpha_1}{y_1\!, \alg_1}\!\phi_1}{ \alpha_2 \mathbin{:=} f_{\alg_2}(y_2) }{\psi \land \kappaX{S} \land \KnowXx{\le\min(\alpha_1\!, \alpha_2)}{y,\alg} (\phi_1 {\land}\, \phi_2)}
}
\tag{\ruleAND{}}
\end{gather*}
\vspace{-1.5ex}
\end{small}
\caption{
(\axTT{}), (\axLT{}), and (\axUT{}) are derived rules for a two-tailed test $\algt{\sideT}$, a lower-tailed $\algt{\sideL}$, and an upper-tailed $\algt{\sideU}$, respectively, 
where $\phiT$, $\phiL$, $\phiU$ are alternative hypotheses (\Sec{sub:properties:hypo}) and $\kappaX{\emptyset}$ is given in \eqref{eq:kappaS}.
(\ruleOR{}) is for the Bonferroni's method with the disjunctive combination $\alg$ of two tests $\alg_1$ and $\alg_2$.
(\ruleAND{}) is for the conjunctive combination $\alg$ of $\alg_1$ and $\alg_2$.}
\label{fig:derived:rules:HT}
\end{figure*}

\subsection{Inference Rules}
\label{sub:rules}

We define the inference rules for belief Hoare logic (BHL).
The rules consist of those for basic command constructs (\Fig{fig:rules:basic}) and for hypothesis tests (\Fig{fig:rules:HT}).

The rules in \Fig{fig:rules:basic} for the basic constructs are the same as those for a standard imperative programming language;
the readers are referred to a standard textbook on the Hoare logic~\cite{winskel} for details.
We add the following remarks to a few rules:
\begin{itemize}
\item In the rules (\ruleIf) and (\ruleLoop), the guard condition $e$ is a Boolean expression implicitly used as a logical predicate in the preconditions and the postconditions.
Translating a Boolean expression into an \ELHT{} assertion is straightforward.
\item 
The rule (\ruleConseq) refines the precondition and relax the postcondition of a triple.
The relation $\Gamma \models \varphi$ 
is used in this rule.
\end{itemize}

The rules in \Fig{fig:rules:HT} are characteristic of BHL.
(\axHist{}) describes the properties of an execution of a hypothesis test command $f_{\alg}$ on a dataset $y$.
Essentially, this rule states that the precondition is obtained by substituting the $p$-value $f_{\alg}(y)$ for the variable $v$ in the postcondition $\psi$.
(\axHist{}) differs from (\axUpdVar) in that
an execution of $f_{\alg}$ on $y$ also 
increases the history variable $h_{y,\alg}$ by $1$.
Recall that $h_{y,\alg}$ denotes the number of all executions of $f_{\alg}$ on $y$ and is updated only by an execution of $f_{\alg}$ on~$y$.

The rule (\axPAR{}) in \Fig{fig:rules:HT} exchanges the sequential composition $C_1; C_2$ with the parallel composition $C_1 \,\|\, C_2$.
We recall that in \Sec{sub:program:syntax}, the restriction $\upd{C_1} \cap \var(C_2) = \upd{C_2} \cap \var(C_1) = \emptyset$ is imposed to ensure that an execution of $C_1$ does not interfere with that of $C_2$, and vice versa.

\subsection{Soundness and Relative Completeness}
\label{sub:completeness}

We show that BHL satisfies soundness and relative completeness as follows.

\begin{restatable}[Soundness]{thm}{PropSound}
\label{thm:PropSound}
Every derivable judgment is valid.
\end{restatable}
We prove \Thm{thm:PropSound} in \App{sub:proof:soundness}.

In contrast, BHL is \emph{not} complete.
As with the standard Hoare logic, 
the rule (\ruleConseq{}) uses the validity of assertions $\Gamma \models \psi \implies \psi’$ and $\Gamma \models \varphi’ \implies \varphi$ as assumptions, which may not have finite proofs because the assertion logic is not complete due to arithmetic.

However, BHL is \emph{relatively complete}~\cite{Cook:78:siamcomp}:
every valid judgment has a finite proof using inference rules of BHL, 
except for the proofs for the assertions that appear as premises in (\ruleConseq{}).

\begin{restatable}[Relative completeness]{thm}{PropComplete}
\label{thm:PropComplete}
Every valid judgment is derivable
except for the proofs for assertions.
\end{restatable}
We prove 
\Thm{thm:PropComplete} in \App{sub:proof:complete}.

\subsection{Remarks on Decidability}
\label{sub:decidability}

We discuss the decidability of BHL, the assertion logic, and its fragments as follows.

We first remark that BHL is \emph{undecidable}, i.e., there is no effective method for determining whether an arbitrary judgment $\env\vdash\triple{\phi}{C}{\phi'}$ is derivable using BHL's inference rules.
This undecidability follows from the undecidability of the halting problem;
that is, if BHL were decidable, there would be an algorithm that could derive 
$\env\vdash\triple{\mytrue{}}{C}{\myfalse{}}$ for an arbitrary program $C$ written in the Turing-complete language $\Prog$,
i.e., that could determine whether an arbitrary program $C$ terminates or not (hence a contradiction).
Nevertheless, undecidable program logic in general is known to be practically useful for real-world programs in many logic-based formal verification techniques, 
such as \cite{Platzer:18:book,Hahnle:19:TCSGI}.

Furthermore, our assertion logic (\Sec{sub:epistemic-language}) is \emph{undecidable}, since it subsumes the first-order logic with arithmetic.
Unlike the first-order logic, even the two-variable, monadic fragment of our assertion logic is also undecidable, because the two-variable, monadic fragment of first-order modal logic is proven to be undecidable when we consider a Kripke frame with a world that is related to infinitely many worlds~\cite{Kripke:62:MLQ,Hughes:96:book-modal-logic}.%

Finally, we remark that 
the fragment of our assertion logic without quantifiers over $\IntVar$ and without arithmetic is \emph{decidable} due to the decidability of the propositional modal logic S5.
In practice, the quantifier-free fragment of our assertion logic can describe the pre-/post-conditions of many popular hypothesis testing methods that involve no loops.
As we illustrate in \Sec{sec:reasoning-with-BHL}, our BHL can efficiently reason about practical issues without loops, such as $p$-value hacking and multiple comparison problems.

\subsection{Derived Rules}
\label{sub:derived:rules}
In \Fig{fig:derived:rules:HT}, we show useful derived rules for typical forms of hypothesis tests.
These can be instantiated to a variety of concrete testing methods;
see \App{sub:instantiation}.

\subsubsection{Derived Rules for Single Hypothesis Tests}
\label{sub:rules:HT}

The derived rules (\axTT{}), (\axLT{}), and (\axUT{})
correspond to two-tailed, lower-tailed, and upper-tailed hypothesis tests, respectively.
Recall that the formulas $\phiL$, $\phiU$, and $\phiT (\eqdef \phiL \lor \phiU)$ denote the alternative hypotheses for the lower-tailed, upper-tailed, and two-tailed tests, respectively (\Sec{sub:properties:hypo})
and $\kappaX{\emptyset}$ is given in \eqref{eq:kappaS}.

The derived rule  (\axTT{}) states that we can perform a two-tailed test program $f_{\algt{\sideT}}$ on a dataset $y$ 
if we have the prior belief $\Possible \phiL \land \Possible \phiU$ that both the lower-tail $\phiL$ and upper-tail $\phiU$ are possible before performing the test (\Sec{sub:prior:beliefs}).
If the test $f_{\algt{\sideT}}$ on $y$ returns a $p$-value $\alpha \in [0,1]$, we obtain the statistical belief in the alternative hypothesis $\phiT$ denoted by $\KnowXx{\alpha}{y,\algt{\sideT}} \phiT$.
The derivation of (\axTT{}) is given by:

\begin{center}
\begin{small}
\centering
\hspace{0ex}
\infer[\scriptsize\ruleConseq{}]{
  \env \vdash
  \triple{ \psi \land \kappaX{\emptyset}
  }{\alpha := f_{\!\algt{\sideT}}(y)}{ \psi \land \kappaX{y,\algt{\sideT}} \land \KnowXx{\alpha}{y,\algt{\sideT}} \phiT
  }
}
{
  \infer[\scriptsize\axHist]{
    \env \vdash
    \triple{ \psi \land \kappaX{\emptyset}
    }{\alpha := f_{\!\algt{\sideT}}(y)}{ \psi \land \kappaX{y,\algt{\sideT}}
    }
  }{
    \{ h_{y,\algt{\sideT}}:\nats \} \subseteq \envI,~~
    \{ y\!:\typelist{\dompop},\, \alpha\!:[0,1] \} \subseteq \envO,~~
    \alpha, h_{y,\algt{\sideT}} \not\in \fv{ \{ \phiL, \phiU, \psi \} }
  }
}
\end{small}
\end{center}
where (\ruleConseq) uses \Propo{prop:StatBeliefHT} (\propBHT{}).
We remark that the derivation does not use $\Gamma \models \psi \implies (\cmpds{y,\algt{\sideT}} \land \Possible \phiL \land \Possible \phiU)$.
However, for the two-tailed test $\algt{\sideT}$ to be meaningful, the postcondition must imply $\cmpds{y,\algt{\sideT}} \land \Possible \phiL \land \Possible \phiU$, 
as mentioned in \Sects{sub:test:meaningful} and~\ref{sub:prior:beliefs}.

If we have the prior belief $\Possible \phiL \land \neg \Possible \phiU$ (resp. $\neg \Possible \phiL \land \Possible \phiU$)
that only the lower-tail $\phiL$ (resp. upper tail $\phiU$) is possible,
then we can apply (\axLT{}) (resp. (\axUT{})) and obtain the statistical belief $\KnowXx{\alpha}{y,\algt{\sideL}} \phiL$ (resp. $\KnowXx{\alpha}{y,\algt{\sideU}} \phiU$).
The derivations of (\axLT{}) and (\axUT{}) are similar to that of (\axTT{}).

\subsubsection{Derived Rules for Multiple Hypothesis Tests}
\label{sub:rules:multiple:HT}

The derived rule (\ruleOR{}) corresponds to the reasoning about two tests $\alg_1$ on $y_1$ and $\alg_2$ on $y_2$ with a \emph{disjunctive} alternative hypothesis $\phi_1 \lor \phi_2$.
As illustrated in \Sec{sec:overview}, a typical example is to test whether a drug has better efficacy than \emph{at least one} of two drugs.

The precondition in (\ruleOR{}) expresses that we have obtained a statistical belief $\KnowXx{\alpha_1}{y_1,\alg_1} \phi_1$ in an alternative hypothesis $\phi_1$ with a $p$-value $\alpha_1$.
If we obtain an output $\alpha_2$ of the second test $\alg_2$, we cannot conclude that $\alpha_2$ is the $p$-value for $\phi_2$,
because the $p$-value when performing the two tests $\alg_1$ and $\alg_2$ simultaneously is larger than $\alpha_1$ and $\alpha_2$.
This is known as the \emph{multiple comparison problem}.

The \emph{Bonferroni's method} is the best-known way to calculate $p$-values for multiple tests~\cite{Bretz:10:book}.
By applying this method, the $p$-value in total is bounded above by $\alpha_1+\alpha_2$; i.e.,
we obtain a statistical belief $\KnowXx{\le\alpha_1+\alpha_2}{y,\alg} (\phi_1\lor\phi_2)$ in the alternative hypothesis $\phi_1 \lor \phi_2$
with the dataset $y \eqdef (y_1, y_2)$.
In BHL, the derived rule (\ruleOR{}) guarantees the correct application of the Bonferroni's method;
i.e., the inference using BHL does \emph{not} make  elementary mistakes (e.g., $\KnowXx{\alpha_2}{y,\alg} \phi_2$) where the reported $p$-value $\alpha_2$ is lower than the actual $p$-value in the multiple comparison.
The derivation for (\ruleOR) is given by:

\begin{center}
\hspace{2ex}
\begin{small}
\centering
\infer[\scriptsize\ruleConseq{}]{
  \env \vdash\triple{\psi \land \kappaX{y_1\!, \alg_1} \land  \KnowXx{\alpha_1}{y_1\!, \alg_1}\!\phi_1}{ \alpha_2 \mathbin{:=} f_{\alg_2}(y_2) }{\psi \land \kappaS \land \KnowXx{\le\alpha_1+\alpha_2}{y,\alg} (\phi_1\lor\phi_2)}
}
{
  \infer[\scriptsize\axHist]{
  \hspace{-10ex}
    \env \vdash\triple{\psi \land \kappaX{y_1\!, \alg_1} }{ \alpha_2 \mathbin{:=} f_{\alg_2}(y_2) }{\psi \land \kappaS}
  }{
    \alpha_1, \alpha_2, h_{y_2\!,\alg_2} \not\in \fv{ \{ \phi_1, \phi_2, \psi \} },~
    S = \{ (y_1, \alg_1), (y_2, \alg_2) \}
  }
}
\end{small}
\end{center}
where (\ruleConseq) uses
\Propo{prop:StatBeliefHT} 
(\propBHTor{}).

In contrast, the derived rule  (\ruleAND{}) formalizes the reasoning about multiple tests
with a \emph{conjunctive} alternative hypothesis 
$\phi_1 \land \phi_2$
(e.g., the program $C_{\rm drug}$ in \Eg{eg:illustrate}, which tests whether a drug has better efficacy than \emph{both} drugs).
According to statistics textbooks (e.g., \cite{Bretz:10:book}), this does not make the $p$-value higher,
i.e., the $p$-value is at most $\min(\alpha_1, \alpha_2)$.
(\ruleAND{}) guarantees the correct procedure for conjunctive hypotheses.
The derivation for (\ruleAND) is given by:

\begin{center}
\begin{small}
\centering
\hspace{0ex}
\infer[\scriptsize\ruleConseq{}]{
  \env \vdash\triple{\psi \land \kappaX{y_1\!, \alg_1} \land  \KnowXx{\alpha_1}{y_1\!, \alg_1}\!\phi_1}{ \alpha_2 \mathbin{:=} f_{\alg_2}(y_2) }{\psi \land \kappaS \land \KnowXx{\le \min(\alpha_1,\alpha_2)}{y,\alg} (\phi_1\land\phi_2)}
}
{
  \infer[\scriptsize\axHist]{
  \hspace{-13ex}
    \env \vdash\triple{\psi \land \kappaX{y_1\!, \alg_1} }{ \alpha_2 \mathbin{:=} f_{\alg_2}(y_2) }{\psi \land \kappaS }
  }{
    \alpha_1, \alpha_2, h_{y_2\!,\alg_2} \not\in \fv{ \{ \phi_1, \phi_2, \psi \} },~
    S = \{ (y_1, \alg_1), (y_2, \alg_2) \}
  }
}
\end{small}
\end{center}
where (\ruleConseq) uses 
\Propo{prop:StatBeliefHT}
(\propBHTand{})
and $y \eqdef (y_1, y_2)$.

\section{Reasoning About Hypothesis Testing Procedures Using BHL}
\label{sec:reasoning-with-BHL}
In this section, we apply our framework to the reasoning about \emph{$p$-value hacking} and \emph{multiple comparison problems} using BHL.

\subsection{Reasoning About $p$-Value Hacking}
\label{sub:p-hacking}

The \emph{$p$-value hacking} (a.k.a. \emph{data dredging}) is a scientifically malignant technique to obtain a low $p$-value.
A typical example is to conduct hypothesis tests on different datasets and ignore the experiment showing a higher $p$-value to report only a lower.

Our framework can describe and reason about programs for $p$-value hacking.
For example, the following program $C_{\phack}$ conducts a hypothesis test $\alg_1$ on a dataset $y_1$ and another $\alg_2$ on $y_2$, 
and reports only a lower $p$-value $\alpha$ while ignoring the higher:
\begin{align*}
    C_{\phack} \eqdef~ &
    (\alpha_1 := f_{\alg_1}(y_1) \,\,\|\,\, \alpha_2 := f_{\alg_2}(y_2));
    \\[-0.3ex] & 
    \myIf{\alpha_1 < \alpha_2}{\alpha := \alpha_1}{\alpha := \alpha_2}.
\end{align*}
We write $\phi_1$ and $\phi_2$ for the alternative hypotheses of the tests 
$\alg_1$ and $\alg_2$, respectively.

Based on the discussion on the prior knowledge in \Sec{sub:test:meaningful},
we assume that we do not have the prior knowledge that these hypotheses are true or the dataset did not come from the population satisfying the requirements of the tests; that is, we have:
\begin{align} \label{eq:phack:assumption}
\psiPre \eqdef
\neg \Know (\phi_1 \lor \neg \cmpds{y_1,\alg_1}) \land \neg \Know (\phi_2 \lor \neg \cmpds{y_2,\alg_2}).
\end{align}

For the reported value $\alpha$ to be an actual $p$-value, the formula
\[
\psi_{\sf post}^{\rm hack} \eqdef
\KnowXx{\le\alpha}{y_1, \alg_1} \phi_1 \lor\, \KnowXx{\le\alpha}{y_2, \alg_2} \phi_2
\] 
needs to hold as a postcondition of $C_{\phack}$.
Thus, at the end of the first line of $C_{\phack}$,
\begin{small}
\begin{align*}
\big( \alpha_1 < \alpha_2 \implies (\KnowXx{\le\alpha_1}{y_1, \alg_1} \phi_1 \lor \KnowXx{\le\alpha_1}{y_2, \alg_2} \phi_2) \big) \land 
\big( \alpha_1 \ge \alpha_2 \implies (\KnowXx{\le\alpha_2}{y_1, \alg_1} \phi_1 \lor \KnowXx{\le\alpha_2}{y_2, \alg_2} \phi_2) \big)
\end{align*}
\end{small}%
must hold  due to the rules (\axUpdVar) and (\ruleIf).
By applying (\ruleConseq) and
the definition of the statistical belief modality,
the following formula needs to hold:
\begin{small}
\begin{align*}
\Know ( \phi_1  \lor  \TestLE{y_1,\alg_1}(\alpha) \lor \neg \cmpds{y_1,\alg_1} )
\lor
\Know ( \phi_2  \lor  \TestLE{y_2,\alg_2}(\alpha) \lor \neg \cmpds{y_2,\alg_2} ).
\end{align*}
\end{small}%
By assumption \eqref{eq:phack:assumption}, this formula implies 
$\Possible \kappaX{y_1, \alg_1} \lor \Possible \kappaX{y_2, \alg_2}$.
By \Propo{prop:StatBeliefHT} (\propBHk), we obtain $\kappaX{y_1, \alg_1} \lor \kappaX{y_2, \alg_2}$;
i.e., only one of the two hypothesis tests has been conducted.

However, by applying (\axPAR) and (\axHist) to $C_{\phack}$'s first line, $\kappaX{\{(y_1, \alg_1), (y_2, \alg_2)\}}$ needs to be satisfied; i.e., both the tests must have been conducted.
Hence a contradiction.
Therefore, we cannot conclude that the reported value $\alpha$ is the actual $p$-value.

Instead, 
for $y \eqdef (y_1, y_2)$
and the disjunctive combination $\alg$ of $\alg_1$ and $\alg_2$,
we derive that $\KnowXx{\le \alpha_1+\alpha_2}{y, \alg} (\phi_1 \lor \phi_2)$ 
is a postcondition of $C_{\phack}$
by using the derived rule (\ruleOR{}).
Therefore, the total $p$-value $\alpha_1+ \alpha_2$ should be reported without ignoring any experiments.

\begin{figure}[t]
\vspace{1ex}
\begin{small}
\centering
\infer[\!\scriptsize\ruleSeq]{
  \env\vdash\triple{ \psiPre }{C_{12}; \!\myIfsubstack{\alpha_{12} {\le} 0.05\!}{C_{13}}{\myskip{}}\!}{ \phiPost }
}{
  \begin{scriptsize}
  \infer[\!\scriptsize\axTT]
  {\env\vdash\triple{ \psiPre }{C_{12}}{\psiPostAB}}
  { \alpha_{12}, h_{y'\!,\alg_{12}} {\not\in}\, \fv{ \{\phi_{12}, \psi \} } }
  \end{scriptsize}
  &\hspace{-12.5ex}
  \infer[\!\scriptsize\ruleIf]{\hspace{3ex}
    \env\vdash\triple{\psiPostAB}{\!\myIfsubstack{\alpha_{12} {\le} 0.05\!}{C_{13}}{\myskip{}}\!}{ \phiPost }
  }{\hspace{-1.5ex}
    \infer[\!\scriptsize\ruleConseq]{
      \hspace{1.5ex}\env\vdash\triple{ \substack{\psiPostAB \land \\[0.1ex] \alpha_{12} {\le} 0.05} }{C_{13}}{ \phiPost }
    }{
      \hspace{-2.5ex}
      \begin{scriptsize}
      \begin{array}{l}
        \infer[\!\scriptsize\ruleAND{}]
        {\hspace{-3.1ex}\env\vdash\triple{ \psiPostAB }{C_{13}}{ \psiPostAC }}
        { \alpha_{12}, \alpha_{13}, h_{y''\!,\alg_{13}} {\not\in}\, \fv{ \{\phi_{12}, \phi_{13}, \psi \} } }
        \\[0.4ex]
        \hspace{5.5ex}\env\models (\psiPostAB \land \alpha_{12} {\le} 0.05) \rightarrow \psiPostAB \hspace{-0ex}
        \\[0.4ex]
        \hspace{5.5ex}\env\models \psiPostAC \rightarrow \phiPost
      \end{array}
      \end{scriptsize}
    }
    &\hspace{-2.5ex}
    \infer[\!\scriptsize\ruleConseq]{
      \env\vdash\triple{ \substack{ \psiPostAB \land \\[0.1ex] \alpha_{12} {>} 0.05} }{\myskip{}}{ \phiPost }
    }{
      \infer[\!\scriptsize\axSkip]{
        \begin{scriptsize}
        \begin{array}{l}
        \hspace{0.1ex}\env\vdash\triple{ \phiPost }{\myskip{}}{ \phiPost }
        \\[0.4ex]
        \env\models (\psiPostAB \land \alpha_{12} {>} 0.05)\!\rightarrow\! \phiPost
        \end{array}
        \end{scriptsize}
      }{}
    }
  }
}
\end{small}
\hspace{0.0ex}
\begin{minipage}{0.14\textwidth}
\begin{small}
\begin{align*}
\mbox{where }~
\cmpds{i} & \eqdef \sampled{y_i}{\Normal(\mu_i, \sigma^2)}{n_i}
\\[-1.0ex]
\psi & \eqdef \hspace{-1.3ex}\bigwedge_{i=1,2,3} \hspace{-1.5ex} \cmpds{i} \land \Possible (\phi_{12} \land \phi_{13})
\\[-1.0ex]
\psiPre & \eqdef \psi \land \kappaX{\emptyset}
\\[-1.0ex]
S & \eqdef \{ (y', \alg_{12}),\, (y'', \alg_{13}) \}
\end{align*}
\end{small}
\vspace{-5.6ex}
\end{minipage}
\hfill~~
\begin{minipage}{0.56\textwidth}
\begin{small}
\begin{align*}
\alpha & \eqdef \min(\alpha_{12}, \alpha_{13})
\\[-1.0ex]
\psiPostAB & \eqdef \psi \land \kappaX{y'\!, \alg_{12}} \land \KnowXx{\alpha_{12}}{y',\alg_{12}} \phi_{12}
\\[-1.0ex]
\psiPostAC & \eqdef \psi \land \kappaX{S} \land \KnowXx{\le\alpha}{y,\alg} (\phi_{12} \land \phi_{13})
\\[-0.6ex]
\phiPost & \eqdef \KnowXx{\le0.05}{y'\!,\alg_{12}} \phi_{12} \rightarrow \KnowXx{\le\alpha}{y,\alg} (\phi_{12} \land \phi_{13}).
\end{align*}
\end{small}
\vspace{-3.8ex}
\end{minipage}
\caption{An outline of the proof for the illustrating program $C_{\rm drug}$ in \Eg{eg:illustrate}.
}
\label{fig:overview:proof}
\end{figure}

\begin{figure}[t]
\vspace{2ex}
\centering
\begin{small}
\centering
\hspace{0ex}
\infer[\!\scriptsize\axPAR{}]{
  \env\vdash\triple{ \psiPre }{ C_{12} \,\|\, C_{13} }{ \KnowXx{\le \alpha_{12} + \alpha_{13}}{y,\alg} (\phi_{12} {\lor} \phi_{13}) }
}
{
  \infer[\!\scriptsize\ruleSeq]{
    \env\vdash\triple{ \psiPre }{ C_{12};\, C_{13} }{ \KnowXx{\le \alpha_{12} + \alpha_{13}}{y,\alg} (\phi_{12} {\lor} \phi_{13}) }
  }{
    \begin{scriptsize}
    \infer[\!\scriptsize\axTT{}]{
      \env\vdash\triple{ \psiPre }{ C_{12} }{ \psiPostAB }
    }{
      \alpha_{12},\, h_{y'\!,\alg_{12}} \not\in \fv{\{\phi_{12}, \psi \}}
    }
    \end{scriptsize}
    &
    \begin{scriptsize}
    \infer[\!\scriptsize\ruleConseq{}]{
        \hspace{-10.5ex}\env\vdash\triple{ \psiPostAB }{C_{13}}{ \KnowXx{\le \alpha_{12} + \alpha_{13}}{y,\alg} (\phi_{12} {\lor} \phi_{13}) }
    }{
      \infer[\!\scriptsize\ruleOR{}]{
        \env\vdash\triple{ \psiPostAB }{C_{13}}{ \psi \land \kappaX{S} \land \KnowXx{\le \alpha_{12} + \alpha_{13}}{y,\alg} (\phi_{12} {\lor} \phi_{13}) }
      }{
        \alpha_{12},\, \alpha_{13},\, h_{y''\!,\alg_{13}} \not\in \fv{\{\phi_{12}, \phi_{13}, \psi \}}
      }
    }
    \end{scriptsize}
  }
}
\end{small}
\vspace{-1ex}
\caption{An outline of the proof for $C_{12} \,\|\, C_{13}$
where 
$\psi \eqdef \bigwedge_{i=1,2,3} \cmpds{i} \land \Possible (\phi_{12} \lor \phi_{13})$,
$\psiPre \eqdef \psi \land \kappaX{\emptyset}$, and
$\psiPostAB \eqdef \psi \land \kappaX{y'\!, \alg_{12}} \land \KnowXx{\alpha_{12}}{y'\!,\alg_{12}} \phi_{12}$.
}
\label{fig:overview:proof:parallel}
\end{figure}

\subsection{Reasoning About Multiple Comparison with Conjunctive Alternative Hypotheses}
\label{sub:multiple-comparison-conjunct}

We illustrate how BHL reasons about the following program in the multiple comparison in \Eg{eg:illustrate}:
\begin{align*}
C_{\rm drug} \eqdef 
C_{12}; \myIf{\alpha_{12} < 0.05}{C_{13}}{\myskip{}},
\end{align*}
where 
$C_{12} \eqdef (\alpha_{12} := f_{\alg_{12}}(y'))$ is the $Z$-test $\alg_{12}$ on $y' = (y_1, y_2)$ with the alternative hypothesis $\phi_{12}$, and
$C_{13} \eqdef {(\alpha_{13} := f_{\alg_{13}}(y''))}$ is the $Z$-test $\alg_{13}$ on $y'' = (y_1, y_3)$ with $\phi_{13}$.
Let $\alg$ be the conjunctive combination of $\alg_{12}$ and $\alg_{13}$,
and $y \eqdef (y', y'')$.

In this example, the derivation of the judgment 
$\env\vdash\triple{ \psiPre }{ C_{\rm drug} }{ \phiPost }$
given in \eqref{eq:illustrate:Hoare} 
guarantees that the hypothesis tests are applied appropriately in the program $C_{\rm drug}$.

\Fig{fig:overview:proof} shows the derivation tree for this judgment.
In the derivation, we obtain:
\[
  \begin{array}{l}
    \env\vdash\triple{\psiPre}{C_{12}}{\psiPostAB}\\[0.1ex]
    \env\vdash\triple{ \psiPostAB \land \alpha_{12} \le 0.05 }{C_{13}}{ \phiPost }\\[0.1ex]
    \env\vdash\triple{ \psiPostAB \land \alpha_{12} > 0.05 }{\myskip{}}{ \phiPost }
  \end{array}
\]
where 
$\psiPostAB \eqdef (\psi \land \kappaX{y'\!, \alg_{12}} \land \KnowXx{\alpha_{12}}{y'\!, \alg_{12}}\phi_{12})$,
$\alpha \eqdef \min(\alpha_{12}, \alpha_{13})$,
and $\phiPost \eqdef (\KnowXx{\le 0.05}{y'\!, \alg_{12}} \phi_{12} \allowbreak \rightarrow \KnowXx{\le\alpha}{y, \alg} (\phi_{12} \land \phi_{13}))$.
The first judgment is derived using the derived rule (\axTT).
The second judgment is derived by the rules (\ruleAND) and (\ruleConseq).
The last judgment is derived from (\axSkip), (\ruleConseq), and  
$\env\models (\psiPostAB \land \alpha_{12} > 0.05) \rightarrow \phiPost$,
which is obtained by 
$\models \KnowXx{\alpha_{12}}{y'\!, \alg_{12}}\phi_{12} \land \alpha_{12} > 0.05 \rightarrow \neg \KnowXx{\le 0.05}{y'\!, \alg_{12}} \phi_{12}$.
Applying (\ruleIf) to the last two judgments, we have: 
\[
\env\vdash\triple{ \psiPostAB }{\myIf{\alpha_{12} \le 0.05}{C_{13}}{\myskip{}}}{ \phiPost },
\] 
composing it with the first judgment by applying (\ruleSeq), we obtain the judgment in \eqref{eq:illustrate:Hoare}.

\subsection{Reasoning About Multiple Comparison with Disjunctive Alternative Hypotheses}
\label{sub:multiple-comparison-disjunct}

In contrast, the program $C_{\rm multi} \eqdef C_{12} \,\|\, C_{13}$ in \eqref{eq:illustrate:program:disjunctive} has a disjunctive alternative hypothesis $\phi_{12} \lor \phi_{13}$ and thus shows a multiple comparison problem.
\Fig{fig:overview:proof:parallel} show the derivation tree for $C_{\rm multi}$.
Since the alternative hypothesis $\phi_{12} \lor \phi_{13}$ is disjunctive, we apply 
(\ruleOR) to obtain the belief 
$\KnowXx{\le \alpha_{12}+\alpha_{13} }{y,\alg} (\phi_{12} \lor \phi_{13})$,
with a $p$-value 
(larger than $\alpha_{12}$ and $\alpha_{13}$) 
at most $\alpha_{12} + \alpha_{13}$.

\section{Discussion}
\label{sec:discuss}
In this section, we provide the whole picture of the justification of statistical beliefs inside and outside BHL.
A statistical belief derived in a program relies on the following three issues: 
(i) the validity of hypothesis testing methods themselves, 
(ii) the satisfaction of the empirical conditions required for the hypothesis tests, and
(iii) the appropriate usage of hypothesis tests in the program.
In our framework, these are respectively addressed by
(a) the validity of BHL's axioms and rules, 
(b) the (manual) confirmation of the preconditions in a judgment, and
(c) the derivation tree for the judgment.

\subsection{Validity of Hypothesis Testing Methods}

The validity of hypothesis testing methods is not ensured by mathematics alone.
The philosophy of statistics has a long history of argument on 
the proper interpretation of hypothesis testing.
One of the most notable examples is the argument between the frequentist and the Bayesian statistics, which still has many issues to be discussed~\cite{Sober:08:book}.

We also remark that statistical methods occasionally involve approximation of numerical values.
Even when the approximation method has a theoretical guarantee, we may need to confirm the validity of the application of the approximation empirically, e.g., by experiments in the specific situation we apply the statistical methods.

For these reasons, we do not attempt to formalize the ``justification'' for hypothesis testing methods within BHL, and left them for future work.
Instead, we introduce simple (derived) rules that can be instantiated with the hypothesis tests commonly used in practice and explained in textbooks, e.g.,~\cite{Hogg:04:book:ims,Kanji:06:book:100stat}.
Then we focus on 
the logical aspects of the appropriate usage of hypothesis tests, which has been a long-standing, practical concern but has not been formalized using symbolic logic before.

One of the advantages of this approach is that we do not adhere to a specific philosophy of statistics, but can model both the frequentist and the Bayesian statistics by instantiating the derived rules for hypothesis tests (\App{sub:instantiation}).

\subsection{Clarification of Empirical Conditions}
\label{sub:empirical}

The hypothesis testing methods usually assume some empirical conditions on the unknown population from which the dataset is sampled.
Typically, many parametric tests require that the population follows a normal distribution.
For instance, the $Z$-test in \Eg{eg:hypothesis-test}
assumes that the population follows a normal distribution with known variance,
but this cannot be rigorously confirmed or justified in general.

In some cases, such conditions on the unknown population are confirmed approximately or partially (i) by exploratory observations of the sampled data and (ii) by prior knowledge of properties of the population (outside the statistical inference).
However, there is no general method for justifying such empirical conditions rigorously.
Thus, the formal justification of those conditions would require further research in statistics.

In the present paper, the empirical conditions on the unknown population remain to be assumptions from the viewpoint of formal logic.
Hence, we describe empirical conditions as the preconditions of a judgment in BHL.
Explicit specification of the preconditions would be useful to prevent errors in the choice of statistical methods.
Furthermore, when we formalize empirical science in future work, it would be crucial to clarify the empirical conditions that justify scientific conclusions.

\subsection{Epistemic Aspects of Statistical Inference}

One of our contributions is to show that epistemic logic is useful to formalize statistical inference.
Although the outcome of a hypothesis test is the \emph{knowledge} determined by the test action, it may form a false \emph{belief}; i.e.,
a rejected null hypothesis may be true, and a retained one may be false.
Hence, the formalization of statistical inference deals with both truth and beliefs, for which epistemic logic is suitable.

The key to formalizing statistical beliefs is to introduce a Kripke semantics with a possible world where a null hypothesis is true (\Sec{sub:assertion:semantics}).
This possible world may not be the real world where we actually apply the hypothesis test.
Notably, the $p$-value in the test is the probability defined in this possible world, and not in the real world.

Our Kripke semantics is essential for modeling the appropriate usage of hypothesis tests in the real world.
We make a distinction between 
(i) ``ideal'' possible worlds where all requirements for the hypothesis tests are satisfied and
(ii) the real world where hypothesis tests are actually conducted but their requirements may not be satisfied.
Without this distinction, we would deal with only mathematical properties of hypothesis testing methods satisfied in ``ideal'' possible worlds, and could not discuss the appropriateness of the actual application of the hypothesis tests in the real world.

By using this model, we have clarified that statistical beliefs depend on prior beliefs (\Sec{sub:prior:beliefs}).
By using the possibility modality $\Possible$,
certain requirements for hypothesis tests are formalized as \emph{prior beliefs}, which may not be true or confirmed in the real world.
For example, the choice of two-tailed or one-tailed tests depends on the prior belief that both lower-tail and upper-tail are possible before applying the test.

Finally, the update of statistical beliefs by a hypothesis test is modeled using a transition between possible worlds.
Since the world records the history of all hypothesis tests, BHL does not allow for hiding any tests to manipulate the statistics (e.g., in $p$-value hacking and in multiple comparisons in \Sec{sec:reasoning-with-BHL}).

\section{Conclusion}
\label{sec:conclude}
In this work, we proposed a new approach to formalizing and reasoning about statistical inference in programs.
Specifically, we 
introduced belief Hoare logic (BHL) for describing and checking the requirement for applying hypothesis tests appropriately.
We proved that this logic is sound and relatively complete w.r.t. the Kripke model for hypothesis tests.
Then we showed that BHL is useful for reasoning about practical issues in hypothesis tests.
In our framework, we clarified the importance of prior beliefs in acquiring statistical beliefs.
We also discussed the whole picture of the justification of statistical inference.
We emphasize that this appears to be the first attempt to introduce a program logic 
for the appropriate application of hypothesis tests.

In ongoing work, we are extending our framework to other kinds of statistical methods~\cite{Kawamoto:23:JELIA}.
We are also developing a verification tool based on this framework
using the same strategy as the existing verifiers based on Hoare logic:
(i) synthesizing a proof tree using the proof rules in Figure 2,
(ii) discovering the conditions of the form $\Gamma \models \varphi$ that must be valid for the given Hoare triple to hold, and
(iii) discharging the discovered conditions using an external solver.

\section*{Acknowledgments}
The authors are supported by ERATO HASUO Metamathematics for Systems Design Project (No. JPMJER1603), JST.
In particular, we thank Ichiro Hasuo for providing the opportunity for us to meet and collaborate in that project.
Yusuke Kawamoto is supported by JST, PRESTO Grant Number JPMJPR2022, Japan, and by JSPS KAKENHI Grant Number JP21K12028, Japan.
Tetsuya Sato is supported by JSPS KAKENHI Grant Number JP20K19775, Japan.
Kohei Suenaga is supported by JST CREST Grant Number JPMJCR2012, Japan.
We thank Kenji Fukumizu for providing helpful information on hypothesis testing.
We also thank anonymous reviewers and Kentaro Kobayashi for their useful comments on the manuscript.

\bibliography{short}

\begin{thebibliography}{}

\end{thebibliography}


\begin{thebibliography}{10}
\expandafter\ifx\csname url\endcsname\relax
  \def\url#1{\texttt{#1}}\fi
\expandafter\ifx\csname urlprefix\endcsname\relax\def\urlprefix{URL }\fi
\expandafter\ifx\csname href\endcsname\relax
  \def\href#1#2{#2} \def\path#1{#1}\fi

\bibitem{Lang:14:inbook}
T.~A. Lang, D.~G. Altman, Statistical Analyses and Methods in the Published
  Literature: The SAMPL Guidelines, John Wiley \& Sons, Ltd, 2014, Ch.~25, pp.
  264--274.
\newblock \href {https://doi.org/https://doi.org/10.1002/9781118715598.ch25}
  {\path{doi:https://doi.org/10.1002/9781118715598.ch25}}.

\bibitem{Wasserstein:16:AS}
R.~L. Wasserstein, N.~A. Lazar, The {ASA} statement on p-values: Context,
  process, and purpose, The American Statistician 70~(2) (2016) 129--133.
\newblock \href {https://doi.org/10.1080/00031305.2016.1154108}
  {\path{doi:10.1080/00031305.2016.1154108}}.

\bibitem{Kawamoto:21:KR}
Y.~Kawamoto, T.~Sato, K.~Suenaga, Formalizing statistical beliefs in hypothesis
  testing using program logic, in: Proc. {KR}'21, 2021, pp. 411--421.
\newblock \href {https://doi.org/10.24963/kr.2021/39}
  {\path{doi:10.24963/kr.2021/39}}.

\bibitem{Hoare:69:CACM}
C.~A.~R. Hoare, An axiomatic basis for computer programming, Commun. {ACM}
  12~(10) (1969) 576--580.
\newblock \href {https://doi.org/10.1145/363235.363259}
  {\path{doi:10.1145/363235.363259}}.

\bibitem{winskel}
G.~Winskel, The Formal Semantics of Programming Languages---An Introduction,
  The {MIT} Press, 1993.

\bibitem{Apt:19:FAOC}
K.~R. Apt, E.~Olderog, Fifty years of hoare's logic, Formal Aspects Comput.
  31~(6) (2019) 751--807.
\newblock \href {https://doi.org/10.1007/s00165-019-00501-3}
  {\path{doi:10.1007/s00165-019-00501-3}}.

\bibitem{separationLogic}
J.~C. Reynolds, Separation logic: {A} logic for shared mutable data structures,
  in: Proc. {LICS}'02, {IEEE} Computer Society, 2002, pp. 55--74.

\bibitem{infinitesimalKSUenaga}
K.~Suenaga, I.~Hasuo, Programming with infinitesimals: {A} while-language for
  hybrid system modeling, in: Proc. {ICALP}'11, Part {II}, Vol. 6756 of LNCS,
  Springer, 2011, pp. 392--403.
\newblock \href {https://doi.org/10.1007/978-3-642-22012-8\_31}
  {\path{doi:10.1007/978-3-642-22012-8\_31}}.

\bibitem{hoareForProbability}
J.~den Hartog, E.~P. de~Vink, Verifying probabilistic programs using a {H}oare
  like logic, Int. J. Found. Comput. Sci. 13~(3) (2002) 315--340.
\newblock \href {https://doi.org/10.1142/S012905410200114X}
  {\path{doi:10.1142/S012905410200114X}}.

\bibitem{epistemicHoareLogic}
E.~Atkinson, M.~Carbin, Programming and reasoning with partial observability,
  Proc. {ACM} Program. Lang. 4~({OOPSLA}) (2020) 200:1--200:28.
\newblock \href {https://doi.org/10.1145/3428268} {\path{doi:10.1145/3428268}}.

\bibitem{vonWright:51:book}
G.~H. von Wright, An Essay in Modal Logic, Amsterdam: North-Holland Pub. Co.,
  1951.

\bibitem{Hintikka:62:book}
J.~Hintikka, Knowledge and Belief: An Introduction to the Logic of the Two
  Notions, Cornell University Press, 1962.

\bibitem{Fagin:95:book}
R.~Fagin, J.~Halpern, Y.~Moses, M.~Vardi, Reasoning about Knowledge, The MIT
  Press, 1995.

\bibitem{Burrows:90:TOCS}
M.~Burrows, M.~Abadi, R.~M. Needham, A logic of authentication, {ACM} Trans.
  Comput. Syst. 8~(1) (1990) 18--36.
\newblock \href {https://doi.org/10.1145/77648.77649}
  {\path{doi:10.1145/77648.77649}}.

\bibitem{Syverson:99:FM}
P.~F. Syverson, S.~G. Stubblebine, Group principals and the formalization of
  anonymity, in: World Congress on Formal Methods (1), 1999, pp. 814--833.
\newblock \href {https://doi.org/10.1007/3-540-48119-2\_45}
  {\path{doi:10.1007/3-540-48119-2\_45}}.

\bibitem{Garcia:05:FMSE}
F.~D. Garcia, I.~Hasuo, W.~Pieters, P.~van Rossum, Provable anonymity, in:
  Proc. {FMSE}, 2005, pp. 63--72.
\newblock \href {https://doi.org/10.1145/1103576.1103585}
  {\path{doi:10.1145/1103576.1103585}}.

\bibitem{Halpern:03:book}
J.~Y. Halpern, Reasoning about uncertainty, The MIT press, 2003.

\bibitem{Huber:08:book}
F.~Huber, C.~Schmidt-Petri, Degrees of belief, Vol. 342, Springer Science \&
  Business Media, 2008.

\bibitem{Bacchus:99:AI}
F.~Bacchus, J.~Y. Halpern, H.~J. Levesque, Reasoning about noisy sensors and
  effectors in the situation calculus, Artif. Intell. 111~(1-2) (1999)
  171--208.
\newblock \href {https://doi.org/10.1016/S0004-3702(99)00031-4}
  {\path{doi:10.1016/S0004-3702(99)00031-4}}.

\bibitem{Kawamoto:19:FC}
Y.~Kawamoto, Statistical epistemic logic, in: The Art of Modelling
  Computational Systems: {A} Journey from Logic and Concurrency to Security and
  Privacy, Vol. 11760 of LNCS, Springer, 2019, pp. 344--362.
\newblock \href {https://doi.org/10.1007/978-3-030-31175-9\_20}
  {\path{doi:10.1007/978-3-030-31175-9\_20}}.

\bibitem{Kawamoto:19:SEFM}
Y.~Kawamoto, \href{https://arxiv.org/pdf/1907.10327}{Towards logical
  specification of statistical machine learning}, in: Proc. {SEFM}, 2019, pp.
  293--311.
\newblock \href {https://doi.org/10.1007/978-3-030-30446-1\_16}
  {\path{doi:10.1007/978-3-030-30446-1\_16}}.
\newline\urlprefix\url{https://arxiv.org/pdf/1907.10327}

\bibitem{Kawamoto:20:SoSyM}
Y.~Kawamoto, An epistemic approach to the formal specification of statistical
  machine learning, Software and Systems Modeling 20~(2) (2020) 293--310.
\newblock \href {https://doi.org/10.1007/s10270-020-00825-2}
  {\path{doi:10.1007/s10270-020-00825-2}}.

\bibitem{Van:07:book-dynamic}
H.~Van~Ditmarsch, W.~van Der~Hoek, B.~Kooi, Dynamic epistemic logic, Vol. 337,
  Springer Science \& Business Media, 2007.

\bibitem{Zadeh:65:IC}
L.~Zadeh, Fuzzy sets, Information and Control 8~(3) (1965) 338--353.
\newblock \href {https://doi.org/https://doi.org/10.1016/S0019-9958(65)90241-X}
  {\path{doi:https://doi.org/10.1016/S0019-9958(65)90241-X}}.

\bibitem{Nguyen:18:book}
H.~T. Nguyen, C.~L. Walker, E.~A. Walker, A First Course in Fuzzy Logic, 4th
  Edition, Chapman \& Hall/CRC, 2018.

\bibitem{Eberhart:21:TAP}
C.~Eberhart, A.~Yamada, S.~Klikovits, S.~Katsumata, T.~Kobayashi, I.~Hasuo,
  F.~Ishikawa, Architecture-guided test resource allocation via logic, in:
  Proc. {TAP}'21, Vol. 12740 of LNCS, Springer, 2021, pp. 22--38.
\newblock \href {https://doi.org/10.1007/978-3-030-79379-1\_2}
  {\path{doi:10.1007/978-3-030-79379-1\_2}}.

\bibitem{Reiter:80:AIJ}
R.~Reiter, A logic for default reasoning, Artif. Intell. 13~(1-2) (1980)
  81--132.
\newblock \href {https://doi.org/10.1016/0004-3702(80)90014-4}
  {\path{doi:10.1016/0004-3702(80)90014-4}}.

\bibitem{Kyburg:02:NMR}
H.~E.~K. Jr., C.~Teng, Evaluating defaults, in: Proc. the 9th International
  Workshop on Non-Monotonic Reasoning ({NMR} 2002), 2002, pp. 257--264.

\bibitem{Kyburg:06:CI}
H.~E.~K. Jr., C.~Teng, Nonmonotonic logic and statistical inference, Comput.
  Intell. 22~(1) (2006) 26--51.
\newblock \href {https://doi.org/10.1111/j.1467-8640.2006.00272.x}
  {\path{doi:10.1111/j.1467-8640.2006.00272.x}}.

\bibitem{Kyburg:99:IJPRAI}
H.~E.~K. Jr., C.~Teng, Statistical inference as default reasoning, Int. J.
  Pattern Recognit. Artif. Intell. 13~(2) (1999) 267--283.
\newblock \href {https://doi.org/10.1142/S021800149900015X}
  {\path{doi:10.1142/S021800149900015X}}.

\bibitem{Fagin:95:PODC}
R.~Fagin, J.~Y. Halpern, Y.~Moses, M.~Y. Vardi, Knowledge-based programs, in:
  Proc. {PODC}'95, {ACM}, 1995, pp. 153--163.
\newblock \href {https://doi.org/10.1145/224964.224982}
  {\path{doi:10.1145/224964.224982}}.

\bibitem{Laverny:05:synthese}
N.~Laverny, J.~Lang, From knowledge-based programs to graded belief-based
  programs, part {I:} on-line reasoning\({}^{\mbox{*}}\), Synth. 147~(2) (2005)
  277--321.
\newblock \href {https://doi.org/10.1007/s11229-005-1350-1}
  {\path{doi:10.1007/s11229-005-1350-1}}.

\bibitem{Sardina:09:PROMAS}
S.~Sardi{\~{n}}a, Y.~Lesp{\'{e}}rance, Golog speaks the {BDI} language, in:
  Proc. {ProMAS}'09, Vol. 5919 of LNCS, Springer, 2009, pp. 82--99.
\newblock \href {https://doi.org/10.1007/978-3-642-14843-9\_6}
  {\path{doi:10.1007/978-3-642-14843-9\_6}}.

\bibitem{Levesque:97:JLP}
H.~J. Levesque, R.~Reiter, Y.~Lesp{\'{e}}rance, F.~Lin, R.~B. Scherl, {GOLOG:}
  {A} logic programming language for dynamic domains, J. Log. Program. 31~(1-3)
  (1997) 59--83.
\newblock \href {https://doi.org/10.1016/S0743-1066(96)00121-5}
  {\path{doi:10.1016/S0743-1066(96)00121-5}}.

\bibitem{Bratman:87:book}
M.~Bratman, Intention, plans, and practical reason (1987).

\bibitem{BelleL:15:IJCAI}
V.~Belle, H.~J. Levesque, {ALLEGRO:} belief-based programming in stochastic
  dynamical domains, in: Proc. {IJCAI} 2015, {AAAI} Press, 2015, pp.
  2762--2769.

\bibitem{Hogg:04:book:ims}
R.~V. Hogg, J.~W. McKean, A.~T. Craig, Introduction to Mathematical Statistics,
  Prentice Hall, 2004.

\bibitem{Kanji:06:book:100stat}
G.~K. Kanji, 100 statistical tests, Sage, 2006.

\bibitem{Bretz:10:book}
F.~Bretz, T.~Hothorn, P.~Westfall, Multiple Comparisons Using R, Chapman and
  Hall/CRC, 2010.
\newblock \href {https://doi.org/10.1201/9781420010909}
  {\path{doi:10.1201/9781420010909}}.

\bibitem{10.5555/1296072}
H.~R. Nielson, F.~Nielson, Semantics with Applications: An Appetizer
  (Undergraduate Topics in Computer Science), Springer-Verlag, 2007.

\bibitem{Cook:78:siamcomp}
S.~A. Cook, Soundness and completeness of an axiom system for program
  verification, {SIAM} J. Comput. 7~(1) (1978) 70--90.
\newblock \href {https://doi.org/10.1137/0207005} {\path{doi:10.1137/0207005}}.

\bibitem{Platzer:18:book}
A.~Platzer, Logical Foundations of Cyber-Physical Systems, Springer, 2018.
\newblock \href {https://doi.org/10.1007/978-3-319-63588-0}
  {\path{doi:10.1007/978-3-319-63588-0}}.

\bibitem{Hahnle:19:TCSGI}
R.~H{\"{a}}hnle, M.~Huisman, Deductive software verification: From
  pen-and-paper proofs to industrial tools, in: Computing and Software Science
  - State of the Art and Perspectives, Vol. 10000 of Lecture Notes in Computer
  Science, Springer, 2019, pp. 345--373.
\newblock \href {https://doi.org/10.1007/978-3-319-91908-9\_18}
  {\path{doi:10.1007/978-3-319-91908-9\_18}}.

\bibitem{Kripke:62:MLQ}
S.~A. Kripke, The undecidability of monadic modal quantification theory,
  Mathematical Logic Quarterly 8~(2) (1962) 113--116.
\newblock \href {https://doi.org/10.1002/malq.19620080204}
  {\path{doi:10.1002/malq.19620080204}}.

\bibitem{Hughes:96:book-modal-logic}
G.~E. Hughes, M.~J. Cresswell, A new introduction to modal logic, Psychology
  Press, 1996.

\bibitem{Sober:08:book}
E.~Sober, Evidence and evolution: The logic behind the science, Cambridge
  University Press, 2008.

\bibitem{Kawamoto:23:JELIA}
Y.~Kawamoto, T.~Sato, K.~Suenaga, Formalizing statistical causality via modal
  logic, in: Proc. {JELIA} 2023, Vol. 14281 of Lecture Notes in Computer
  Science, Springer, 2023, pp. 681--696.
\newblock \href {https://doi.org/10.1007/978-3-031-43619-2\_46}
  {\path{doi:10.1007/978-3-031-43619-2\_46}}.

\bibitem{Neyman:33:RS}
J.~Neyman, E.~S. Pearson, On the problem of the most efficient tests of
  statistical hypotheses, Philosophical Transactions of the Royal Society of
  London. Series A, Containing Papers of a Mathematical or Physical Character
  231 (1933) 289--337.

\end{thebibliography}

\appendix
\section{Instantiation to Concrete Testing Methods}
\label{sub:instantiation}

The derived rules for hypothesis tests in \Fig{fig:derived:rules:HT} are instantiated with concrete examples of tests given in standard textbooks on statistics (e.g.,~\cite{Hogg:04:book:ims}) as follows.

\begin{example}[Two-tailed $Z$-test]\label{eg:reasoning:Z-test}
We recall the two-tailed $Z$-test 
$\algp{\phi_0} = (\phi_0, \allowbreak \Test, \Normal(0, 1), \allowbreak \unlikelys{\sideT}\!, \Normal(\mu_{\ppl1},\sigma^2)\times \Normal(\mu_{\ppl2},\sigma^2))$
in Example \ref{eg:Z-test:likeliness}.
By applying the derived rule (\axTT{}), the procedure of this test with beliefs is expressed as the valid BHL judgment:
\begin{align}
 (\envI, \envO)  \vdash \,
  &\{
  \Possible (\mu_{\ppl1} < \mu_{\ppl2}) \land \Possible (\mu_{\ppl1} > \mu_{\ppl2}) \land \kappaX{\emptyset}
  \} \notag \\ 
  &\qquad \alpha := \Pr_{\tstat \sim \Normal(0,1)}\!\Big[\, |\tstat| \geq  \Big| {\textstyle\frac{\mean(y_1) - \mean(y_2)}{\sigma \sqrt{\nicefrac{1}{\size{y_1}} + \nicefrac{1}{\size{y_2}}}}} \Big| \,\Big] 
  \label{eq:example:ztest:teststatistic}\\
  &\{
  \KnowXx{\alpha}{y,\algp{\phi_0}} (\mu_{\ppl1} \neq \mu_{\ppl2}) \land \Possible (\mu_{\ppl1} \neq \mu_{\ppl2})
  \land \kappaX{y,\algp{\phi_0}}
  \} \notag.
\end{align}
Precisely, the $p$-value $\Pr_{\tstat \sim \Normal(0,1)}[\ldots]$ in \eqref{eq:example:ztest:teststatistic} is given by the procedure $f_{\!\alg_{\phi_0}}(y)$. 
\end{example}
We next show the instantiation to the classical likelihood ratio test with 
a simple null hypothesis $\xi=\xi_0$ and a simple alternative hypothesis $\xi=\xi_1$ (thus we suppose $\xi_0 \neq \xi_1$),
namely, in the setting of the Neyman-Pearson lemma~\cite{Neyman:33:RS}.
\begin{example}[Likelihood ratio test]\label{eg:reasoning:likelihood-test}
The goal of the likelihood ratio test is to determine which of two candidate distributions $D_p, D_q \in \Dists\reals$ is better to fit a dataset $y=(y_1, \ldots, y_{n})$ of sample size $n$.
The alternative hypothesis $\phiL \eqdef (\xi = \xi_1)$ (resp. the null hypothesis $\xi = \xi_0$) represents that the actual distribution is $D_p$ (resp. $D_q$).

To apply this test, we are expected to have the prior knowledge $\Know (\xi = \xi_0 \lor \xi = \xi_1)$ that $\xi$ is either $\xi_0$ or $\xi_1$.
Let $\phiU \eqdef (\xi \neq \xi_0 \land \xi \neq \xi_1)$.
Then the prior knowledge is denoted by $\Know \neg \phiU$, which is logically equivalent to $\neg \Possible \phiU$.

Formally, this test is denoted by $\algp{\phi_0} 
= (\phi_0, \allowbreak \Test, \Dthp, \allowbreak \unlikelys{\sideL}\!, P(\xi))$
such that: 
\begin{gather*}
\phi_0 \eqdef (\xi = \xi_0),\quad
\phiL \eqdef (\xi = \xi_1),\quad
\unlikelys{\mathsf{L}} \eqdef \{ (r, r')\in\reals\times\reals ~|~  r \leq r' \},\\
\Test(y)\eqdef  \frac{\prod_{i = 1}^{n} q(y_i)}{\prod_{i = 1}^{n} p(y_i)},\quad
P(\xi) \eqdef \begin{cases}
                D_q & (\xi = \xi_0)\\[-0.5ex]
                D_p & (\xi = \xi_1)
             \end{cases},
             \quad
\Dthp \eqdef \frac{\prod_{i = 1}^{n} q(D_q)}{\prod_{i = 1}^{n} p(D_q)}
\end{gather*}
where $p$ and $q$ are the density functions of $D_p$ and $D_q$, respectively. 
The probability distributions $p(D_q)$ and $q(D_q)$ are the push-forward measures of $D_q$ along $p$ and $q$ respectively.
The likelihood function $L$ is defined by 
$L(y|\xi_0) = {\textstyle \prod_{i = 1}^{n} q(y_i)}$ and $L(y|\xi_1) = {\textstyle \prod_{i = 1}^{n} p(y_i)}$, and the test statistic
$\Test(y)$ is called the \emph{likelihood ratio}.
 
In the likelihood ratio test, for a given $p$-value $\alpha$ and a threshold $k$ such that $\Pr_{d_1,\ldots,d_{n} \sim D_q}[\Test((d_1,\ldots,d_{n})) \leq k] \leq \alpha$, 
if we have $\Test(y) \leq k$,
the likelihood $L(y|\xi_0)$ is too small to accept the distribution $D_q$.
We then conclude that the other candidate $D_p$ is better to fit $y$ (thus this test is lower-tailed).
The $p$-value of this test is given by:
\begin{equation}\label{eq:sem:HT:instance:1}
    \Pr_{d_1,\ldots,d_{n} \sim D_q}[\Test((d_1,\ldots,d_{n})) \leq \Test(y)].
\end{equation}
By instantiating the $p$-value $\sem{f_{\algp{\phi_0}}(y)}$, we obtain (\ref{eq:sem:HT:instance:1}).
By applying the derived rule ($\axLT{}$), we obtain a valid BHL judgment corresponding the likelihood ratio test:
\begin{align*}
 (\envI, \envO)  \vdash \,
  &\{
  \Possible (\xi = \xi_1) \land \neg \Possible (\xi \neq \xi_0 \land \xi \neq \xi_1) \land \kappaX{\emptyset}
  \}\\ 
  &\qquad \alpha := \Pr_{d_1,\ldots,d_{n} \sim D_q}[\Test((d_1,\ldots,d_{n})) \leq \Test(y)]
  \\
  &\{
  \KnowXx{\alpha}{y,\algp{\phi_0}}
  \land \Possible (\xi = \xi_1)
  \land \kappaX{y,\algp{\phi_0}}
  \}.
\end{align*}
\end{example}

We can deal with Bayesian hypothesis tests in an analogous way.

\begin{example}[Bayesian hypothesis test]\label{eg:reasoning:Bayesian-test}
Consider the Bayesian likelihood ratio test with a dataset $y$ of sample size $n$,
prior distributions $D_{p'}, D_{q'} \in \Dists\reals$ 
with density functions $p'$ and $q'$, and 
posterior distributions $D_{p(z)}, D_{q(z)} \in \Dists\reals$
with density functions $p(-|z)$ and $q(-|z)$.
The goal of this test is to determine whether
the dataset $y$ is sampled from $D_{q(z)}$ where $z$ follows $D_{q'}$.
The alternative hypothesis $\xi = \xi_1$ (resp. the null hypothesis $\xi = \xi_0$) is that $y$ is sampled from $D_{q(z)}$ where $z$ follows $D_{q'}$ (resp. from $D_{p(z)}$ where $z$ follows $D_{p'}$).
As with \Eg{eg:reasoning:likelihood-test}, this test requires the prior knowledge $\Know (\xi = \xi_0 \lor \xi = \xi_1)$.

We first define the following statistical model with the parameter $\xi$.
\begin{align*}
(P(\xi))(S) \eqdef& \int_{\reals} P_1(\xi, z)(S)~dP_0(\xi)(z) \quad(S \subseteq \reals \colon \text{measurable})\\
&\text{ where }~~ P_0(\xi) =
        \begin{cases}
            D_{q'} & (\xi = \xi_0)\\
            D_{p'} & (\xi = \xi_1)
        \end{cases}, \quad
P_1(\xi, z) =
        \begin{cases}
            D_{q(z)} & (\xi = \xi_0)\\
            D_{p(z)} & (\xi = \xi_1)
        \end{cases}{.}
\end{align*}
In this definition, $P_0(\xi)$ and $P_1(\xi,z)$ are prior and posterior distributions, and 
$P(\xi)$ is the distribution of $y$ sampled from $P_1(\xi,z)$ where $z$ follows $P_0(\xi)$.

This hypothesis test can be denoted by $\algp{\phi_0} 
= (\phi_0, \allowbreak \Test, \Dthp, \allowbreak \unlikelys{\sideL}\!, P(\xi))$
where:
\begin{gather*}
\phi_0 \eqdef(\xi = \xi_0), \quad
\phiL \eqdef (\xi = \xi_1), \quad
\unlikelys{\mathsf{L}} \eqdef \{ (r, r')\in\reals\times\reals ~|~  r \leq r' \}\\
\Test(y) \eqdef \tfrac{\textstyle  \int q'(z)\prod_{i = 1}^{n} q(y_i|z) dz }{\textstyle \int p'(z)\prod_{i = 1}^{n} p(y_i|z) dz}, \quad
\Dthp \eqdef \tfrac{\textstyle  \int q'(z)\prod_{i = 1}^{n} q(D_{q(z)}|z) dz }{\textstyle \int p'(z)\prod_{i = 1}^{n} p(D_{q(z)}|z) dz}.
\end{gather*}
Unlike the (classical) likelihood ratio test, the test statistic $\Test(y)$ is the \emph{Bayes factor}, that is, the ratio of the marginal likelihoods
$L(y|\xi_0) = {\textstyle\int q'(z)\prod_{i = 1}^{n} q(y_i|z) dz}$ and
$L(y|\xi_1) = {\textstyle \int p'(z)\prod_{i = 1}^{n} p(y_i|z) dz}$.

As with the likelihood ratio test, we obtain a valid BHL judgment for the Bayesian hypothesis test by applying the derived rule ($\axLT{}$).
\end{example}

\section{Proofs for Technical Results}
\label{sec:proofs}

In \App{sub:proof:assertions}, we prove the propositions on statistical beliefs.
In \App{sub:proofs:semantics:basic}, we show basic results on structural operational semantics.
In \App{sub:appendix:parallel}, we show remarks on parallel compositions.
In \App{sub:proof:soundness}, we prove BHL's soundness.
In \App{sub:proof:complete}, we show BHL's relative completeness.
In \Tbl{tab:symbols:syn} and \ref{tab:symbols:sem}, we recall notations used in this paper.

\begin{table}[t]
\begin{minipage}[t]{.48\textwidth}
   \begin{center}
   \caption{Notations for syntax.}
   \label{tab:symbols:syn}
   \begin{small}
   \begin{tabular}{@{} cl @{}}
      \toprule
      \cmidrule(r){1-2}
      Symbols & Descriptions \\
      \midrule
      $y$		& Dataset \\
      $\alpha,\err$	& $p$-value \\
      $h_{y,\alg}$ & History variable on a test $\alg$ on $y$ \\
      $f$		& Function symbol \\
      $f_{\!\alg}$	& Procedure for a test $\alg$ \\
      $C$		& Program \\
      $\sampled{y}{x}{n}$ & $y$ is $n$ data sampled from $x$ \\
      $\followed{y}{x}$ & $y$ is sampled from $x$ \\
      $\phiT, \phiU, \phiL$	& Alternative hypotheses \\
      $\nuh \phiT, \nuh \phiU, \nuh \phiL$\hspace{-2ex}	& Null hypotheses \\
      $\KnowBowEt \phi$ & Statistical belief on $\phi$ \\
      \bottomrule
  \end{tabular}
  \end{small}
  \end{center}
\end{minipage}
\hfill
\begin{minipage}[t]{.48\textwidth}
   \begin{center}
   \caption{Notations for semantics.}
   \label{tab:symbols:sem}
   \begin{small}
   \begin{tabular}{@{} cl @{}}
      \toprule
      \cmidrule(r){1-2}
      Symbols \hspace{-2ex} & Descriptions \\
      \midrule
      $\calp(\cals)$	& All multisets over a set $\cals$ \\
      $\Dists\cals$	& All distributions over a set~$\cals$ \\
      $\M$	& Kripke model \\
      $w$		& Possible world \\
      $\sigmaw$	& Memory in a world $w$ \\
      $\testHist_{w}$	& Test history in a world $w$ \\
      $\calo$	& All data values \\
      $a$	& Action \\
      $\alg$	& Hypothesis test \\
      $\Test(y)$	& Test statistic of a dataset $y$ \\
      $\unlikelyst{\side}{\Test}$	& Likeliness relation \\
      \bottomrule
  \end{tabular}
  \end{small}
  \end{center}
\end{minipage}
\end{table}

\subsection{Proof for Properties of Statistical Beliefs}
\label{sub:proof:assertions}

We show the proof for Proposition~\ref{prop:PriorBelief} as follows.
\begin{proof}
\begin{enumerate}
\item The claim is clear from 
$\phiT \eqdef \phiU \lor \phiL$ and $\nuh \phiT \eqdef \neg \phiU \land \neg \phiL$.
\item 
The claim is shown as follows.
\begin{align*}
\models \Know (\phiU \lor \nuh \phiU)
\,\Leftrightarrow~ &
\models \Know (\phiU \lor (\neg \phiU \land \neg \phiL))
& \mbox{ (by $\nuh \phiU \eqdef \neg \phiU \land \neg \phiL$)}
\\ \,\Leftrightarrow~ &
\models \Know (\phiU \lor \neg \phiL)
\\ \,\Leftrightarrow~ &
\models \Know \neg \phiL
& \mbox{ (by $\models \phiU \rightarrow \neg \phiL$)}
\end{align*}
\item
The proof of this claim is analogous to that of the second claim.
\end{enumerate}%
\end{proof}

We show the proof for Proposition~\ref{prop:StatBelief} as follows.
\begin{proof}
\begin{enumerate}
\item {\rm (\propNu)} is straightforward from \eqref{eq:NegA} and \eqref{eq:sem:HT}.
\item We show {\rm (\propSBfour)} as follows.
Let $w$ be a world such that $w \models \KnowBowEt\phi$.
Then $w \models \Know (\phi \lor \TestBowE{y,\alg}(\err) \lor \neg \cmpds{y,\alg})$.
By the axiom (5) of $\Know$, we obtain 
$w \models \Know\Know(\phi \lor \TestBowE{y,\alg}(\err) \lor \neg \cmpds{y,\alg})$, 
hence $w \models \Know\KnowBowEt\phi$.
\item We show {\rm (\propSBfive)} as follows.
Let $w$ be a world such that $w \models \PossibleBowEt\phi$.
Then $w \models \Possible (\phi  \land  \neg\TestBowE{y,\alg}(\err) \land \cmpds{y,\alg})$.
By the axiom (4) of $\Know$, we obtain 
$w \models \Know\Possible (\phi  \land \neg \TestBowE{y,\alg}(\err) \land \cmpds{y,\alg})$, 
and thus $w \models \Know\PossibleBowEt\phi$.
\item We show {\rm (\propSBk)} as follows.
Let $w$ be a world such that $w \models \Know\phi$.
Since $\KnowBowEt \phi$ is defined by $\Know( \phi \lor \TestBowE{y,\alg}(\err) \lor \neg\cmpds{y,\alg})$,
we obtain $w \models \KnowBowEt\phi$.
Therefore $\models \Know\phi \rightarrow \KnowBowEt\phi$.
\item We show {\rm (\propSBdeg)} as follows.
Recall that $\KnowE$ is the abbreviation for $\KnowXx{=\err}{y,\alg}$.
Assume that $\err \le \err'$.
Then we obtain the following formulas on the strength of confidence levels:
$\models \TestA{y,\alg}(\err) \rightarrow \TestA{y,\alg}(\err')$.
By definition, we obtain
$\models \KnowE\phi \rightarrow \KnowXx{\err'}{y,\alg}\phi$
and $\models \PossibleXx{\err'}{y,\alg}\phi \rightarrow \PossibleE\phi$.
\item 
The claim {\rm (\propSBf)} is immediate from the definition of $\KnowE$.
\end{enumerate}
\end{proof}

We show the proof for Proposition~\ref{prop:StatBeliefHT} as follows.

\begin{proof}
\begin{enumerate}
\item We show the direction from left to right in {\rm (\propBHk)} as follows.
Let $w$ be a world such that $w \models \kappaX{S'}$.
By $(w, w) \in \calr$, we have $w \models \Possible \kappaX{S'}$.
Let $w'$ be a world such that $(w, w') \in \calr$.
Since the test history is observable, $\testHist_{w'} = \testHist_{w}$; hence $w' \models \kappaX{S'}$.
Thus, $w \models \Know\kappaX{S'}$.
The other direction can also be shown straightforwardly.

\item We show {\rm (\propBHT)} as follows.
Let $w$ be a world such that 
$w \models \kappaYA{y}$.
Let $w'$ be a world such that $w \relO w'$.
By definition, we have $\obs(w) = \obs(w')$, hence $\testHist_{w} = \testHist_{w'}$.
Then, by 
$w \models \kappaYA{y}$, we obtain 
$w' \models \kappaYA{y}$.
By {\rm (\propNu)},
we have $w' \models \NegA{y,\alg}(f_{\alg}(y))$.
Then by $w' \models \kappaYA{y}$, we obtain
$w' \models \phi \lor \TestA{y,\alg}(f_{\alg}(y)) \lor \neg \cmpds{y,\alg}$.
Therefore, we have $w \models \Know (\phi \lor \TestA{y,\alg}(f_{\alg}(y)) \lor \neg \cmpds{y,\alg})$, namely,  
$w \models \KnowXt{f_{\alg}(y)} \phi$.

\item We show {\rm (\propBHTor)} as follows.
Let $\alpha_1 = f_{\alg_1}(y_1)$, $\alpha_2 = f_{\alg_2}(y_2)$, and $\epsilon = \alpha_1 + \alpha_2$.
Let $w$ be a world such that $w \models \kappaS$.
Let $w'$ be a world such that $w \relO w'$.
By definition, we have $\obs(w) = \obs(w')$, hence $\testHist_{w} = \testHist_{w'}$.
Then, by $w \models \kappaS$, we obtain $w' \models \kappaS$.
Thus $\testHist_{w'} = \{ \mem_{w'}(y_1) \mapsto \{ \alg_1 \},\, \mem_{w'}(y_2) \mapsto \{ \alg_2 \} \}$.

Now we show $w' \models \KnowXx{\le\epsilon}{(y_1,y_2),\alg} (\phi_1\lor\phi_2)$ as follows.
For $i = 1, 2$, we denote by
$\alg_i = (\nuh \phi_i, \Test_i, \Dtix{\nuh \phi_i}, \unlikelyst{\side_i}{\Test_i}, \popl_{i})$
the hypothesis test with the null hypothesis $\nuh \phi_i$.
For each $i = 1, 2$, by the definition of $f_{\alg_i}$, we have the statistical belief that the alternative hypothesis $\phi_i$ is true with the significance level $\alpha_i$; i.e.,
\begin{align}\label{eq:proof:belief:alpha:dis}
\Pr_{\tstat \sim \Dtix{\nuh \phi_i}}[\, \tstat \unlikely_{\Test_i}^{(\side_i)} \Test_i(\sigmaw(y_i)) \,] = \alpha_i.
\end{align}

By $\obs(w) = \obs(w')$ and $y_1,y_2\in\varO$,
\begin{align}\label{eq:database:unchanged:dis}
\sigmawp(y_1) = \sigmaw(y_1) \mbox{ and } \sigmawp(y_2) = \sigmaw(y_2).
\end{align}

Recall that
the disjunctive combination is
$\alg = (\nuh\, (\phi_1{\lor} \phi_2), \Test, D,  \unlikelyst{\side_1, \side_2}{\Test}, \popl)$
where $\Test(y_1,y_2)=(\Test_1(y_1), \Test_2(y_2))$,
$D$ is a coupling of $\Dtax{\nuh \phi_1}$ and $\Dtbx{\nuh \phi_2}$, and
$(\tstat_1,\tstat_2) \unlikelyst{\side_1,\side_2}{\Test} (\tstat'_1,\tstat'_2)$ iff
either $\tstat_1 {\unlikelyst{\side_1}{\Test_1}} \tstat'_1$ or $\tstat_2 {\unlikelyst{\side_2}{\Test_2}} \tstat'_2$.
Then we obtain:
\begin{align*}
&~~~
\hspace{1ex}\Pr_{(\tstat_1,\tstat_2) \sim D}[\, (\tstat_1,\tstat_2) \unlikely_{\Test}^{(\side_1,\side_2)} \Test(\sigmawp(y_1),\sigmawp(y_2)) \,]
\\ &=
\Pr_{(\tstat_1,\tstat_2) \sim D}[\, \tstat_1 \unlikely_{\Test_1}^{(\side_1)} \Test_1(\sigmawp(y_1)) \lor \tstat_2 \unlikely_{\Test_2}^{(\side_2)} \Test_2(\sigmawp(y_2)) \,] 
\hspace{2ex}\text{(by definition)}
\\ &\le
\hspace{-2ex}\Pr_{\hspace{2ex}\tstat_1 \sim \Dtax{\nuh \phi_1}}\hspace{-2.5ex}[\, \tstat_1\unlikely_{\Test_1}^{(\side_1)} \Test_1(\sigmawp(y_1)) \,] +
\hspace{-2ex}\Pr_{\hspace{2ex}\tstat_2 \sim \Dtbx{\nuh \phi_2}}\hspace{-2.5ex}[\, \tstat_2\unlikely_{\Test_2}^{(\side_2)} \Test_2(\sigmawp(y_2)) \,] 
\\ &= \alpha_1 + \alpha_2
\hspace{35ex}
\text{(by Equations \eqref{eq:database:unchanged:dis}, \eqref{eq:proof:belief:alpha:dis})}
\\ &= \epsilon
{.}
\end{align*}
Recall that $y = (y_1, y_2)$.
Then $w' \models \NegLeE{y,\alg}(\epsilon)$.
Hence by $w' \models \kappaS$, we obtain
$w' \models \phi \lor \TestLeE{y,\alg}(\epsilon)$.
Therefore, $w \models \KnowXx{\le \epsilon}{y, \alg} (\phi_1\lor\phi_2)$.

\item We show {\rm (\propBHTand)} as follows.
Let $\alpha_1 = f_{\alg_1}(y_1)$, $\alpha_2 = f_{\alg_2}(y_2)$, and $\epsilon' \eqdef \min(f_{\alg_1}(y_1),\allowbreak f_{\alg_2}(y_2))$.
Let $w$ be a world such that $w \models \kappaS$.
Let $w'$ be a world such that $w \relO w'$.
By definition, we have $\obs(w) = \obs(w')$, hence $\testHist_{w} = \testHist_{w'}$.
Then, by $w \models \kappaS$, we obtain $w' \models \kappaS$.
Thus $\testHist_{w'} = \{ \mem_{w'}(y_1) \mapsto \{ \alg_1 \},\, \mem_{w'}(y_2) \mapsto \{ \alg_2 \} \}$.

Now we show $w' \models \KnowXx{\le\epsilon'}{(y_1,y_2), \alg} (\phi_1\land\phi_2)$ as follows.
For $i = 1, 2$, we denote by
$\alg_i = (\nuh \phi_i, \Test_i, \Dtix{\nuh \phi_i}, \unlikelyst{\side_i}{\Test_i}, \popl_{i})$
the hypothesis test with the null hypothesis $\nuh \phi_i$.
For each $i = 1, 2$, by the definition of $f_{\alg_i}$, we have the statistical belief that the alternative hypothesis $\phi_i$ is true with the significance level $\alpha_i$; i.e.,
\begin{align}\label{eq:proof:belief:alpha:con}
\Pr_{\tstat \sim \Dtix{\nuh \phi_i}}[\, \tstat \unlikely_{\Test_i}^{(\side_i)} \Test_i(\sigmaw(y_i)) \,] = \alpha_i.
\end{align}

By $\obs(w) = \obs(w')$ and $y_1,y_2\in\varO$,
\begin{align}\label{eq:database:unchanged:con}
\sigmawp(y_1) = \sigmaw(y_1) \mbox{ and } \sigmawp(y_2) = \sigmaw(y_2).
\end{align}

Recall that
the conjunctive combination is
$\alg = (\nuh\, (\phi_1{\land} \phi_2), \Test, D,  \unlikelyst{\side_1,\side_2}{\Test}, \popl)$
where $\Test(y_1,y_2)=(\Test_1(y_1), \Test_2(y_2))$,
$D$ is a coupling of $\Dtax{\nuh \phi_1}$ and $\Dtbx{\nuh \phi_2}$, and
$(\tstat_1,\tstat_2) \unlikelyst{\side_1,\side_2}{\Test} (\tstat'_1,\tstat'_2)$ iff
$\tstat_1 \unlikelyst{\side_1}{\Test_1} \tstat'_1$ and $\tstat_2 \unlikelyst{\side_2}{\Test_2} \tstat'_2$.
Then we obtain:
\begin{align*}
\alpha &\eqdef
\hspace{0ex}\Pr_{(\tstat_1,\tstat_2) \sim D}[\, (\tstat_1,\tstat_2) \unlikely_{\Test}^{(\side_1,\side_2)} \Test(\sigmawp(y_1),\sigmawp(y_2)) \,] 
\\ &=
\hspace{0ex}\Pr_{(\tstat_1,\tstat_2) \sim D}[\, \tstat_1 \unlikely_{\Test_1}^{(\side_1)} \Test_1(\sigmawp(y_1)) \land \tstat_2 \unlikely_{\Test_2}^{(\side_2)} \Test_2(\sigmawp(y_2)) \,] 
\hspace{1ex}
\text{(by definition)}
\\ &\le
\hspace{0ex}\Pr_{(\tstat_1,\tstat_2) \sim D}[\, \tstat_1 \unlikely_{\Test_1}^{(\side_1)} \Test_1(\sigmawp(y_1)) \,] 
\\ &= \alpha_1
\hspace{39.3ex}
\text{(by Equations \eqref{eq:database:unchanged:con},\eqref{eq:proof:belief:alpha:con})}
{.}
\end{align*}
A similar inequality holds for $\alpha_2$.
Thus 
$\alpha \le \min(\alpha_1, \alpha_2) = \epsilon'$.
Recall that $y = (y_1, y_2)$.
Then we have $w' \models \NegLeE{y,\alg}(\epsilon')$.
By $w' \models \kappaS$, we obtain
$w' \models \phi \lor \TestLeE{y,\alg}(\epsilon')$.
Therefore, $w' \models \KnowXx{\le \epsilon'}{y, \alg} (\phi_1\land\phi_2)$.\!%
\end{enumerate}
\end{proof}

\setcounter{equation}{5}

\subsection{Basics Results on Operational Semantics}
\label{sub:proofs:semantics:basic}
We recall basic results on structural operational semantics. 
We first show that executions of sequential compositions can be decomposed into 
ones of its components. 
\begin{lem}\label{lemma_seq_dec}
Suppose $\langle C_1 ; C_2,w \rangle \longrightarrow^k w'$.
There are $0< l < k$ and sequences $u_1, u_2$ such that $w' = w;u_1;u_2$,\,
$\langle C_1,w \rangle \longrightarrow^l w;u_1$, and 
$\langle C_2, w;u_1 \rangle \longrightarrow^{k - l}w;u_1;u_2$.
\end{lem}
\begin{proof}
We prove by induction on $k$.
If $k = 0, 1$, the statement is vacuously true.
If $k = k' + 2$ for $k' \ge 0$, we have one of the following two cases:
\begin{align}
\langle C_1 ; C_2,w \rangle &\longrightarrow \langle C'_1 ; C_2, w;u' \rangle  \longrightarrow^{k' + 1}w', \label{caseA_lemma_seq_dec}\tag{a} \\
\langle C_1 ; C_2,w \rangle &\longrightarrow \langle C_2, w;u' \rangle  \longrightarrow^{k' + 1}w' \label{caseB_lemma_seq_dec}\tag{b}.
\end{align}
In the case (\ref{caseA_lemma_seq_dec}), by induction hypothesis, 
there are $0< l' < k' + 1 $ and sequences $u'_1, u_2$ such that $w' = w;u';u'_1;u_2$,\,
$\langle C'_1,w;u' \rangle \longrightarrow^{l'} w;u';u'_1$, and 
$\langle C_2, w;u';u'_1 \rangle \longrightarrow^{k' + 1 - l'} w'$.
Thus, $\langle C_1, w \rangle \longrightarrow^{l' + 1} w;(u';u'_1)$ and 
$\langle C_2, w;(u';u'_1)\rangle \longrightarrow^{k - (l' + 1)} w'$.
In the case (\ref{caseB_lemma_seq_dec}), by the definition of execution, we have 
$\langle C_1, w \rangle \longrightarrow w;u'$ for some $u'$.
\end{proof}

We recall that the 
executions of single commands $\myskip$, $v \detassign e$ and $v \detassign f_{\!\alg}(y)$ are deterministic, 
hence the semantic relation $\sem{\ac}$ of each action $\ac$ is functional. 
These semantic \emph{functions} can be rewritten explicitly as follows:
\begin{small}
\begin{align*}
\sem{\myskip}(w) &= w; (\mem_w,\myskip,\testHist_w)\\
\sem{v \detassign e}(w) &= w; (\mem_w[v \mapsto \sem{e}_{\mem_w}],v \detassign e,\testHist_w)\\
\sem{v \detassign f_{\!\alg}(y)}(w) &= w; ((\mem_w[v \,{\mapsto} \sem{f_{\!\alg}(y)}_{\mem_w}])\zeta_{h_{y, \alg}},v \detassign f_{\!\alg}(y), \testHist_w \uplus \{ \mem_w(y) \mapsto \{\!\alg\} \})\\
&\qquad\text{ where }
 (m'\zeta_{h_{y, \alg}})(v) \eqdef 
\begin{cases}
m'(h_{y, \alg}) + 1 & v = h_{y, \alg}\\
m'(v) &\text{otherwise}
\end{cases}
\end{align*}
\end{small}
We remark here that the \emph{incrementation} $\zeta_{h_{y, \alg}}$ and the substitution $[v \,{\mapsto} \sem{f_{\!\alg}(y)}_{\mem_w}]$ are commutative:
$
\mem_w[v \,{\mapsto} \sem{f_{\!\alg}(y)}_{\mem_w}])\zeta_{h_{y, \alg}}
=
(\mem_w\zeta_{h_{y, \alg}})[v \,{\mapsto} \sem{f_{\!\alg}(y)}_{\mem_w}])
$.

If the memories of the current states of two worlds are identical, 
execution paths starting at these worlds can be simulated by each other, and can be written explicitly.

\begin{lem}\label{lemma_exec_depends_only_last_state}
Let $w_1,w_2$ be two possible worlds.
Suppose that
$\mem_{w_1}(v) = \mem_{w_2}(v)$ holds for all $v \in \var (C) \cap \varO$.
If $\langle C,w_1 \rangle \longrightarrow^k w'_1$ for some $w'_1$,
then there are $l \in \mathbb{N}$ with $0 \leq l \leq k$ and 
a sequence $\ac_1,\ac_2,\ldots, \ac_{l}$ of actions such that: 
\begin{enumerate}
\item 
\label{lemma_exec_depends_only_last_state:1}
$w'_1 = (\sem{\ac_{l}}\circ \cdots \circ \sem{\ac_1})(w_1)$ holds, and 
\item
\label{lemma_exec_depends_only_last_state:2}
$\langle C,w_2 \rangle \longrightarrow^k  w'_2$ and 
$\mem_{w'_1}(v) = \mem_{w'_2}(v)$ for all $v \in \var (C) \cap \varO$ hold where $w'_2 = (\sem{\ac_{l}}\circ \cdots \circ \sem{\ac_1})(w_2)$.
\end{enumerate}
\end{lem}
\begin{proof}
Suppose $\langle C,w_1 \rangle \longrightarrow^k w'_1$.
We prove by induction on $k$.
If $k=0$, the statement holds vacuously.
If $k = 1$, we have the following four cases:
\begin{itemize}
\item Case $C \equiv \myskip$.~
By the definition of $\sem{\myskip}$, we have $l = 1$, $\ac_1 = \myskip$ and $w'_1 = \sem{\myskip}(w_1)$. 
Let $w'_2 = \sem{\myskip}(w_2)$. 
Then we obtain $\langle \myskip,w_2 \rangle \longrightarrow^k w'_2$ and 
$\mem_{w'_1}(v) = \mem_{w_1}(v) = \mem_{w_2}(v) = \mem_{w'_2}(v)$ for all $v \in \var (C) \cap \varO$.
\item Case $C \equiv (v \detassign e)$.~
By the definition of $\sem{v \detassign e}$, we have $l = 1$, $\ac_1 = (v \detassign e)$ and $w'_1 = \sem{v \detassign e}(w_1)$. 
Let $w'_2 = \sem{v \detassign e}(w_2)$. 
Then we obtain $\langle v \detassign e,w_2 \rangle \longrightarrow^k w'_2$ and for all $v' \in \var (C) \cap \varO$,
\[
\mem_{w'_1}(v') = \mem_{w_1}[v \mapsto \sem{e}_{\mem_{w_1}}](v') = \mem_{w_2}[v \mapsto \sem{e}_{\mem_{w_2}}](v')  = \mem_{w'_2}(v').
\]
\item Case $C \equiv (v \detassign f_{\!\alg}(y))$.~
By the definition of $\sem{v \detassign f_{\!\alg}(y)}$, we have $l = 1$, $\ac_1 = (v \detassign f_{\!\alg}(y))$ and $w'_1 = \sem{v \detassign f_{\!\alg}(y)}(w_1)$. 
Let $w'_2 = \sem{v \detassign f_{\!\alg}(y)}(w_2)$. 
Then we obtain $\langle v \detassign e,w_2 \rangle \longrightarrow^k w'_2$ and for all $v' \in \var (C) \cap \varO$,
\[
\mem_{w'_1}(v') = \mem_{w_1}[v \mapsto \sem{f_{\!\alg}(y)}_{\mem_{w_1}}](v') 
= \mem_{w_2}[v \mapsto \sem{f_{\!\alg}(y)}_{\mem_{w_2}}](v') = \mem_{w'_2}(v').
\]

\item Case $C \equiv \myLoop{e}{C'}$ and $\sem{e}_{\mem_{w}} = \myfalse$.~
We have $l = 0$ and for all $v \in \var (C) \cap \varO$,
$\mem_{w'_1}(v) = \mem_{w_1}(v) = \mem_{w_2}(v) = \mem_{w'_2}(v)$. 
\end{itemize}

If $k = k' + 2$ for $k' \ge 0$, we have
$
\langle C,w_1 \rangle \longrightarrow \langle C',w''_1 \rangle 
\longrightarrow^{k' + 1} w'_1
$.
By induction hypothesis, 
there is $\ac_1 ,\ldots ,\ac_{l'}$ such that 
$w'_1 = (\sem{\ac_{l'}}\circ \cdots \circ \sem{\ac_1})(w''_1)$ for some $l' \leq k' + 1$, and 
if $\mem_{w''_1}(v) = \mem_{w''_2}(v) $ holds for all $v \in \var (C') \cap \varO$ then
$\langle C',w''_2 \rangle \longrightarrow^{k' + 1} w'_2$ and 
$\mem_{w'_1}(v) = \mem_{w'_2}(v) $ holds for all $v \in \var (C') \cap \varO$
where 
$w'_2 = (\sem{\ac_{l'}}\circ \cdots \circ \sem{\ac_1})(w''_2)$.

It suffices to show that 
for the first step $\langle C,w_1 \rangle \longrightarrow \langle C',w''_1 \rangle $ of execution,
\begin{enumerate}
\item at most a single action is performed, i.e., either $w''_1 = w_1$ or $w''_1 = \sem{\ac'}(w_1)$, and 
\item if $\mem_{w_1} (v) = \mem_{w_2} (v) $ for all $v \in \var (C) \cap \varO$, then 
$\langle C,w_2 \rangle \longrightarrow \langle C',w''_2 \rangle $ where 
$w''_2 = w_2$ if $w''_1 = w_1$, and $w''_2 =  \sem{\ac'}(w_2)$ if $w''_1 = \sem{\ac'}(w_1)$.
\end{enumerate}
We prove this by induction on the inference tree as follows.
Recall that by assumption, $\mem_{w_1}(v) = \mem_{w_2}(v) $ holds for all $v \in \var (C) \cap \varO$.
\begin{itemize}
\item 
If the inference is 
$\langle \myLoop{e}{C_1},w_1 \rangle \longrightarrow \langle C_1 ; \myLoop{e}{C_1},w''_1 \rangle$
where $\sem{e}_{\mem_{w_1}} = \mytrue$,
then $w''_1 = w_1$.
From $\sem{e}_{\mem_{w_2}} = \sem{e}_{\mem_{w_1}} = \mytrue$, we conclude 
$\langle \myLoop{e}{C_1},w_2 \rangle \longrightarrow \langle C_1 ; \myLoop{e}{C_1},w''_2 \rangle$ where $w''_2 = w_2$.

\item 
Similarly, if the inference is
$\langle \myIf{e}{C_1}{C_2},w_1 \rangle \longrightarrow \langle C_1,w''_1 \rangle$,
where $\sem{e}_{\mem_{w_1}} = \mytrue$,
then $w''_1 = w_1$ and
$\langle \myIf{e}{C_1}{C_2},w_2 \rangle \longrightarrow \langle C_1,w_2 \rangle$.

\item
If the last step of inference is derived by one of the following rules
{
\small
\[
\frac{
  \langle C_1, w \rangle \longrightarrow \langle C'_1, w'' \rangle
}{
  \langle C_1;C_2, w \rangle \longrightarrow \langle C'_1;C_2, w'' \rangle
},
~~~~
\frac{
  \langle C_1, w \rangle \longrightarrow \langle C'_1, w'' \rangle
}{
  \langle C_1\| C_2, w \rangle \longrightarrow \langle C'_1\|C_2, w'' \rangle
}
\]
}%
then we apply the induction hypothesis to $\langle C_1, w_1 \rangle \longrightarrow \langle C'_1, w''_1 \rangle$.
We have either $w''_1 = w_1$ or $w''_1 = \sem{\ac'}(w_1)$, and 
$\langle C_1, w_2 \rangle \longrightarrow \langle C'_1, w''_2 \rangle$ where 
$w''_2 = w_2$ if $w''_1 = w_1$, and 
$w''_2 =  \sem{\ac'}(w_2)$ if $w''_1 = \sem{\ac'}(w_1)$.
Then by the same rule, we conclude 
$\langle C, w_2 \rangle \longrightarrow \langle C', w''_2 \rangle$.

\item The other cases are shown in a similar way.
\end{itemize}
\end{proof}

\subsection{Remarks on Parallel Compositions}
\label{sub:appendix:parallel}

We present some remarks on parallel compositions.
We first show that in general, parallel compositions contain sequential compositions.

\begin{lem}\label{lemma_parallel_includes_sequential}
For any possible world $w$,
we have $\sem{C_1 ; C_2}(w) \subseteq \sem{C_1 \| C_2}(w)$.
\end{lem}
\begin{proof}
Suppose $\langle C_1 ; C_2,w \rangle \longrightarrow^\ast w'$ 
Thanks to \Lem{lemma_seq_dec}, 
there are $l, k > 0$ and $w''\in\calw$ such that
$\langle C_1,w \rangle \longrightarrow^{l} w''$ and 
$\langle C_2, w'' \rangle \longrightarrow^{k} w'$.
We show 
$\langle C_1 \| C_2,w \rangle \longrightarrow^{l+k} w'$ by induction on $l$.
If $l = 0$, the statement holds vacuously.
If $l = l' + 1$,  
the inference $\langle C_1,w \rangle \longrightarrow^l w''$ can be decomposed into
$\langle C_1,w \rangle \longrightarrow \langle C'_1, w_0 \rangle \longrightarrow^{l'} w''$ for some $w_0\in\calw$.
Hence, 
$
\langle C'_1;C_2, w_0 \rangle \longrightarrow^{l'}  \langle C_2,w''\rangle \longrightarrow^{k}w'
$.
By induction hypothesis, we obtain 
$
\langle C_1 \| C_2,w \rangle \longrightarrow
\langle C'_1 \| C_2,w_0 \rangle \longrightarrow^{l' + k} w'$.
This completes the proof.
\end{proof}

Next, we show that for a world $w$ and a parallel composition $C_1 \| C_2$, 
a world $w' \in \sem{C_1 \| C_2}(w)$ is convertible to a pair of 
$w_1 \in \sem{C_1 }(w)$ and $w_2 \in \sem{C_2}(w)$ and vice versa.
Recall that we imposed the restriction $\upd{C_b} \cap \var(C_{3-b}) = \emptyset$ for $b = 1,2$.

Let us consider $\langle C_b,w \rangle \longrightarrow^\ast w'_b$ for $b = 1,2$.
By \Lem{lemma_exec_depends_only_last_state}, we obtain
$w'_b = (\sem{\ac^b_{l_b}} \circ \cdots \circ \sem{\ac^b_{1}} )(w)$ for $b = 1,2$,  $l_b \ge 0$, and a sequence $\ac^b_1,\ldots,\ac^b_{l_b}$ of actions.
Then, we can define the following possible world:
\[
w' = (\sem{\ac^2_{l_2}} \circ \cdots \circ \sem{\ac^2_{1}} \circ \sem{\ac^1_{l_1}} \circ \cdots \circ \sem{\ac^1_{1}} )(w).
\]

We first show that this can be an execution of $C_1\|C_2$ starting at the world $w$.
\begin{lem}
\label{sub:appendix:parallel:A}
If $\langle C_b,w \rangle \longrightarrow^\ast (\sem{\ac^b_{l_b}} \circ \cdots \circ \sem{\ac^b_{1}} )(w)$ holds for each $b =1,2$, then
we have $\langle C_1\|C_2,w\rangle \longrightarrow^\ast (\sem{\ac^2_{l_2}} \circ \cdots \circ \sem{\ac^2_{1}} \circ \sem{\ac^1_{l_1}} \circ \cdots \circ \sem{\ac^1_{1}} )(w)$.
\end{lem}

\begin{proof}
Let $w_3 = (\sem{\ac^1_{l_1}} \circ \cdots \circ \sem{\ac^1_{1}})(w)$.
By the assumptions of this lemma, we have
$\langle C_1\|C_2,w\rangle \longrightarrow^\ast \langle C_2,w_3 \rangle$.
Since $\upd {C_1} \cap \var (C_2) = \emptyset$,
$\mem_{w_3}(v) = \mem_w (v)$ holds for all $v \in \var (C_2) \cap \varO$.
By \Lem{lemma_exec_depends_only_last_state},
we conclude:
\[
\langle C_1\|C_2,w\rangle \longrightarrow^\ast
\langle C_2,w_3 \rangle
\longrightarrow^\ast (\sem{\ac^2_{l_2}} \circ \cdots \circ \sem{\ac^2_{1}} \circ \sem{\ac^1_{l_1}} \circ \cdots \circ \sem{\ac^1_{1}} )(w).%
\]
\end{proof}

Second, we show the converse of the above lemma.
Let $w'$ be a world such that
$\langle C_1\|C_2,w\rangle {\longrightarrow^\ast} w'$.
By \Lem{lemma_exec_depends_only_last_state}, 
there is a sequence $\ac_1,\ldots,\ac_{n'}$ of actions such that:
\[
w' = (\sem{\ac_{n'}} \circ \cdots \circ \sem{\ac_1} )(w).
\]
Then we can decompose it into executions of $C_1$ and $C_2$ in the following sense.
\begin{lem}
\label{sub:appendix:parallel:B}
The sequence $\ac_1,\ldots,\ac_{n'}$ of actions can be decomposed into two subsequences
$\ac_{L^1_1},\ldots,\ac_{L^1_{n'_1}}$ and 
$\ac_{L^2_1},\ldots,\ac_{L^2_{n'_2}}$ such that
for each $b=1,2$,
$\langle C_b ,w\rangle \longrightarrow^\ast w'_b$ and
$w'_b = (\sem{\ac_{L^b_{n'_b}}}\circ \cdots \circ \sem{\ac_{L^b_1}}) (w)$.
\end{lem}

\begin{proof}
By assumption, there is a $k \ge 0$ such that
$\langle C_1\|C_2,w\rangle \longrightarrow^k w'$.
We prove this lemma by induction on $k$.
If $k = 0,1$, the statement holds vacuously.
Suppose $k = k' + 2$ for $k' \ge 0$. 
We decompose that execution into $\langle C_1\|C_2,w\rangle \longrightarrow \gamma \longrightarrow^{k'+1} w'$.
\begin{itemize}
\item Case $\gamma \equiv \langle C'_1\|C_2, w'' \rangle$ for some $C'_1$ and $w''$.~
By definition, we should have
$\langle C_1,w\rangle \longrightarrow \langle C'_1,w''\rangle$.
Then we have the following two cases.
\begin{itemize}
\item Case $w'' = w$.~
By induction hypothesis, we obtain
$\langle C'_1,w \rangle \longrightarrow^\ast w'_1$ and 
$\langle C_2,w \rangle \longrightarrow^\ast w'_2$.
Then, we also have $\langle C_1,w \rangle \longrightarrow \langle C'_1,w \rangle \longrightarrow^\ast w'_1$.

\item Case $w'' = \sem{\ac_1}(w)$.~
We have $L^1_1 = 1$.  
By applying the induction hypothesis to 
$\langle C'_1\| C_2,\sem{\ac_1}(w) \rangle \longrightarrow^{k'+1} w' $,
there are two subsequences $\ac_{L^1_2},\ldots,\ac_{L^1_{n'_1}}$ and $\ac_{L^2_1},\ldots,\ac_{L^2_{n'_1}}$
of $\ac_2, \ldots \ac_{n'}$ such that:
\begin{align*}
\langle C'_1,\sem{\ac_1}(w) \rangle &\longrightarrow^\ast (\sem{\ac_{L^1_{n'_1}}}\circ \cdots\circ \sem{\ac_{L^1_{1}}})(\sem{\ac_1}(w)) = w'_1,\\[-0.5ex]
\langle C_2,\sem{\ac_1}(w)  \rangle &\longrightarrow^\ast (\sem{\ac_{L^2_{n'_2}}}\circ \cdots\circ \sem{\ac_{L^2_{1}}})(\sem{\ac_1}(w)).
\end{align*}
Since the action $\ac_1$ is performed in the program $C_1$ and $\upd{C_1} \cap \var (C_2) = \emptyset$,
we have $\mem_{\sem{\ac_1}(w)}(v) = \mem_w (v)$ for all $v \in \var (C_2) \cap \varO$.
Thus, by \Lem{lemma_exec_depends_only_last_state},
we conclude:
\[
\langle C_2,w \rangle \longrightarrow^\ast (\sem{\ac_{L^2_{n'_2}}}\circ \cdots\circ \sem{\ac_{L^2_{1}}})(w) = w'_2.
\]
\end{itemize}
\item Case $\gamma \equiv \langle C_2, w'' \rangle$ for some $w''$.~
By definition, we should have
$\langle C_1,w\rangle \longrightarrow w''$ and $\langle C_2,w''\rangle \longrightarrow^\ast  w'$.
Then $w'' = w'_1$ and 
$\mem_{w}(v) = \mem_{w'_1}(v)$ holds for all $v \in \var (C_2) \cap \varO$.
We have the following two cases.
\begin{itemize}
\item Case $w'_1 = w$.~
We immediately obtain
$\langle C_1,w\rangle \longrightarrow w$ and $\langle C_2,w\rangle \longrightarrow^\ast  w' = (\sem{\ac_{n'}}\circ\cdots\circ\sem{\ac_1})(w)$.

\item Case $w_1 = \sem{\ac_1}(w)$.~
We have $L^1_1 = 1$,
$\langle C_1,w\rangle \longrightarrow \sem{\ac_1}(w)$, and $\langle C_2,\sem{\ac_1}(w) \rangle \longrightarrow^\ast  w'$.
Since $\ac_1$ belongs to executions in $C_1$ and $\upd{C_1} \cap \var (C_2) = \emptyset$,
we have $\mem_{\sem{\ac_1}(w)}(v) = \mem_w (v)$ for all $v \in \var (C_2) \cap \varO$.
Thus by \Lem{lemma_exec_depends_only_last_state}, we obtain
$\langle C_2, w \rangle \longrightarrow^\ast   (\sem{\ac_{n'}}\circ \cdots \circ \sem{\ac_{2}}) (w)$.
\end{itemize}
\item The other cases are proved in a similar way.
\end{itemize}
\end{proof}

Executions of programs can be nondeterministic due to parallel compositions.
However, since two programs in parallel do not interfere with each other,
their executions result in the same memory and test history as follows.

\begin{lem}\label{lemma_uniqueness_final_result}
For any $w' \in \sem{C_1 \| C_2}(w)$, there is a $w^\ast \in \sem{C_1 ; C_2}(w)$ such that  
$\mem_{w'} = \mem_{w^\ast}$ and $\testHist_{w'} = \testHist_{w^\ast}$.
\end{lem}
\begin{proof}
Let $w' \in \sem{C_1 \| C_2}(w)$.
Then $\langle C_1\|C_2,w\rangle \longrightarrow^\ast w'$. 
By \Lem{lemma_exec_depends_only_last_state}, 
there is a sequence $\ac_1,\ldots,\ac_{n'}$ of actions such that
$w' = (\sem{\ac_{n'}} \circ \cdots \circ \sem{\ac_1} )(w)$.
By \Lem{sub:appendix:parallel:B}, the sequence $\ac_1,\ldots,\ac_{n'}$ can be decomposed into two subsequences
$\ac_{L^1_1},\ldots,\ac_{L^1_{n'_1}}$ and $\ac_{L^2_1},\ldots,\ac_{L^2_{n'_2}}$ such that:
\[
\langle C_1 ,w\rangle \longrightarrow^\ast (\sem{\ac_{L^1_{n'_1}}}\circ \cdots \circ \sem{\ac_{L^1_1}}) (w), 
\quad 
\langle C_2 ,w\rangle \longrightarrow^\ast (\sem{\ac_{L^2_{n'_2}}}\circ \cdots \circ \sem{\ac_{L^2_1}}) (w).
\]
Now we define:
\[
w^\ast = (\sem{\ac_{L^2_{n'_2}}}\circ \cdots \circ \sem{\ac_{L^2_1}} 
\circ \sem{\ac_{L^1_{n'_1}}}\circ \cdots \circ \sem{\ac_{L^1_1}})(w).
\]
Then $w^\ast \in \sem{C_1 \| C_2} (w)$.
By \Lem{sub:appendix:parallel:A}, we obtain $\langle C_1\|C_2,w\rangle \longrightarrow^\ast w^\ast$.
We now show $\mem_{w'}(v) = \mem_{w^\ast}(v)$ for all $v \in \var$ as follows.
If $v \in \var(C_1)$ then 
no substitution in $C_2$ updates the value of $v$, since $\upd{C_2} \cap \var(C_{1}) = \emptyset$. 
Hence,
$\mem_{w'}(v) = \mem_{(\sem{\ac_{L^1_{n'_1}}}\circ \cdots \circ \sem{\ac_{L^1_1}}) (w))} (v) = \mem_{w^\ast}(v)$.
Symmetrically, if $v \in \var(C_2)$ then
$\mem_{w'}(v) = \mem_{(\sem{\ac_{L^2_{n'_2}}}\circ \cdots \circ \sem{\ac_{L^2_1}}) (w)} (v) = \mem_{w^\ast}(v)$.
If $v$ is a history variable, then the value of $v$ does not depend on the order of actions, because every update increases $v$ by $1$.
Thus, we conclude 
$\mem_{w'}(v) = \mem_{w^\ast}(v)$.
For the other case, the program $C_1 \| C_2$ does not change the value of $v$, hence
$\mem_{w'}(v) = \mem_{w^\ast}(v) = \mem_{w}(v)$.

Finally, since the test histories are multisets, we have $\testHist_{w'} = \testHist_{w^\ast}$.
\end{proof}

\begin{lem}
\label{lem:exchange:par:seq}
For any possible world $w\in\calw$, any formula $\phi\in\Fml$, and any interpretation function $\cali: \IntVar \rightarrow \ints^*$, we have:
\[
\sem{C_1 ; C_2}(w) \modelsi \phi
 ~\mathrel{\mathrm{ iff }}~ \sem{C_1 \| C_2}(w) \modelsi \phi.
\]
\end{lem}
\begin{proof}
By \Lem{lemma_uniqueness_final_result}, 
for any $w' \in  \sem{C_1 \| C_2}(w)$,
there is a function $\qfun{w}: \calw \rightarrow \calw$ such that 
$\qfun{w}(w') \in \sem{C_1 ; C_2}(w)$, $\mem_{\qfun{w}(w')} = \mem_{w'}$, and $\testHist_{\qfun{w}(w')} = \testHist_{w'}$.

By \Lem{lemma_parallel_includes_sequential},
it is sufficient to show the following statement:
\[
\mbox{for all } w' \in  \sem{C_1 \| C_2}(w),~~
w' \modelsi \phi \mathrel{\mbox{ iff }} \qfun{w}(w') \modelsi \phi.
\]

We prove this by induction on the construction of the world $w$ and the formula $\phi$.
\begin{itemize}
\item Case $\phi \equiv \eta(u_1, \ldots, u_k)$.~
Since $\mem_{w'} = \mem_{\qfun{w}(w')}$ and $V_{w'}(\eta) = V_{\qfun{w}(w')}(\eta)$, 
\begin{align*}
w' \modelsi \phi
& ~\mbox{ iff }~
(\sem{u_1}_{w'}, \ldots, \sem{u_k}_{w'}) \in V_{w'}(\eta)\\
& ~\mbox{ iff }~
(\sem{u_1}_{\qfun{w}(w')}, \ldots, \sem{u_k}_{\qfun{w}(w')}) \in V_{\qfun{w}(w')}(\eta)\\
& ~\mbox{ iff }~
\qfun{w}(w') \modelsi \phi.
\end{align*}
\item Case $\phi \equiv \Know \phi_1$.~
We prove the direction from left to right, that is,
\[
w' \modelsi \Know \phi_1
\mathrel{\mbox{ implies }}
\qfun{w}(w') \modelsi \Know \phi_1.
\]
We first recall the interpretation of $\Know$.
\begin{align*}
w' \modelsi\Know \phi_1
& ~\mbox{ iff }~
\mbox{for all } w'_1 \in\calw,~ (w', w'_1) \in \relO ~~\mbox{implies}~~ w'_1 \modelsi \phi_1.
\end{align*}
Assume that $w' \modelsi\Know \phi_1$.
Let $w'_1$ be a world such that $(w', w'_1) \in \relO$.
Then $w'_1 \modelsi \phi_1$.
Let $w_1$ be a world such that $(w, w_1) \in \relO$.

Since $w' \in \sem{C_1 \| C_2}(w)$, 
by \Lem{lemma_exec_depends_only_last_state}, there is a sequence $\ac_{1},\ldots,\ac_{n'}$ of actions such that $w' = (\sem{\ac_{n'}}\circ \cdots \circ \sem{\ac_{1}})(w)$.
Let $\ac_{L^1_1},\ldots,\ac_{L^1_{n'_1}}$ and $\ac_{L^2_1},\ldots,\ac_{L^2_{n'_2}}$ be the sequences of actions in $C_1$ and in $C_2$, respectively.
By constructions in Lemma \ref{lemma_uniqueness_final_result} (using Lemma \ref{sub:appendix:parallel:B}),
we then obtain:
\[
\qfun{w}(w') = (\sem{\ac_{L^2_{n'_2}}}\circ \cdots \circ \sem{\ac_{L^2_{1}}} \circ \sem{\ac_{L^1_{n'_1}}}\circ \cdots \circ \sem{\ac_{L^1_1}})(w),
\]
where 
the sequence $\ac_{1},\ldots,\ac_{n'}$ is decomposed into the subsequences 
$\ac_{L^1_1},\ldots,\ac_{L^1_{n'_1}}$ and 
$\ac_{L^2_1},\ldots,\ac_{L^2_{n'_2}}$.

Since $C_1 \| C_2$ reads only observable variables, and $(w', w'_1) \in \relO$ and $(w, w_1) \in \relO$, 
we obtain $w_1' = (\sem{\ac_{n'}}\circ \cdots \circ \sem{\ac_{1}})(w_1)$ and $w'_1 \in \sem{C_1 \| C_2}(w_1)$.
By induction hypothesis, we have $\qfun{w_1}(w'_1) \modelsi \phi_1$.

Let $w'_2$ be a world such that $(\qfun{w}(w') , w'_2) \in \relO$.
Then we have $w'_2 = (\sem{\ac_{L^2_{n'_2}}}\circ \cdots \circ \sem{\ac_{L^2_{1}}} \circ \sem{\ac_{L^1_{n'_1}}}\circ \cdots \circ \sem{\ac_{L^1_1}})(w_1)$.
Since $w_1' = (\sem{\ac_{n'}}\circ \cdots \circ \sem{\ac_{1}})(w_1)$,
we obtain $\qfun{w_1}(w'_1) = w'_2$.
Hence, we obtain $w'_2 \modelsi \phi_1$ from $\qfun{w_1}(w'_1) \modelsi \phi_1$.

Since $w'_2$ is an arbitrary possible world such that $(\qfun{w}(w') , w'_2) \in \relO$, 
we conclude $\qfun{w}(w') \modelsi \Know \phi_1$.

The direction from right to left can be proved straightforwardly by \Lem{lemma_parallel_includes_sequential}.

\item Cases $\phi \equiv \neg \phi_1$, $\phi \equiv \phi_1 \land \phi_2$ and $\phi \equiv \forall i. \phi_1$.~
The statement is proved immediately by induction hypothesis.
\end{itemize}
\end{proof}

\subsection{Proof for BHL's Soundness}
\label{sub:proof:soundness}

To prove BHL's soundness and relative completeness,
we show the following lemma.
\begin{restatable}{lem}{LemExpressAssign} \label{lem:LemExpressAssign}
Let $\psi\in\Fml$, and $\cali$ be any interpretation function over $\IntVar$.
Then:
\begin{align}
\sem{\myskip}(w)  \modelsi \psi 
&~~\mbox{ iff }~~
w \modelsi \psi
\label{lem:LemExpressAssign:skip}
\\
\sem{v \detassign e}(w)  \modelsi \psi 
&~~\mbox{ iff }~~
w \modelsi \psi[\nicefrac{e}{v}]
\label{lem:LemExpressAssign:asgn} 
\\
\sem{v \detassign f_{\alg}(y)}(w) \modelsi \psi
&~~\mbox{ iff }~~
 w \modelsi \psi[\nicefrac{f_{\alg}(y)}{v}, \nicefrac{h_{y,\alg}+1}{h_{y,\alg}}].
\label{lem:LemExpressAssign:test} 
\end{align}
\end{restatable}

\begin{proof}[Proof of \eqref{lem:LemExpressAssign:skip} in \Lem{lem:LemExpressAssign}]
Let $w' \eqdef \sem{\myskip}(w) = w;(\mem_w,\myskip,\testHist_w)$.
We prove the statement by induction on $\psi$ as follows.
\begin{itemize}
\item Case $\psi \equiv \eta(u_1, \ldots, u_k)$.~
Since $V_{w}(\eta) = V_{w'}(\eta)$ and $\mem_{w} = \mem_{w'}$, we obtain:
\begin{align*}
\sem{\myskip}(w) \modelsi \eta(u_1, \ldots, u_k)
&~~\mbox{ iff }~~ w' \modelsi \eta(u_1, \ldots, u_k)\\
&~~\mbox{ iff }~~ (\sem{u_1}_{w'}^{\cali}, \ldots, \sem{u_k}_{w'}^{\cali}) \in V_{w'}(\eta)\\
&~~\mbox{ iff }~~ (\sem{u_1}_{w}^{\cali}, \ldots, \sem{u_k}_{w}^{\cali}) \in V_{w}(\eta)\\
&~~\mbox{ iff }~~  w \modelsi \eta(u_1, \ldots, u_k).
\end{align*}
\item Case $\psi \equiv \Know \phi$.~
We have:
\begin{align}
\sem{\myskip}(w) \modelsi \Know \phi 
&~\mbox{ iff }~ w' \modelsi \Know \phi  \nonumber \\
&~\mbox{ iff }~
\mbox{for all $w_1 \in \calw$,~ $(w', w_1) \in \relO$}
~\mbox{implies}~ w_1 \modelsi \phi \nonumber \\
&~\mbox{ iff }~
\mbox{for all $w_2 \in \calw$,~ $(w, w_2) \in \relO$}
~\mbox{implies}~ w_2 \,{\modelsi} \phi.\!%
& \mbox{($\dag$)} \nonumber
\end{align}
The last equivalence 
($\dag$)
is proved as follows.
To show the direction from left to right,
let $(w, w_2) \in \relO$.
We take $w_1 = \sem{\myskip}(w_2) = w_2;(\mem_{w_2},\myskip,\testHist_{w_2})$.
We then obtain $(w', w_1) \in \relO$, hence $w_1 \modelsi \phi$.
By induction hypothesis, we conclude $w_2 \modelsi \phi$.
To show the other direction,
let $(w', w_1) \in \relO$.
We take $w_2 = w_1[0];\cdots;w_1[\len{w_1}-2]$.
We then obtain $(w,w_2) \in \relO$ and $\sem{\myskip}(w_2) =  w_1$.
Hence $w_2 \modelsi \phi$.
By induction hypothesis, we conclude $w_1 \modelsi \phi$.

\item Cases $\psi \equiv \neg \phi$, $\psi \equiv \phi_1 \land \phi_2$, $\psi \equiv \forall i. \phi$.~
Immediate by induction hypothesis.
\end{itemize}
\end{proof}

\begin{proof}[Proof of \eqref{lem:LemExpressAssign:asgn} in \Lem{lem:LemExpressAssign}]
Let $w' \eqdef \sem{v \detassign e}(w) = w;(\mem_w[v \mapsto \sem{e}_w],v \detassign e,\testHist_w)$.
We prove the statement by induction on $\psi$ as follows.
\begin{itemize}
\item Case $\psi \equiv \eta(u_1, \ldots, u_k)$.~
For any term $u$, we have $\sem{u}_{w'}^{\cali} = \sem{u_1[\nicefrac{e}{v}]}_{w}^{\cali}$.
Since $V_{w}(\eta) = V_{w'}(\eta)$ and $\mem_{w'} = \mem_w[v \mapsto \sem{e}_w]$, we have:
\begin{align*}
\sem{v \detassign e}(w) \modelsi \eta(u_1, \ldots, u_k)
&~~\mbox{ iff }~~ w' \modelsi \eta(u_1, \ldots, u_k)\\
&~~\mbox{ iff }~~ (\sem{u_1}_{w'}^{\cali}, \ldots, \sem{u_k}_{w'}^{\cali}) \in V_{w'}(\eta)\\
&~~\mbox{ iff }~~ (\sem{u_1[\nicefrac{e}{v}]}_{w}^{\cali}, \ldots, \sem{u_k[\nicefrac{w}{v}]}_{w}^{\cali}) \in V_{w}(\eta)\\
&~~\mbox{ iff }~~  w \modelsi \eta(u_1, \ldots, u_k)[\nicefrac{e}{v}].
\end{align*}

\item Case $\psi \equiv \Know \phi$.~
We have:
\begin{align}
\sem{v \detassign e}(w) \,{\modelsi}\, \Know \phi 
&\mbox{ iff }~ w' \modelsi \Know \phi \nonumber \\
&\mbox{ iff }~
\mbox{for all $w_1 \,{\in}\, \calw$, $(w', w_1) \,{\in}\, \relO$}
~\mbox{implies}~ w_1 \modelsi \phi \nonumber \\
&\mbox{ iff }~
\mbox{for all $w_2 \,{\in}\, \calw$, $(w, w_2) \,{\in}\, \relO$}
~\mbox{implies}~ w_2 \modelsi \phi[\nicefrac{e}{v}].
~\mbox{ ($\dag$)} \nonumber
\end{align}
The last equivalence 
($\dag$) 
is derived
as with the proof of \eqref{lem:LemExpressAssign:skip}.
To show the direction from left to right, 
let $(w, w_2) \in \relO$
and $w_1 = \sem{v \detassign e}(w_2) = w_2;(\mem_{w_2}[v \mapsto \sem{e}_{w_2}],v \detassign e,\testHist_{w_2})$.
Then $(w', w_1) \in \relO$.
By induction hypothesis, we conclude $w_2 \modelsi \phi[\nicefrac{e}{v}]$.
To show the other direction, let $(w', w_1) \in \relO$
and $w_2 = w_1[0];\cdots;w_1[\len{w_1}-2]$.
Then $\sem{v \detassign e}(w_2) =  w_1$ and $(w, w_2) \in \relO$.
By induction hypothesis, we conclude $w_1 \modelsi \phi$.

\item Cases $\psi \equiv \neg \phi$, $\psi \equiv \phi_1 \land \phi_2$, $\psi \equiv \forall i. \phi$.~ 
Immediate by induction hypothesis.
\end{itemize}
\end{proof}

\begin{proof}[Proof of \eqref{lem:LemExpressAssign:test} in \Lem{lem:LemExpressAssign}]
We prove \eqref{lem:LemExpressAssign:test} by induction on $\psi$ as follows.
We define:
\begin{align*}
w'
&\eqdef \sem{v \detassign f_{\alg}(y)}(w) = w;
(
\mem_w[v \mapsto \sem{f_{\alg}(y)}_w, h_{y,\alg} \mapsto \sem{h_{y,\alg}}_w + 1],
\ac',\testHist'
)
\end{align*}
where $\ac' = v \detassign f_{\alg}(y)$ and $\testHist' = \testHist_w \uplus \{\mem_w (y) \mapsto \{\alg\} \} $.
\begin{itemize}
\item Case $\psi \equiv \eta(u_1, \ldots, u_k)$.~
As with the proof of \eqref{lem:LemExpressAssign:asgn}, we obtain:
\begin{align*}
\sem{v \detassign f_{\alg}(y)}(w) \modelsi \eta(u_1, \ldots, u_k)
&~~\mbox{ iff }~~  w \modelsi \eta(u_1, \ldots, u_k)[\nicefrac{f_{\alg}(y)}{v}, \nicefrac{h_{y,\alg}+1}{h_{y,\alg}}].
\end{align*} 

\item Case $\psi \equiv \Know \phi$.~
We have:
\begin{align}
&\sem{v \detassign f_{\alg}(y)}(w) \modelsi \Know \phi \nonumber \\
&\mbox{iff }~ w' \modelsi \Know \phi \nonumber \\
&\mbox{iff }~
\mbox{for all $w_1 \in\calw$,~ $(w', w_1) \in \relO$}
~\mbox{implies}~ w_1 \modelsi \phi \nonumber \\
&\mbox{iff }~
\mbox{for all $w_2 \in\calw$,~ $(w, w_2) \in \relO$}
~\mbox{implies}~ w_2 \modelsi \phi[\nicefrac{f_{\alg}(y)}{v}, \nicefrac{h_{y,\alg}+1}{h_{y,\alg}}]
{.}
~~\mbox{ ($\dag$)} \nonumber
\end{align}
The last equivalence 
($\dag$)
is proved in a similar way to \eqref{lem:LemExpressAssign:asgn}.
To show the direction from left to right,
let $(w, w_2) \in \relO$.
We define $w_1$ by:
\begin{align*}
w_1 
&\eqdef \sem{v \detassign f_{\alg}(y)}(w_2)
= w_2;
\Big(\,
\substack{
\mem_{w_2}[v \mapsto {\scriptsize\sem{f_{\alg}(y)}_{w_2}},\, h_{y,\alg} \mapsto {\scriptsize\sem{h_{y,\alg}}_{w_2}} + 1],\\
v \detassign f_{\alg}(y),\testHist_{w_2} \uplus \{\mem_{w_2}(y) \mapsto \{\alg\} \}
\hspace{6ex}
}
\,\Big).
\end{align*}
Then $(w', w_1) \in \relO$.
By induction hypothesis, we conclude 
$w_2 \modelsi \phi[\nicefrac{f_{\alg}(y)}{v}, \allowbreak \nicefrac{h_{y,\alg}+1}{h_{y,\alg}}]$.
To show the other direction,
let $(w', w_1) \in \relO$
and $w_2 = w_1[0]; \allowbreak \cdots;w_1[\len{w_1}-2]$.
Then $\sem{v \detassign f_{\alg}(y)}(w_2) =  w_1$ and $(w, w_2) \in \relO$.
By induction hypothesis, we conclude $w_1 \modelsi \phi$.

\item Cases $\psi \equiv \neg \phi$, $\psi \equiv \phi_1 \land \phi_2$, $\psi \equiv \forall i. \phi$.~ 
Immediate by induction hypothesis.
\end{itemize}

\end{proof}

Now we prove the soundness of BHL
as follows.

\PropSound*

\begin{proof}[Proof for Theorem~\ref{thm:PropSound}]
We obtain the validity of the axioms and rules for basic constructs in \Fig{fig:rules:basic} as usual.

We show the validity of \axHist{} as follows.
Recall that the precondition in \axHist{} is 
$\psiPre \eqdef \psi[v \,{\mapsto} f_{\alg}\!(y),\, h_{y,\alg} \mapsto (h_{y,\alg}+1)]$.
Let $w = (\mem_{w}, \ac_{w}, \testHist_{w})$ be a possible world such that $w \models \psiPre$.
Let $w' \eqdef \allowbreak \sem{v := f_{\alg}(y)}(w)$.
By the semantics, we have
$\mem_{w'} = \mem_{w}[v \,{\mapsto} \sem{f_{\alg}(y)}_{\mem_{w}},\, h_{y,\alg} \mapsto \sem{h_{y,\alg}+1}_{\mem_{w}}]$.
By \Lem{lem:LemExpressAssign}, 
we obtain $w' \models \psi$.
Therefore, $(\envI\!,\, \envO) \vdash \triple{ \psiPre }{v := f_{\alg}(y)}{\psi}$ is valid.

We show the validity of \axPAR{} as follows.
Assume that $(\envI, \envO) \vdash\triple{\psi}{C_1; C_2}{\psi'}$.
Let $w$ be a world such that $w \models \psi$.
Then $\sem{C_1; C_2}(w) \models \psi'$.
By \Lem{lem:exchange:par:seq} in \App{sub:appendix:parallel},
we obtain $\sem{C_1 \,\|\, C_2}(w) \models \psi'$.
Hence, $(\envI\!,\, \envO) \vdash\triple{\psi}{C_1 \,\|\, C_2}{\psi'}$.

Therefore, we obtain BHL's soundness.
\end{proof}

\subsection{Proof for BHL's Relative Completeness}
\label{sub:proof:complete}

To prove BHL's relative completeness (\Thm{thm:PropComplete}),
we show the notions of extension and weakest preconditions as follows.

\begin{definition}[Extension] \label{def:extension} \rm
For a Kripke model $\M$ with a domain $\calw$ and an interpretation function $\cali$, we define the \emph{extension} of a formula $\phi\in\Fml$ by:
\begin{align*}
\phi^{\cali} \eqdef
\{ w\in\calw \mid w \modelsi \phi \}.
\end{align*}
For a $\phi\in\Fml$ and a set $S$ of worlds, we write $S \models \phi$ iff $w \models \phi$ for all $w\in S$.
\end{definition}

\begin{definition}[Weakest precondition] \label{def:wp} \rm
The \emph{weakest precondition} 
of a formula $\phi$ w.r.t. a program $C$ in $\cali$ is defined by:
\begin{align*}
\wpc{\cali}(C, \phi) \eqdef
\{ w \in \calw \mid \sem{C}(w) \modelsi \phi \}.
\end{align*}
\end{definition}

To derive BHL's relative completeness, we show the expressiveness of \ELHT{}.

\begin{restatable}[Expressiveness]{prop}{PropExpressive}
\label{prop:PropExpresive}
The assertion language 
\ELHT{} is expressive;
i.e., for every program $C$ and every formula $\phi\in\Fml$, there is a formula $F\in\Fml$ such that 
$F^{\cali} = \wpc{\cali}(C, \phi)$.
\end{restatable}

\begin{proof}
Let $w\in\calw$ and $\phi\in\Fml$.
By the definition of the weakest precondition, it is sufficient to prove that there is a formula 
$\wpf{\phi}{C} \in\Fml$ 
such that:
\begin{align}\label{eq:expressive:proof}
w \modelsi \wpf{\phi}{C}
~\mbox{ iff }~
\sem{C}(w) \modelsi \phi
{.}
\end{align}

We show this by induction on the program $C$.
The proof is analogous to~\cite{winskel} except for the case of parallel composition.
\begin{itemize}
\item Case $C \equiv \myskip{}$.~
Let $\wpf{\phi}{C} \eqdef \phi$.
Then:
\begin{align*}
w \modelsi \wpf{\phi}{C}
~\mbox{ iff }~&
w \modelsi \phi
&(\text{by Lemma \ref{lem:LemExpressAssign}  \eqref{lem:LemExpressAssign:skip}})
\\~\mbox{ iff }~&
\sem{\myskip{}}(w) \modelsi \phi
{.}
\end{align*}

\item 
Case $C \equiv {v \detassign e}$.~
Let $\wpf{\phi}{C} \eqdef \phi[\nicefrac{e}{v}]$.
Then:
\begin{align*}
w \modelsi \wpf{\phi}{C}
~\mbox{ iff }~&
w \modelsi \phi[\nicefrac{e}{v}]
&(\text{by Lemma \ref{lem:LemExpressAssign}  \eqref{lem:LemExpressAssign:asgn}})
\\~\mbox{ iff }~&
\sem{v \detassign e}(w) \modelsi \phi
{.}
\end{align*}

\item 
Case $C \equiv v \detassign f_{\alg}(y)$.~
Let $\wpf{\phi}{C} \eqdef \phi[\nicefrac{f_{\alg}(y)}{v}, \allowbreak \nicefrac{(h_{y,\alg}+1)}{h_{y,\alg}}]$.
Then we have:
\begin{align*}
w \modelsi \wpf{\phi}{C}
~~\mbox{ iff }~~&
w \modelsi \phi[\nicefrac{f_{\alg}(y)}{v}, \nicefrac{(h_{y,\alg}+1)}{h_{y,\alg}}]
&(\text{by Lemma \ref{lem:LemExpressAssign} \eqref{lem:LemExpressAssign:test}})
\\~~\mbox{ iff }~~&
\sem{v \detassign f_{\alg}(y)}(w) \modelsi \phi
{.}
\end{align*}

\item 
Case $C \equiv C_1 ; C_2$.~
Let
$\wpf{\phi}{C} \eqdef \wpf{\wpf{\phi}{C_2}}{C_1}$.
Then: 
\begin{align*}
w \modelsi \wpf{\phi}{C}
~\mbox{ iff }~&
w \modelsi \wpf{\wpf{\phi}{C_2}}{C_1}
\\~\mbox{ iff }~&
\sem{C_1}(w) \modelsi \wpf{\phi}{C_2}
& \mbox{ (by induction hypothesis)}
\\~\mbox{ iff }~&
\sem{C_2}(\sem{C_1}(w)) \modelsi \phi
& \mbox{ (by induction hypothesis)}
\\~\mbox{ iff }~&
\sem{C_1 ; C_2}(w) \modelsi \phi
{.}
\end{align*}

\item 
Case $C \equiv C_1 \| C_2$.~
Let
$\wpf{\phi}{C} \eqdef \wpf{\wpf{\phi}{C_2}}{C_1}$.
Then:
\begin{align*}
w \modelsi \wpf{\phi}{C}
~\mbox{ iff }~&
w \modelsi \wpf{\wpf{\phi}{C_2}}{C_1}
& \mbox{ (by applying the case $C = C_1 ; C_2$)}
\\~\mbox{ iff }~&
\sem{C_1 ; C_2}(w) \modelsi \phi
& \mbox{ (by \Lem{lem:exchange:par:seq})}
\\~\mbox{ iff }~&
\sem{C_1 \| C_2}(w) \modelsi \phi
{.}
\end{align*}

\item 
Case $C \equiv \myIf{e}{C_1}{C_2}$.~
Let 
$
\wpf{\phi}{C} \eqdef (e\land \wpf{\phi}{C_1}) \lor (\neg e\land \wpf{\phi}{C_2})
$.
Then:
\begin{align*}
\lefteqn{w \modelsi \wpf{\phi}{C}}\\
&\mbox{iff }~
w \modelsi (e\land \wpf{\phi}{C_1}) \lor (\neg e\land \wpf{\phi}{C_2})
\\&\mbox{iff }\mbox{either}
\left( \sem{e}_{w} = \mytrue \mbox{ and } w \modelsi \wpf{\phi}{C_1} \right)
\mbox{ or }
\left( \sem{e}_{w} = \myfalse \mbox{ and } w \modelsi \wpf{\phi}{C_2} \right)
\\&\mbox{iff }\mbox{either}
\left( \sem{e}_{w} = \mytrue \mbox{ and } \sem{C_1}(w) \modelsi \phi \right)
\mbox{ or }
\left( \sem{e}_{w} = \myfalse \mbox{ and } \sem{C_2}(w) \modelsi \phi \right)
\\ &\hspace{45ex} \mbox{ (by induction hypotheses)}
\\[-2ex]&\mbox{iff }~
\sem{C}(w) \modelsi \phi
{.}
\end{align*}

\item 
Case $C \equiv \myLoop{e}{C'}$.~
The way of the proof is similar to~\cite{winskel}.
By the semantics of \Prog{}, $\sem{C}(w) \modelsi \phi$ is logically equivalent to:
\begin{align}
\forall k~ \forall w_0, \ldots, w_k \in \calw. \nonumber
&~~~ w = w_0 \mbox{ and } \nonumber \\
&~~~ \forall i=0, \ldots, k.\,  \bigl( w_i \modelsi e \mbox{ and } \sem{C'}(w_i) = w_{i+1} \bigr) \nonumber \\
&~~~\hspace{13.5ex} \mbox{implies } w_k \modelsi e \lor \phi
{.}
\label{eq:complete:with-world}
\end{align}

To describe this using the assertion language \ELHT{}, we replace the possible worlds $w_i$ ($i = 0, \ldots, k$) with equivalent assertions as follows.
Let $\ov{v} = (v_1, \ldots, v_l)$ be all observable and invisible variables occurring in $C$ or $\phi$.
Then $\ov{v} \cap \IntVar = \emptyset$.
For $i = 0, \ldots, k$ and $j = 1, \ldots, l$, let $s_{ij} = w_i(v_j)$ and
$\ov{s_i} = (s_{i1}, \ldots, s_{il})\in \ints^{l}$.
We write $\phi[\nicefrac{\ov{s_i}}{\ov{v}}]$ for the assertion obtained by the simultaneous substitution of $\ov{s_i}$ for $\ov{v}$ in $\phi$.
Then each $w_i$ can be converted into the equivalent substitution $[\nicefrac{\ov{s_i}}{\ov{v}}]$ as follows:
for any interpretation function $\cali$,
\begin{align}
w_i \modelsi \phi
~~\mbox{ iff }~~ &
\modelsi \phi[\nicefrac{\ov{s_i}}{\ov{v}}]
~~\mbox{ iff }~~ 
w \modelsi \phi[\nicefrac{\ov{s_i}}{\ov{v}}].
\label{eq:complete:remove-world}
\end{align}
Then neither observable nor invisible variable occurs in $\phi[\nicefrac{\ov{s_i}}{\ov{v}}]$.

By \eqref{eq:complete:remove-world}, we can replace the worlds $w_0, \ldots, w_k$ in \eqref{eq:complete:with-world} with their corresponding substitutions $[\nicefrac{\ov{s_0}}{\ov{v}}], \ldots, [\nicefrac{\ov{s_k}}{\ov{v}}]$.
Thus, \eqref{eq:complete:with-world} is logically equivalent to:
\begin{align}
\forall k~ \forall \ov{s_0}, \ldots, \ov{s_k} \in \ints^{l}. \nonumber 
&~~~ w \modelsi (\ov{v}= \ov{s_0}) \mbox{ and } \nonumber \\
&~~~ \forall i=0, \ldots, k.\, 
\bigl( w \modelsi e[\nicefrac{\ov{s_i}}{\ov{v}}] \mbox{ and } 
\sem{C'}(w_i) = w_{i+1} \bigr) \nonumber \\
&~~~\hspace{13.5ex} \mbox{implies } w \modelsi (e \lor \phi)[\nicefrac{\ov{s_k}}{\ov{v}}]
{.}
\label{eq:complete:converted1}
\end{align}

To express $\sem{C'}(w_i) = w_{i+1}$ as a formula, we derive:
\begin{align}
&
\sem{C'}(w_i) = w_{i+1}
 \nonumber \\ &
\mbox{ iff }~
\sem{C'}(w_i) \neq \emptyset
~\mbox{ and }~
\sem{C'}(w_i) \modelsi \ov{v}=\ov{s_{i+1}}
& \hspace{0.5ex}\mbox{(by induction hypothesis)}\hspace{-3ex}
\nonumber \\ &\mbox{ iff }~
w_i \modelsi 
\neg\wpf{\myfalse}{C'}
~\mbox{ and }~
w_i \modelsi 
\wpf{(\ov{v} = \ov{s_{i+1}})}{C'}
& \mbox{(by \eqref{eq:complete:remove-world})}\hspace{-3ex}
\nonumber \\ &\mbox{ iff }~
w \modelsi 
\neg\wpf{\myfalse}{C'} [\nicefrac{\ov{s_i}}{\ov{v}}]
\land
\wpf{(\ov{v} = \ov{s_{i+1}})}{C'} [\nicefrac{\ov{s_i}}{\ov{v}}]
{.}
\label{eq:complete:converted2}
\end{align}
Now we define $\wpf{\phi}{C}$ by:
\begin{align}
\wpf{\phi}{C} \eqdef\,
& 
\forall k~ \forall \ov{s_0}, \ldots, \ov{s_k} \in \ints^{l}. \nonumber \\
& \bigl( (\ov{v}= \ov{s_0}) \land 
~ \forall i=0, \ldots, k.\, 
  (e\land \neg\wpf{\myfalse}{C'} \land \wpf{(\ov{v} = \ov{s_{i+1}})}{C'})
  [\nicefrac{\ov{s_i}}{\ov{v}}]
\bigr) \nonumber \\
& \rightarrow (e \lor \phi)[\nicefrac{\ov{s_k}}{\ov{v}}]
{.}
\label{eq:complete:converted:form}
\end{align}
By \eqref{eq:complete:converted2} and \eqref{eq:complete:converted:form},
$w\modelsi \wpf{\phi}{C}$ is logically equivalent to \eqref{eq:complete:converted1}.
Hence,
\begin{align*}
\sem{C}(w) \modelsi \phi
~~\mbox{ iff }~~ &
\eqref{eq:complete:with-world}
\\ ~~\mbox{ iff }~~ &
\eqref{eq:complete:converted1}
\\ ~~\mbox{ iff }~~ &
w\modelsi \wpf{\phi}{C}.
\end{align*}
\end{itemize}
Therefore, we obtain \eqref{eq:expressive:proof}.
\end{proof}

\begin{restatable}{lem}{LemDeduceWP}
\label{lem:LemDeduceWP}
For a program $C$ and a formula $\phi\in\Fml$, let $\wpf{\phi}{C}$ be a formula expressing the weakest precondition, i.e., $(\wpf{\phi}{C})^{\cali} = \wpc{\cali}(C, \phi)$.
Then we obtain $\env\vdash\triple{\wpf{\phi}{C}}{C}{\phi}$.
\end{restatable}

\begin{proof}
We show the lemma by induction on the program $C$ as follows.
\begin{itemize}
\item 
Case $C \equiv \myskip{}$.~
By applying the rule (\textsc{Conseq}) to $\models \wpf{\phi}{C} \rightarrow \phi$ and the axiom (\textsc{Skip}) $\env\vdash\triple{\phi}{C}{\phi}$,
we obtain $\env\vdash\triple{\wpf{\phi}{C}}{C}{\phi}$.

\item 
Case $C \equiv v \detassign e$.~
By applying the rule (\textsc{Conseq}) to $\models \wpf{\phi}{C} \rightarrow \phi[v\mapsto e]$ and the axiom (\textsc{UpdVar}) $\env\vdash\triple{\phi[v\mapsto e]}{C}{\phi}$,
we obtain $\env\vdash\triple{\wpf{\phi}{C}}{C}{\phi}$.
\item 
Case $C \equiv v \detassign f_{\alg}(y)$.~
By applying the rule (\textsc{Conseq}) to 
$\models \wpf{\phi}{C} \implies \phi[v \mapsto f_{\alg}(y), h_{y,\alg} \mapsto h_{y,\alg}+1]$ and the axiom (\textsc{Hist}),
we obtain $\env\vdash\triple{\wpf{\phi}{C}}{C}{\phi}$.

\item 
Case $C \equiv C_1 ; C_2$.~
By induction hypothesis, we have:
\begin{align*}
\env\vdash\triple{\wpf{\wpf{\phi}{C_2}}{C_1}}{C_1}{\wpf{\phi}{C_2}}
~~\mbox{ and }~~
\env\vdash\triple{\wpf{\phi}{C_2}}{C_2}{\phi}
{.}
\end{align*}
By applying (\textsc{Seq}), we obtain
$\env\vdash\triple{\wpf{\wpf{\phi}{C_2}}{C_1}}{C_1 ; C_2}{\phi}$.
Hence by applying (\textsc{Conseq}) with $\models \wpf{\phi}{C_1 ; C_2} \rightarrow \wpf{\wpf{\phi}{C_2}}{C_1}$,
we conclude $\env\vdash\triple{ \wpf{\phi}{C_1 ; C_2}}{C_1 ; C_2}{\phi}$.

\item 
Case $C \equiv C_1 \| C_2$.~
By applying the case of the sequential composition $C_1 ; C_2$,
we have $\env\vdash\triple{\wpf{\phi}{C_1 ; C_2}}{C_1 ; C_2}{\phi}$.
By applying the rules (\textsc{Par}) and (\textsc{Conseq}) with $\models \wpf{\phi}{C_1 \| C_2} \rightarrow \wpf{\phi}{C_1 ; C_2}$, 
we conclude $\env\vdash\triple{\wpf{\phi}{C_1 \| C_2}}{C_1 \|  C_2}{\phi}$.

\item 
Case $C \equiv \myIf{e}{C_1}{C_2}$.~
By induction hypothesis, we have:
\[
\env\vdash\triple{\wpf{\phi}{C_1} }{C_1}{\phi}  ~~\mbox{ and }~~
\env\vdash\triple{\wpf{\phi}{C_2}}{C_2}{\phi}.
\]
By applying (\textsc{Conseq}), we have
$\env\vdash\triple{e \land \wpf{\phi}{C_1} }{C_1}{\phi}$
and
$\env\vdash\triple{\neg e \land \wpf{\phi}{C_2}}{C_2}{\phi}$.
Then, by applying the rules (\textsc{If})
and (\textsc{Conseq}) with:
\[
\models \wpf{\phi}{\myIf{e}{C_1}{C_2}} \rightarrow (e \land \wpf{\phi}{C_1})\lor (\neg e \land \wpf{\phi}{C_2}),
\]
we conclude:
\[
\env\vdash\triple{  \wpf{\phi}{\myIf{e}{C_1}{C_2}}   }{\myIf{e}{C_1}{C_2}}{\phi}.
\]

\item 
Case $C \equiv \myLoop{e}{C'}$.~
Let $w\in\calw$ and $\cali$ be an interpretation function. Then:
\begin{align*}
w \modelsi \wpf{\phi}{\myLoop{e}{C'}}
&~\mbox{ iff }~\sem{\myLoop{e}{C'}}(w) \modelsi \phi\\
&~\mbox{ iff }~
\mbox{either } (\sem{e}(w) = \mytrue ~\mbox{ and }~ \sem{C';\myLoop{e}{C'}}(w) \modelsi \phi)\\[-0.2ex]
&\qquad ~~\,\mbox{ or }~~(\sem{e}(w) = \myfalse ~\mbox{ and }~ w \modelsi \phi)\\
&~\mbox{ iff }~
\mbox{either } (\sem{e}(w) = \mytrue ~\mbox{ and }~ \sem{C'}(w) \modelsi \wpf{\phi}{\myLoop{e}{C'}})\\[-0.2ex]
&\qquad ~~\,\mbox{ or }~~ (\sem{e}(w) = \myfalse ~\mbox{ and }~ w \modelsi \phi)\\
&~\mbox{ iff }~
w \modelsi e \land \wpf{ \wpf{\phi}{\myLoop{e}{C'}} }{C'}~~\mbox{ or }~~ w \modelsi \neg e \land \phi.
\end{align*}
This implies: 
\begin{align}
&\models e \land \wpf{\phi}{\myLoop{e}{C'}} \implies \wpf{\wpf{\phi}{\myLoop{e}{C'}}}{C'} \label{weakest_precondition:while:caseA}\\ 
&\models \neg e \land \wpf{\phi}{\myLoop{e}{C'}} \implies  \phi.\label{weakest_precondition:while:caseB}
\end{align}
By induction hypothesis, we have:
\[
\env\vdash\triple{\wpf{\wpf{\phi}{\myLoop{e}{C'}}}{C'}}{ C' }{\wpf{\phi}{\myLoop{e}{C'}}}.
\]
By applying (\textsc{Conseq}) with \eqref{weakest_precondition:while:caseA}, we have:
\[
\env\vdash\triple{e\, \land \wpf{\phi}{\myLoop{e}{C'}}}{ C' }{\wpf{\phi}{\myLoop{e}{C'}}}.
\]
By applying (\textsc{Loop}), we obtain:
\[
\env\vdash\triple{\wpf{\phi}{\myLoop{e}{C'}} }{ \myLoop{e}{C'} }{\neg e \land \wpf{\phi}{\myLoop{e}{C'}}}.
\]
By applying (\textsc{Conseq}) with \eqref{weakest_precondition:while:caseB}, we conclude:
\[
\env\vdash\triple{\wpf{\phi}{\myLoop{e}{C'}} }{ \myLoop{e}{C'} }{\phi}.
\]
\end{itemize}
\end{proof}

Finally, we prove the relative completeness of BHL as follows.

\PropComplete*

\begin{proof}[Proof of \Thm{thm:PropComplete}]
We assume the validity of a judgment $\env\models\triple{\psi}{C}{\phi}$.
Let $w$ be a world such that $w\models\psi$.
By the validity of the judgment, 
we have $w \in \wpc{\cali}(C, \phi)$.

By \Propo{prop:PropExpresive}, there exists a formula $\wpf{\phi}{C}$ that expresses the weakest precondition, that is, $(\wpf{\phi}{C})^{\cali} = \wpc{\cali}(C, \phi)$ for any interpretation function $\cali$.
Thus, it follows from $w\models\psi$ and $w \in \wpc{\cali}(C, \phi)$ that $w \models \wpf{\phi}{C}$.
Hence $\Gamma \models \psi \rightarrow \wpf{\phi}{C}$.

By \Lem{lem:LemDeduceWP}, we obtain $\env\vdash\triple{\wpf{\phi}{C}}{C}{\phi}$.
By applying the rule (\textsc{Conseq}) to $\Gamma \models \psi \rightarrow \wpf{\phi}{C}$ and $\env\vdash\triple{\wpf{\phi}{C}}{C}{\phi}$,
we obtain $\env\vdash\triple{\psi}{C}{\phi}$.
\end{proof}

\end{document}